\documentclass[10pt,twocolumn,letterpaper]{article}

\usepackage{cvpr}
\usepackage{times}
\usepackage{epsfig}
\usepackage{graphicx}
\usepackage{amsmath}
\usepackage{amssymb}

\usepackage[bf,small]{caption}

\usepackage{hyperref,bibspacing}

\usepackage{multirow}
\usepackage{subfig}
\usepackage{comment}
\usepackage{url}
\usepackage{amsthm}
\usepackage{xspace}

\newcommand{\fastsdp}{SDCut\xspace}
\newcommand{\lowrank}{LowRank\xspace}

\newcommand{\be}{\mathbf e}
\newcommand{\bu}{\mathbf u}
\newcommand{\bx}{\mathbf x}
\newcommand{\bt}{\mathbf t}
\newcommand{\bh}{\mathbf h}
\newcommand{\by}{\mathbf y}
\newcommand{\bc}{\mathbf c}

\newcommand{\bb}{\mathbf b}
\newcommand{\bp}{\mathbf p}

\newcommand{\bX}{\mathbf X}

\newcommand{\bP}{\mathbf P}
\newcommand{\bD}{\mathbf D}
\newcommand{\bW}{\mathbf W}
\newcommand{\bA}{\mathbf A}
\newcommand{\bI}{\mathbf I}
\newcommand{\bK}{\mathbf K}
\newcommand{\bH}{\mathbf H}
\newcommand{\bC}{\mathbf C}
\newcommand{\bN}{\mathbf N}
\newcommand{\bM}{\mathbf M}
\newcommand{\bR}{\mathbf R}
\newcommand{\bL}{\mathbf L}
\newcommand{\bB}{\mathbf B}
\newcommand{\bZ}{\mathbf Z}
\newcommand{\bY}{\mathbf Y}

\newcommand{\sst}{\mathrm{s.t.}}

\ifx\theorem\undefined
\newtheorem{theorem}{Theorem}
\newenvironment{theorem*}{\par\noindent{\bf Theorem\ }}{\hfill\\[2mm]}

\newtheorem{claim}{Result}

\newenvironment{corollary*}{\par\noindent{\bf Corollary\ }}{\hfill\\[2mm]}

\fi

\def\T{{\!\top}}

\cvprfinalcopy %

\begin{document}

\title{A Fast Semidefinite Approach to Solving Binary Quadratic
Problems
\thanks{\em  
Submitted to IEEE Conf.\ Computer Vision and
Pattern Recognition on 15 Nov.\ 2012;
Accepted 24 Feb.\ 2013. 
Content may be slightly different from the final published version.
}
}

\author{Peng Wang, Chunhua Shen, Anton van den Hengel              
\\
School of Computer Science, The University of Adelaide, Australia}

\maketitle
\begin{abstract}

    Many computer vision problems can be formulated as binary
    quadratic programs (BQPs).  Two classic relaxation methods are
    widely used for solving BQPs, namely, spectral methods and
    semidefinite programming (SDP), each with their own advantages and
    disadvantages.  Spectral relaxation is simple and easy to
    implement, but its bound is loose.    Semidefinite relaxation has
    a tighter bound, but its computational complexity is high for
    large scale problems.  We present a new SDP formulation for BQPs,
    with two desirable properties.  First, it has a similar relaxation 
    bound to conventional SDP formulations.  Second, compared with
    conventional SDP methods, the new SDP formulation leads to a
    significantly more efficient and scalable dual optimization
    approach, which has the same degree of complexity as spectral
    methods.  Extensive experiments on various applications
    including clustering, image segmentation, co-segmentation and
    registration demonstrate the usefulness of our SDP formulation for
    solving large-scale BQPs.   
   
\end{abstract}

\tableofcontents

\section{Introduction}

Many problems in computer vision can be formulated as binary quadratic problems, such as 
image segmentation, image restoration, graph-matching and problems formulated by Markov Random Fields (MRFs).
Because general BQPs are NP-hard, they are commonly approximated by spectral or semidefinite relaxation.

Spectral methods convert BQPs  into  eigen-problems.
Due to their simplicity, spectral methods have been applied to a variety of problems in computer vision, 
such as image segmentation~\cite{Shi2000normalized,Yu2004segmentation}, motion segmentation~\cite{Lauer09spectralmotion} 
and many other MRF applications~\cite{Cour_solvingmarkov}.
However, the bound of spectral relaxation is loose and can lead to poor solution quality 
in many cases~\cite{Guattery98onthe,Lang05fixingtwo,Kannan00onclusterings}.
Furthermore, the spectral formulation is hard to generalize to
accommodate inequality constraints~\cite{Cour_solvingmarkov}.

In contrast, SDP methods produce tighter approximations than spectral methods,
which have been applied to problems including image segmentation~\cite{Heiler2005semi}, restoration~\cite{Keuchel2003binary,Olsson07solvinglarge},
subgraph matching~\cite{Schellewald05}, co-segmentaion~\cite{Joulin2010dis} and general MRFs~\cite{Torr2003MRFSDP}.
The disadvantage of SDP methods, however, is their poor scalability
for large-scale problems.
The worst-case complexity of solving a generic SDP problem 
involving a matrix variable of size $ n \times n $ and $ {\cal O}(n) $
linear constraints is about $\mathcal{O}(n^{6.5})$,
using interior-point methods.

In this paper, we present a new SDP formulation for BQPs (denoted by \fastsdp).
Our approach achieves higher quality solutions than spectral methods
while being significantly faster than the conventional SDP
formulation.
Our main contributions  are as follows. 

 \noindent
 ($ i $) 
      A new SDP formulation (\fastsdp) is proposed to solve binary quadratic problems. 
      By virtue of its use of  the dual formulation, our approach is
      simplified and can be solved efficiently by first order
      optimization methods, \eg, quasi-Newton methods.  \fastsdp has
      the same level of computational complexity as
      spectral methods, roughly $\mathcal{ O }(n^3)$,
      which is much lower than the conventional SDP formulation using
      interior-point method. %
      \fastsdp also achieves a similar bound with the conventional SDP
      formulation and therefore produces better estimates than
      spectral relaxation.

 \noindent 
 ($ ii $)
We demonstrate the flexibility of \fastsdp by applying it to a few computer vision applications.
      The \fastsdp formulation allows additional equality or inequality constraints,
      which enable it to have a broader application area than the spectral method.

{\bf Related work}
Our method is motivated by the work of Shen
\etal~\cite{Shen2011scalabledual}, which presented a fast dual SDP
approach to Mahalanobis metric learning.
The  Frobenius-norm regularization in their objective function plays an important role, which leads to a simplified dual formulation.
They, however, focused on learning a metric for nearest neighbor
classification. In contrast, here we are interested in discrete
combinatorial optimization problems arising in computer vision.  
In~\cite{Journee2010lowrank}, the SDP problem was reformulated by the
non-convex low-rank factorization $\bX = \bY \bY^{\T}$, where $\bY \in
\mathbb{R}^{n \times m}, m \ll n$. 
This method finds a locally-optimal low-rank solution, and runs faster than the interior-point method.
We compare \fastsdp with the method in~\cite{Journee2010lowrank}, on image co-segmentation.
The results show that our method achieves a better solution quality and a faster running speed.
Olsson \etal~\cite{Olsson07solvinglarge} proposed fast SDP methods
based on spectral sub-gradients and trust region methods.
Their methods cannot be extended to accommodate inequality constraints, while ours
is much more general and flexible.   
Krislock \etal\ 
\cite{krislock2012improved}
have independently formulated a similar SDP for the
MaxCut problem, which is simpler than the problems that we solve here. 
Moreover, they focus on globally solving the MaxCut problem using branch-and-bound.

{\bf Notation}
A matrix is denoted by a bold capital letter ($\bX$) and a column vector is by a bold lower-case letter ($\bx$).
$\mathcal{S}_n$ denotes the set of $n \times n$ symmetric matrices.
$\bX \succcurlyeq \mathbf{0}$ represents that the matrix $\bX$ is positive semidefinite (p.s.d.).
For two vectors, $\bx \leq \by$ indicates the element-wise inequality;
$\mathbf{diag}(\cdot)$ denotes the diagonal entries of a matrix.
The trace of a matrix is denoted as $\mathrm{trace}(\cdot)$.
The rank of a matrix is denoted as $\mathrm{rank}(\cdot)$.
$\lVert \cdot \rVert_1$ and $\lVert \cdot \rVert_2$ denote the $\ell_1$ and $\ell_2$ norm of a vector respectively.
$\lVert \bX \rVert_F^2 = \mathrm{trace}(\bX \bX^{\T}) = \mathrm{trace}(\bX^{\T} \bX)$ is the Frobenius norm.
The inner product of two matrices is defined as $\langle \bX, \bY \rangle = \mathrm{trace}(\bX^{\T}\bY)$.
$\bX \circ \bY$ denotes the Hadamard product of $\bX$ and $\bY$.
$\bX \otimes \bY$ denotes the Kronecker product of $\bX$ and $\bY$.
$\bI_n$ indicates the $n \times n$ identity matrix and $\be_n$ denotes an $n \times 1$ vector with all ones.
${\lambda_i}(\bX)$ and $\mathbf{\bp}_i(\bX)$ indicate the $i$th eigenvalue and the corresponding eigenvector of the matrix $\bX$.
We define the positive and negative part of $\bX$ as:
\begin{align}
{\textstyle 
    \bX_{+} = \sum_{\lambda_i >0} \lambda_i \mathbf{\bp}_{i} \mathbf{\bp}_{i}^{\T}, \ \ 
    \bX_{-} = \sum_{\lambda_i <0} \lambda_i \mathbf{\bp}_{i}
    \mathbf{\bp}_{i}^{\T},
}
\end{align}
    and explicitly $\bX = \bX_{+} + \bX_{-}$.

{\bf Euclidean projection onto the p.s.d.\ cone}
Our method relies on the following results 
(see Sect. 8.1 of~\cite{boyd2004convex}):
\begin{align}
{\textstyle \bX_{+} =  \mathrm{argmin}_{\bY \succcurlyeq \mathbf{0} } \  \lVert \bY - \bX \rVert_F^2 }.\label{eq:notation_1}
\end{align}
Although~\eqref{eq:notation_1} is an SDP problem, it can be solved efficiently by using eigen-decomposition.
This is the key observation to simplify our SDP formulation.

\section{Spectral and Semidefinite Relaxation}

As a simple example of a
binary quadratic problem, we consider the following optimization
problem:
\begin{align}
\min_{\bx} &\ \bx^{\T} \bA \bx, \ \sst \ \bx \in \{-1,1\}^n,
\label{eq:backgd_bqp} 
\end{align}
where $\bA \in \mathcal{S}_n$. {\it The integrality constraint makes the BQP
problem non-convex and NP-hard}.

One of the spectral methods (again by way of example) relaxes the
constraint $\bx \in \{-1,1\}^n$ to $\lVert \bx \rVert_2^2 = n$:
\begin{align}
\min_{\bx} \bx^{\T} \bA \bx, \ \sst \ \lVert \bx \rVert_2^2 = n. \label{eq:backgd_spectral}
\end{align}
This problem can be solved by the eigen-decomposition of $\bA$ in $\mathcal{O}(n^{3})$ time.
Although appealingly simple to  implement, the spectral relaxation often yields poor solution quality.
There is no guarantee on the bound of its solution with respect to the optimum of~\eqref{eq:backgd_bqp}.
The poor bound of spectral relaxation has been verified by a variety of authors~\cite{Guattery98onthe,Lang05fixingtwo,Kannan00onclusterings}.
Furthermore, it is difficult to generalize the spectral method to BQPs with linear or quadratic inequality constraints.
Although linear equality constraints can be considered~\cite{Gour2006}, 
solving~\eqref{eq:backgd_spectral} under additional inequality constraints is in general NP-hard~\cite{Cour_solvingmarkov}.

Alternatively, BQPs can be relaxed to semidefinite programs.
Firstly, let us consider an equivalent problem of~\eqref{eq:backgd_bqp}:
\begin{align}
\min_{\bX \succcurlyeq \mathbf{0}} \  \langle \bX, \bA \rangle, \ \sst \ \mathrm{diag}(\bX) = \be, \mathrm{rank}(\bX)=1.  \label{eq:backgd_sdp0}
\end{align}
The original problem is lifted to the space of rank-one p.s.d.\   matrices of the form $\bX = \bx \bx^{\T}$,
The number of variables increases from $n$ to $n(n+1)/2$.
Dropping the only non-convex rank-one constraint,~\eqref{eq:backgd_sdp0} is a convex SDP problem, 
which can be solved conveniently by standard convex optimization toolboxes, \eg, SeDuMi~\cite{Sturm98usingsedumi} and SDPT3~\cite{Toh99sdpt3}.
The SDP relaxation is tighter than spectral relaxation~\eqref{eq:backgd_spectral}.
In particular, it has been proved in~\cite{Goemans95improved} that the
expected values of solutions are bounded for the SDP formulation of
some BQPs (\eg, MaxCut).
Another advantage of the SDP formulation is the ability of solving
problems of more general forms, \eg, quadratically constrained quadratic program (QCQP).
Quadratic constraints on $\bx$ are transformed to linear constraints on $\bX = \bx \bx^{\T}$.
In summary, the constraints for SDP can be either equality or inequality.

The general form of the SDP problem is expressed as:
\begin{subequations}
\begin{align}
\min_{\bX \succcurlyeq \mathbf{0}} &\  \langle \bX,\bA \rangle, \\
\sst &\ \langle  \bX, \bB_i \rangle = b_i, \  \  \forall i = 1, \dots, p \label{eq:backgd_sdp1_cons1}, \\
     &\ \langle  \bX, \bB_j \rangle \leq b_j, \ \forall j = {p\!+\!1}, \dots, m  \label{eq:backgd_sdp1_cons2}.
\end{align}
\label{eq:backgd_sdp1}
\end{subequations}
The most significant drawback of SDP methods is the poor scalability to large problems.  
Most optimization toolboxes, \eg, SeDuMi~\cite{Sturm98usingsedumi} and SDPT3~\cite{Toh99sdpt3},
use the interior-point method for solving SDP problems, which has $\mathcal{O}(n^{6.5})$ complexity,
making it impractical for large scale problems.

\section{\fastsdp Formulation}

    Before we present the new SDP formulation, 
we first introduce a property of the following set:
\begin{align}
\Omega(\eta) = \{ \bX \in  \mathcal{S}_n | \bX \succcurlyeq \mathbf{0}, \mathrm{trace}(\bX) = \eta \}.
\end{align}
The set $\Omega(\eta)$ is known as a spectrahedron, which is the
intersection of a linear subspace (\ie $\mathrm{trace}(\bX) = \eta$)
and the p.s.d.\   cone.

For the set $\Omega(\eta)$, we have the following theorem, which is an extension of the one in~\cite{Malick2007spherical}.
\begin{theorem} (The spherical constraint on a spectrahedron).
For $ \bX \in \Omega(\eta)$, we have the inequality $\lVert \bX \rVert_F \leq \eta$, in which 
the equality holds if and only if $\mathrm{rank}(\bX) = 1$.
\label{thm:1}
\end{theorem}
\begin{proof}
For a matrix $\bX \in \Omega(\eta)$, $\lVert \bX \rVert_F^2 =
\mathrm{trace} (\bX \bX^\T ) = \lVert \mathbf{\lambda}(\bX) \rVert_2^2 \leq
\lVert \mathbf{\lambda}(\bX) \rVert_1^2$.
Because $\bX \succcurlyeq \mathbf{0}$, then $\mathbf{\lambda}(\bX)
\geq \mathbf{0}$ and $\lVert \mathbf{\lambda}(\bX) \rVert_1 =
\mathrm{trace}(\bX)$.
Therefore
\begin{align}
\lVert \bX \rVert_F = \lVert \mathbf{\lambda}(\bX) \rVert_2 \leq \lVert (\mathbf{\lambda}(\bX)) \rVert_1 = \eta. \label{eq:fastsdp_spherical}
\end{align}
Because $\lVert \bx \rVert_2 = \lVert \bx \rVert_1$ holds if and only if only one element in $\bx$ is non-zero, 
the equality holds for~\eqref{eq:fastsdp_spherical} if and only if there is only one non-zero eigenvalue for $\bX$, i.e., $\mathrm{rank}(\bX) = 1$.
\end{proof}
This theorem shows the rank-one constraint is equivalent to $\lVert
\bX \rVert_F = \eta$ for p.s.d.\ matrices with a fixed trace.

The constraint on $\mathrm{trace}(\bX)$ is common in the SDP formulation for BQPs. 
For $\bx \!\in\! \{ -1,1 \}^n$, we have $\mathrm{diag}(\bx \bx^{\T}) = \be$, and so $\mathrm{trace}(\bx \bx^{\T}) = n$.
Therefore $\lVert \bX \rVert_F \leq \eta$ is implicitly involved in the SDP formulation of BQPs.

Then we have a geometrical interpretation of SDP relaxation. 
The non-convex spherical constraint $\lVert \bX \rVert_F = \eta$ is relaxed to the convex inequality constraint $\lVert \bX \rVert_F \leq \eta$:
\begin{align}
\min_{\bX \succcurlyeq \mathbf{0}} \  \langle \bX,\bA \rangle, \ \ \
\sst \ \lVert \bX \rVert_F^2 - \eta^2 \leq 0, \ \eqref{eq:backgd_sdp1_cons1},  \ \eqref{eq:backgd_sdp1_cons2}. \  \label{eq:fastsdp_analysis1}
\end{align}

Inspired by the spherical constraint, we consider the following SDP formulations:
\begin{align}
\min_{\bX \succcurlyeq \mathbf{0}} \  \langle \bX,\bA \rangle, \ \ \
\sst \ \lVert \bX \rVert_F^2 - \eta^2 \leq \rho, \ \eqref{eq:backgd_sdp1_cons1}, \ \eqref{eq:backgd_sdp1_cons2}. \label{eq:fastsdp_analysis2} \\
\min_{\bX \succcurlyeq \mathbf{0}} \  \langle \bX,\bA \rangle + \sigma ( \lVert \bX \rVert_F^2 - \eta^2), \ \ \
\sst \ \eqref{eq:backgd_sdp1_cons1}, \ \eqref{eq:backgd_sdp1_cons2} \label{eq:fastsdp_analysis3}.
\end{align}
where $\rho < 0$ and $\sigma > 0$ are scalar parameters. Given a $\rho$, one can always find a $\sigma$,
making the problems~\eqref{eq:fastsdp_analysis2} and~\eqref{eq:fastsdp_analysis3} equivalent.

The problem~\eqref{eq:fastsdp_analysis2} has the same objective function with~\eqref{eq:fastsdp_analysis1},
but its search space is a subset of the feasible set of~\eqref{eq:fastsdp_analysis1}.
Hence~\eqref{eq:fastsdp_analysis2} finds a sub-optimal solution to~\eqref{eq:fastsdp_analysis1}.
The gap between the solution of~\eqref{eq:fastsdp_analysis2} and~\eqref{eq:fastsdp_analysis1} vanishes when $\rho$ approaches $0$.

On the other hand, because $\lVert \bX \rVert_F^2 - \eta^2 \leq 0$, 
the objective function of~\eqref{eq:fastsdp_analysis3} is not larger than the one of~\eqref{eq:fastsdp_analysis1}.
When $\sigma$ approaches $0$, the problem~\eqref{eq:fastsdp_analysis3} is equivalent to~\eqref{eq:fastsdp_analysis1}. 
For a small $\sigma$, the solution of~\eqref{eq:fastsdp_analysis3} approximates the solution of~\eqref{eq:fastsdp_analysis1}.
{\em When $\sigma$ approaches $0$, the bound of~\eqref{eq:fastsdp_analysis3} is arbitrarily close to the bound of~\eqref{eq:fastsdp_analysis1}.}

Although problems~\eqref{eq:fastsdp_analysis2} and~\eqref{eq:fastsdp_analysis3}
can be converted into standard SDP problems, 
solving them using interior-point methods can be very slow.
Next, we show that the dual of~\eqref{eq:fastsdp_analysis3} has a much simpler form.
\begin{claim} 
The dual problem of~\eqref{eq:fastsdp_analysis3} can be simplified to 
\begin{align}
\max_{\bu} &\,\,\,\, - \frac{1}{4\sigma} \lVert \bC(\bu)_{-} \rVert_F^2 \!-\! \bu^{\T} \bb \!-\! \sigma \eta^2,  \label{eq:fastsdp_dual} \\
\sst       &\,\,\,\, u_j \geq 0,  \, \forall j = p+1, \dots, m,  \notag
\end{align}
where $\bC(\bu) = \sum_{i=1}^{m} u_i \bB_i + \bA$.
\end{claim}
\begin{proof}
The Lagrangian of the primal problem~\eqref{eq:fastsdp_analysis3} is:
\begin{align}
\mathrm{L} (\bX, \bu, \bZ)  =& \langle \bX, \bA \rangle - \langle \bX, \bZ \rangle + \sigma \lVert \bX \rVert_F^2 - \sigma \eta^2 \notag \\
           & + \sum_{i=1}^m u_i (\langle \bX, \bB_i\rangle \! - \! b_i),  \label{eq:fastsdp_lagrangian}
\end{align}
with $\bZ \succcurlyeq \mathbf{0}$ and $u_j \geq \mathbf{0}, \ \forall j = p+1, \dots, m$.
$\bZ \in \mathbb{R}^{n \times n}$ is the dual variable \wrt the constraint $\bX \succcurlyeq \mathbf{0}$;
$\bu \in \mathbb{R}^m$ is the dual variable \wrt the constraints~\eqref{eq:backgd_sdp1_cons1},~\eqref{eq:backgd_sdp1_cons2}.

Since the primal problem~\eqref{eq:fastsdp_analysis3} is convex, and
both the primal and dual problems are feasible,  strong duality holds. 
The primal optimal $\bX^\star$ is a minimizer of $\mathrm{L}(\bX, \bu^\star, \bZ^\star)$,
\ie, $\nabla_{\bX = \bX^\star} \mathrm{L} (\bX, \bu^\star, \bZ^\star)  = 0$.
Then we have
\begin{align}
\bX^\star \!=\! \frac{1}{2\sigma} (\bZ^\star \!-\! \bA \!-\! \sum_{i=1}^{m} u_i^\star \bB_i) = \frac{1}{2\sigma} (\bZ^\star \!-\! \bC(\bu^\star)).
\end{align}  
By substituting $\bX^\star$ in the Lagrangian~\eqref{eq:fastsdp_lagrangian}, we obtain the dual problem:
\begin{align}
\max_{\bu, \bZ} &\quad -\frac{1}{4\sigma} \lVert \bZ - \bC(\bu) \rVert_F^2 - \bu^{\T} \bb - \sigma \eta^2, \label{eq:fastsdp_dual_2} \\
\sst            &\quad \bZ \succcurlyeq \mathbf{0}, \,\, u_j \geq \mathbf{0}, \ \forall j = p+1, \dots, m.  \notag
\end{align}
As the dual~\eqref{eq:fastsdp_dual_2} is still a SDP problem, 
it seems that no efficient method can be used to solve~\eqref{eq:fastsdp_dual_2} directly, other than the interior-point algorithms.

Fortunately, the p.s.d.\ matrix variable $\bZ$ can be eliminated.
Given a fixed $\bu$, the dual~\eqref{eq:fastsdp_dual_2} can be simplified to:
\begin{align}
\min_{\bZ} \  \lVert \bZ - \bC(\bu) \rVert_F^2, \ \sst \ \bZ \succcurlyeq \mathbf{0}. \label{eq:fastsdp_projection}
\end{align}
Based on~\eqref{eq:notation_1}, the problem~\eqref{eq:fastsdp_projection} has 
an explicit solution: $\bZ = \bC(\bu)_{+}$.
By substituting $\bZ$ to~\eqref{eq:fastsdp_dual_2}, the dual problem is simplified to~\eqref{eq:fastsdp_dual}.
\end{proof}

We can see that the simplified dual problem~\eqref{eq:fastsdp_dual} is {\em not} a SDP problem.
The number of dual variables is $m$, \ie, the number of constraints in the primal problem~\eqref{eq:fastsdp_analysis3}.
In most of cases, $m \ll n^2$ where $n^2$ is the number of primal variables, 
and so the problem size of the dual is much smaller than that of the primal.

The gradient of the objective function of~\eqref{eq:fastsdp_dual} can be calculated as
\begin{align}
\mathrm{g}(u_i) = - \frac{1}{2\sigma} \left\langle \bC(\bu)_{-}, \bB_i \right\rangle - b_i, \forall i = 1, \dots, m. \label{eq:fastsdp_grd}
\end{align}
Moreover, {\em the objective function of~\eqref{eq:fastsdp_dual} is differentiable but not necessarily twice differentiable,}
which can be inferred on the results in Sect.~5 in~\cite{boyd2004convex}.

Based on the following relationship: 
\begin{align} 
\bX^{\star} = \frac{1}{2\sigma} (\bC(\bu^\star)_{+} \!-\! \bC(\bu^\star)) = - \frac{1}{2\sigma} \bC(\bu^\star)_{-}, \label{eq:fastsdp_primal_dual}
\end{align}
the primal optimal $\bX^\star$ can be calculated from the dual optimal $\bu^\star$.

{\bf Implementation}
We have used L-BFGS-B~\cite{Zhu94lbfgsb} for the optimization of~\eqref{eq:fastsdp_dual}.
All code is written in MATLAB (with mex files) and the results are tested on a $2.7$GHz Intel CPU.

The convergence tolerance settings of L-BFGS-B is set to the default, and the number of limited-memory vectors is set to $200$.
Because we need to calculate the value and gradient of the dual objective function at each gradient-descent step,
a partial eigen-decomposition should be performed to compute $\bC(\bu)_{-}$ at each iteration; this is the most computationally expensive part.  
The default ARPACK embedded in MATLAB is used to calculate the eigenvectors smaller than $0$. 
Based on the above analysis, a small $\sigma$ will improve the solution accuracy;
but we find that the optimization problem becomes ill-posed for an extremely small $\sigma$, and more iterations are needed for convergence.
In our experiments, $\sigma$ is set within the range of $[10^{-4}, 10^{-2}]$.

There are several techniques to speed up the eigen-decomposition process for \fastsdp:
(1)   In many cases, the matrix $\bC(\bu)$ is sparse or structural, which leads to an efficient way for calculating $\bC \bx$ for an arbitrary vector $\bx$.
Furthermore, because ARPACK only needs a callback function for the matrix-vector multiplication,
the process of eigen-decomposition can be very fast for matrices with specific structures.
(2)
As the step size of gradient-descent, $\lVert \Delta \bu \rVert_1$, becomes significantly small after some initial iterations, 
the difference $\lVert \bC(\bu) \!- \!\bC(\bu\!+\!\Delta \bu) \rVert_1$ turns to be small as well.
Therefore, the eigenspace of the current $\bC$ is a good choice of the starting point for the next eigen-decomposition process.
A suitable starting point can accelerate convergence considerably.

After solving the dual using L-BFGS-B, the optimal primal $\bX^\star$
is calculated from the dual optimal $\bu^\star$ based on~\eqref{eq:fastsdp_primal_dual}.

Finally, the optimal variable $\bX^\star$ should be discretized to the feasible binary solution $\bx^\star$.
The discretization method is dependent on specific applications, which
will be discussed separately in the section of applications.

In summary, the \fastsdp is solved by the following steps.

\noindent
{\bf Step 1}: Solve the dual problem~\eqref{eq:fastsdp_dual} using L-BFGS-B,
    based on the application-specific $\bA$, $\bB$, $\bb$ and the
    $\sigma$ chosen by the user.
      The gradient of the objective function is calculated through~\eqref{eq:fastsdp_grd}. 
      The optimal dual variable $\bu^\star$ is obtained when the dual~\eqref{eq:fastsdp_dual} is solved.

\noindent
{\bf Step 2}:
    Compute the optimal primal variable $\bX^\star$ using~\eqref{eq:fastsdp_primal_dual}.
 
\noindent    
{\bf Step 3}:    
    Discretize $\bX^\star$ to a feasible binary solution $\bx^\star$.

{\bf Computational Complexity}
The complexity for eigen-decomposition is $\mathcal{O}(n^3)$ where $n$ is the number of rows of matrix $\bA$,
therefore our method is $\mathcal{O}(kn^3)$ where $k$ is the number of gradient-descent steps of L-BFGS-B.
$k$ can be considered as a constant, which is irrelevant with the matrix size in our experiments.
Spectral methods also need the computation of the eigenvectors of the same matrix $\bA$,
which means they have the same order of complexity with \fastsdp.
As the complexity of  interior-point SDP solvers is $\mathcal{O}(n^{6.5})$,
our method is much faster than the conventional SDP method. 

Our method can be further accelerated by using faster eigen-decomposition method:
a problem that has been studied in depth for a long time.
Efficient algorithms and well implemented toolboxes have been available recently.
By taking advantage of them, \fastsdp can be applied to even larger problems.

\section{Applications}

In this section, we show several applications of \fastsdp in computer vision.
Because \fastsdp can handle different types of constraints (equality/inequality, linear/quadratic),
it can be applied to more problems than spectral methods.

\subsection{Application 1: Graph Bisection}

{\bf Formulation}
Graph bisection is a problem of separating the vertices of a weighted graph into two disjoint sets with equal cardinality, 
and minimize the total weights of cut edges.
The problem can be formulated as:
\begin{align}
\min_{\bx \in \{-1,+1\}^n} \bx^{\T} \bL \bx, \ \sst \  \bx^{\T} \be = 0,
\label{eq:gb_1}
\end{align}
where $\bL = \bD - \bW$ is the graph Laplacian matrix, 
$\bW$ is the weighted affinity matrix,  and $\bD = \mathbf{diag}(\bW \be)$ is the degree matrix.
The classic spectral clustering approaches, \eg, RatioCut and NCut~\cite{Shi2000normalized}, are in the following forms:
\begin{align}
\text{{ RatioCut:}}  &\ \min_{\bx \in \mathbb{R}^n} \bx^{\T} \bL         \bx, \ \sst \  \bx^{\T} \be  = 0, \lVert \bx \rVert_2^2 = n,  \\
\text{{ NCut:}}      &\ \min_{\bx \in \mathbb{R}^n} \bx^{\T} \tilde{\bL} \bx, \ \sst \  \bx^{\T} \bc  = 0, \lVert \bx \rVert_2^2 = n, 
\end{align}
where $\tilde{\bL} = \bD^{-1/2} \bL \bD^{-1/2}$ and $\bc = \bD^{1/2} \be$. 
The solutions of RatioCut and NCut are the second least eigenvectors of $\bL$ and $\tilde{\bL}$, respectively.
\begin{figure}[t]
\centering
\subfloat{
\includegraphics[width=0.11\textwidth]{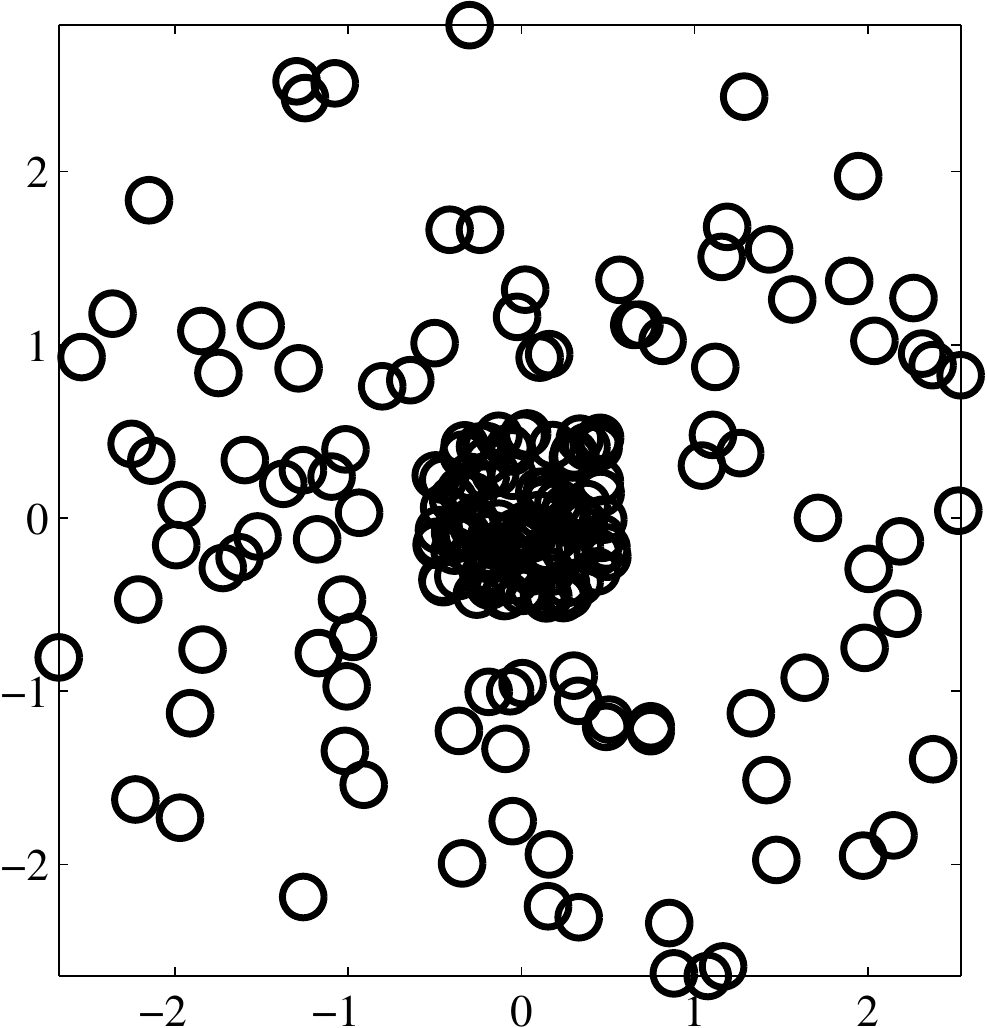}
\includegraphics[width=0.11\textwidth]{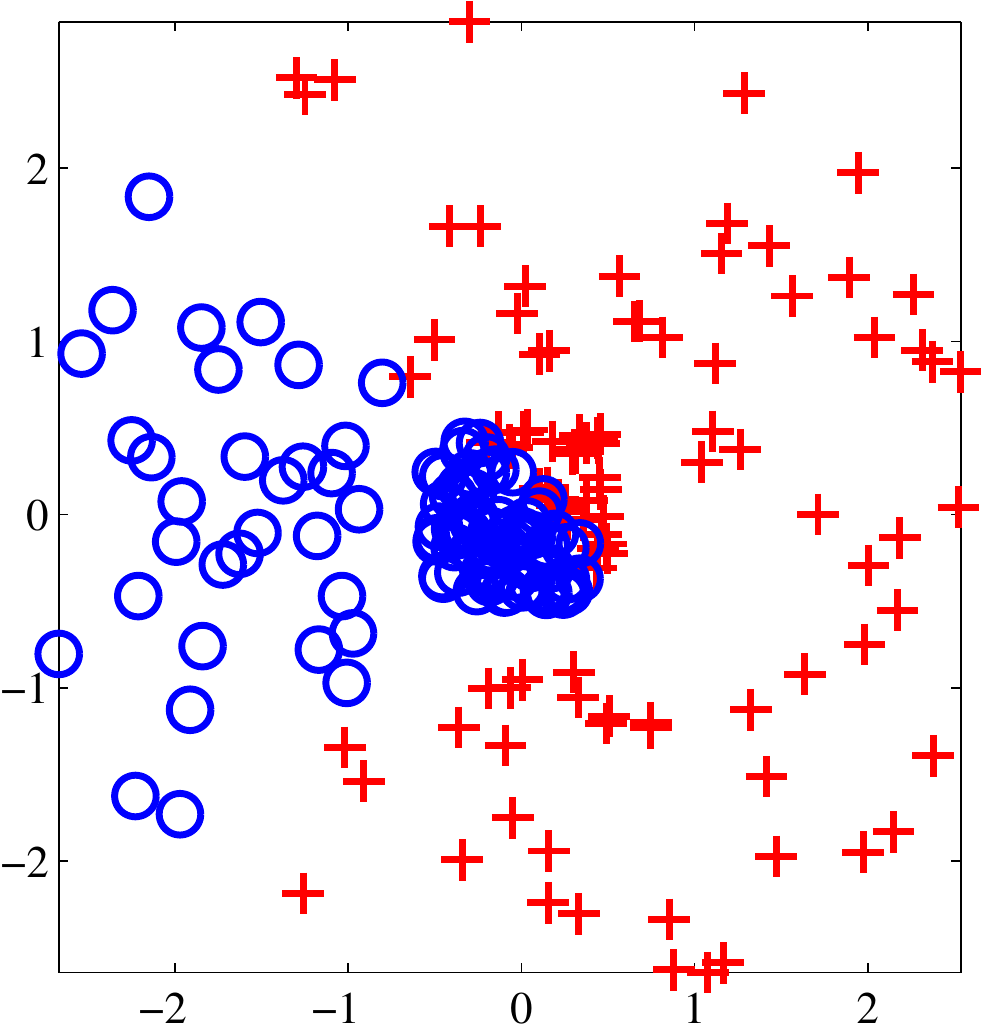}
\includegraphics[width=0.11\textwidth]{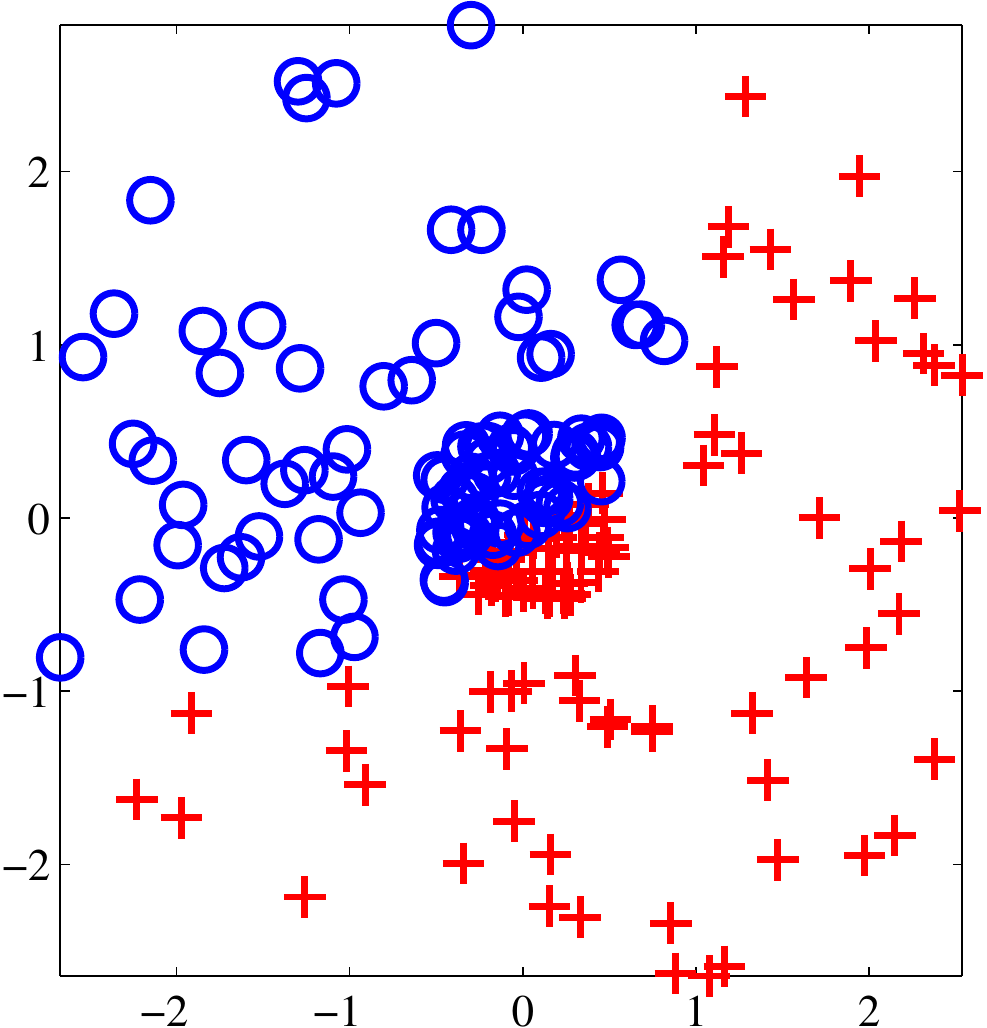}
\includegraphics[width=0.11\textwidth]{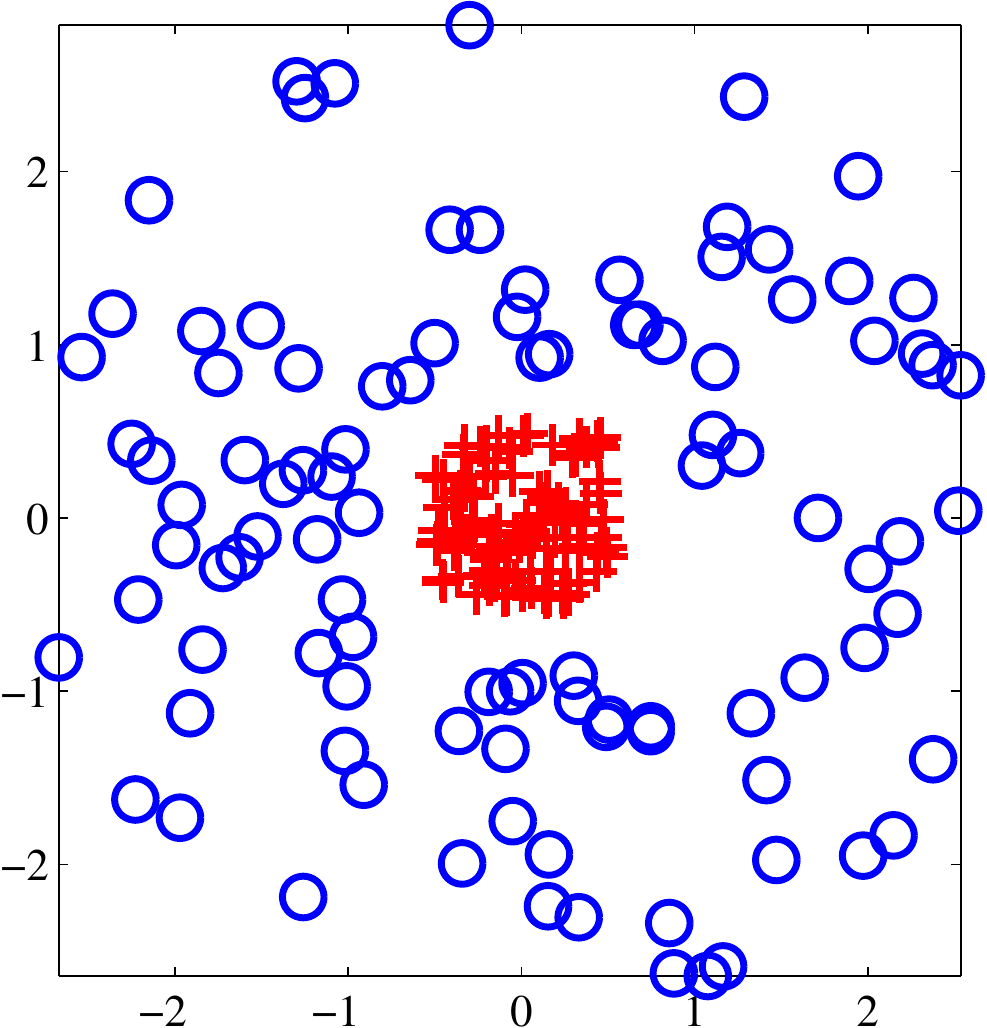}
\centering
}\\
\subfloat{
\centering
\includegraphics[width=0.11\textwidth]{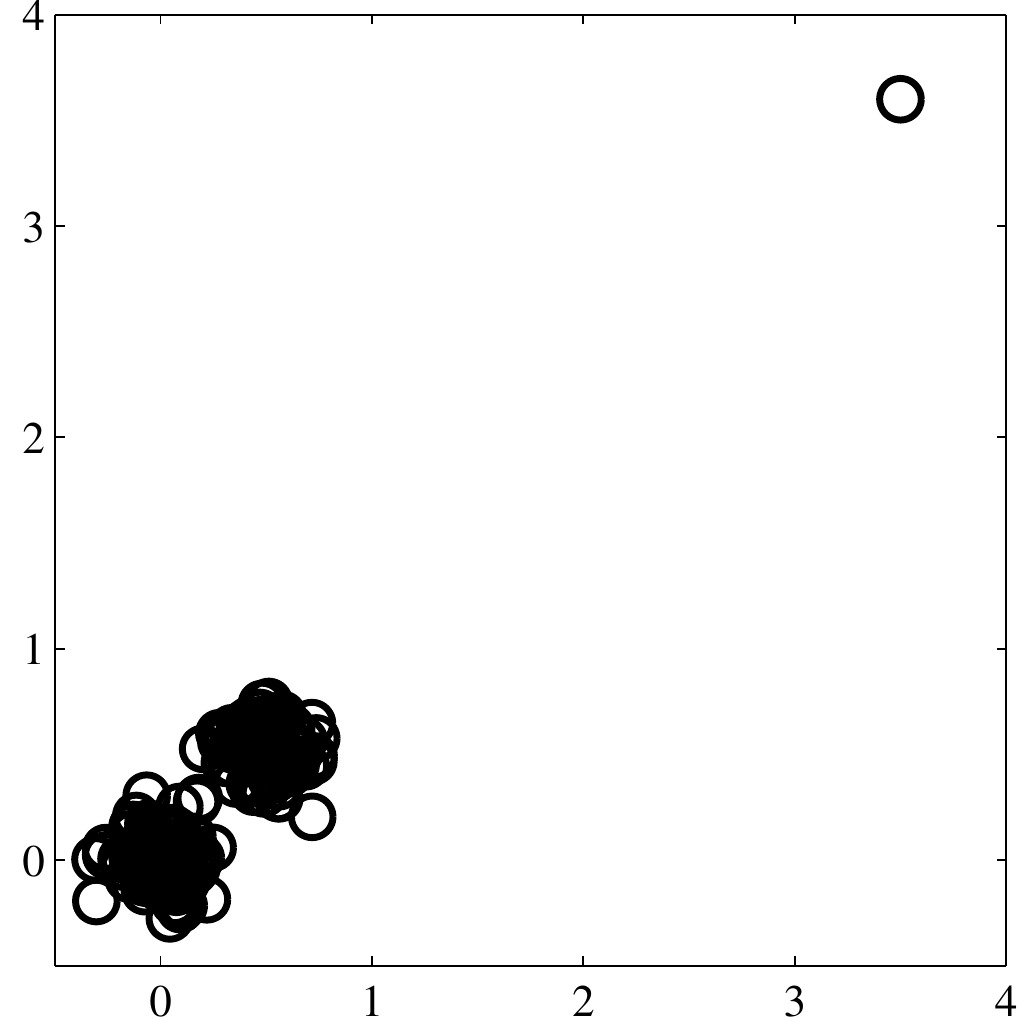} 
\includegraphics[width=0.11\textwidth]{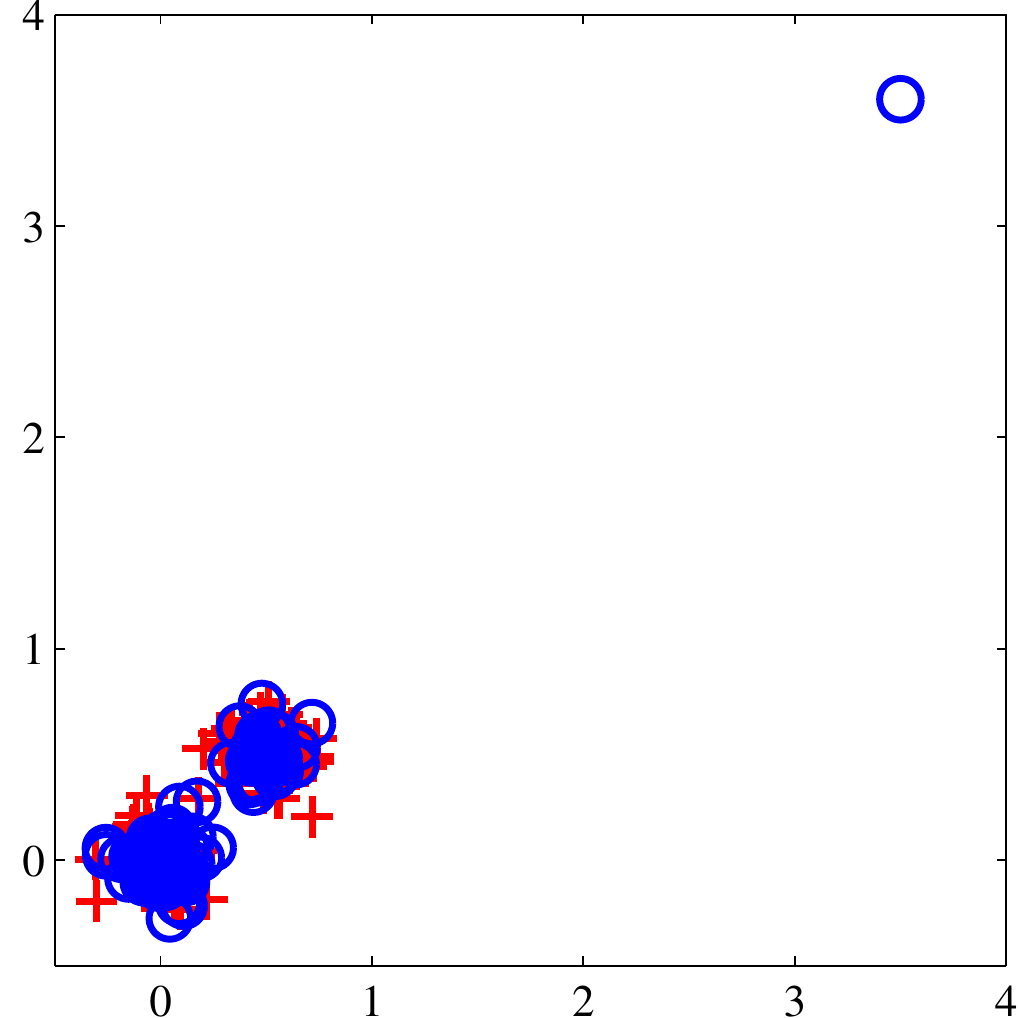}
\includegraphics[width=0.11\textwidth]{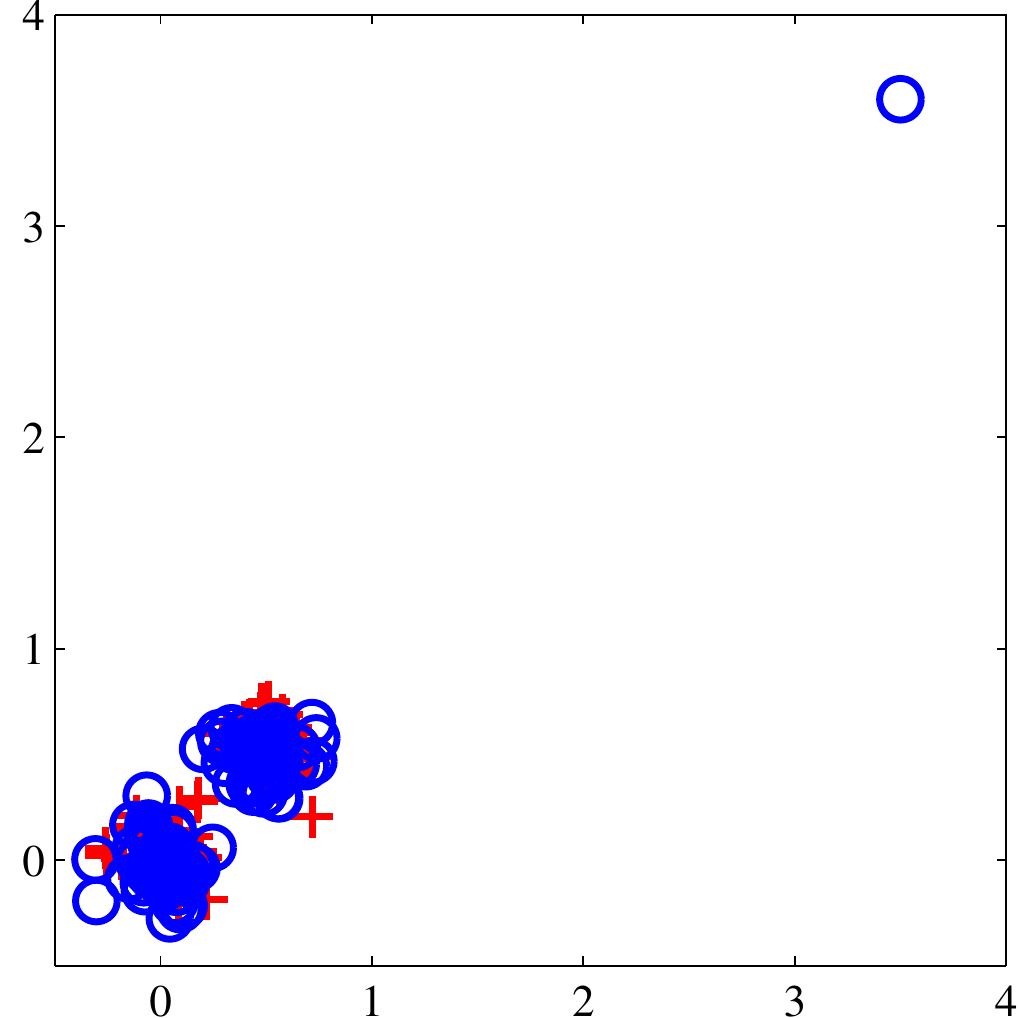}
\includegraphics[width=0.11\textwidth]{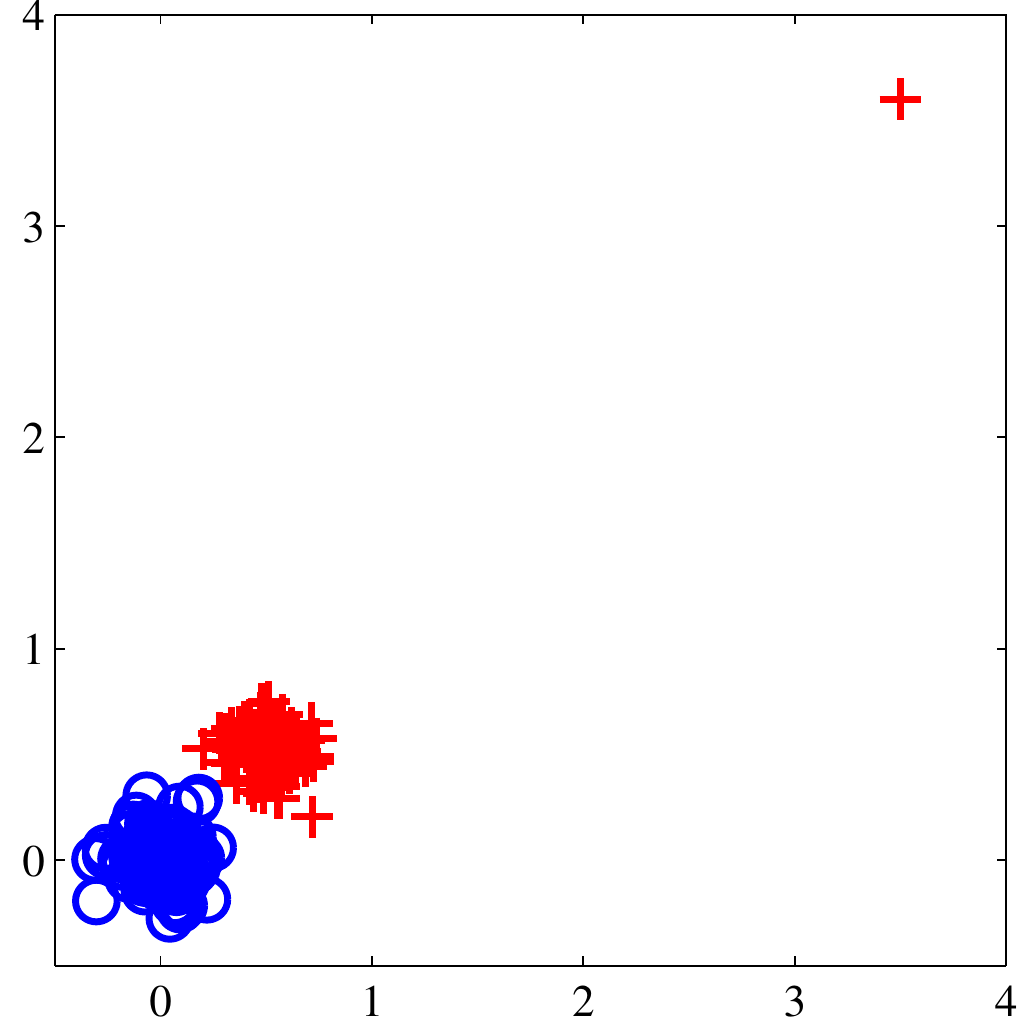}
}\\
{\footnotesize Original data} \hspace{0.8cm} {\footnotesize NCut} \hspace{1cm} {\footnotesize RatioCut} \hspace{1cm} {\footnotesize \fastsdp} 
\caption{Results of 2d points bisection. %
The thresholds are set to the median of score vectors. The two classes
of points are shown in red  `+'  and blue `$ \circ $'.
RatioCut and NCut fail to separate the points correctly, while \fastsdp succeeds.}
\label{fig:2d-cluster}
\end{figure}
\begin{figure*}[t]
\begin{minipage}[t]{0.31\linewidth}
    \centering
\includegraphics[width=.825\linewidth]{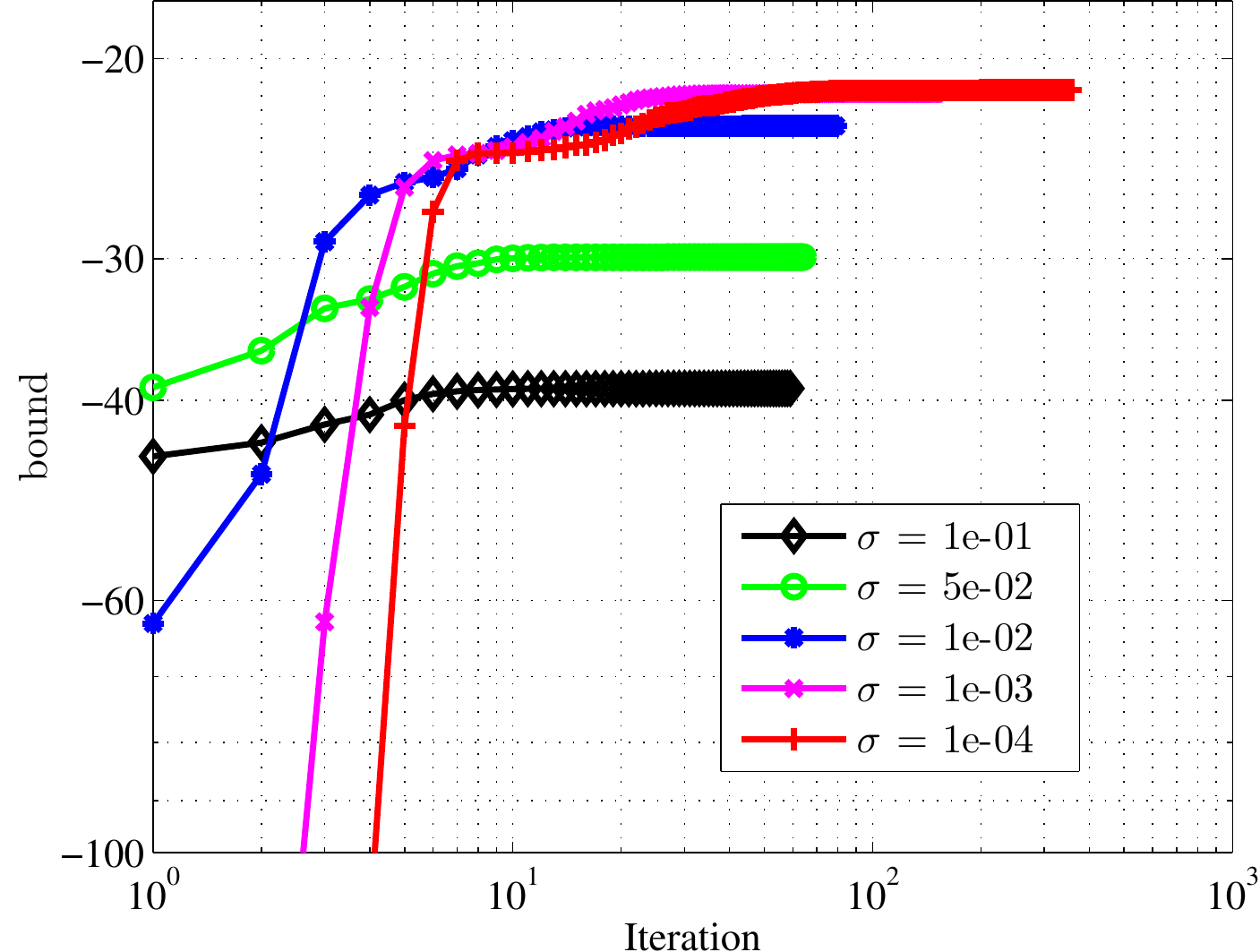}
\caption{The convergence of the objective value of the dual~\eqref{eq:fastsdp_dual}, which can be seen as a lower bound. 
         \fastsdp is tested to bisect a random graph with $200$ vertices and $0.5$ density. The bound is better when $\sigma$ is smaller.}
\label{fig:tst_sigma_obj}
\end{minipage}
\hspace{0.5cm}
\begin{minipage}[t]{0.64\linewidth}
\centering
\includegraphics[width=0.4\textwidth]{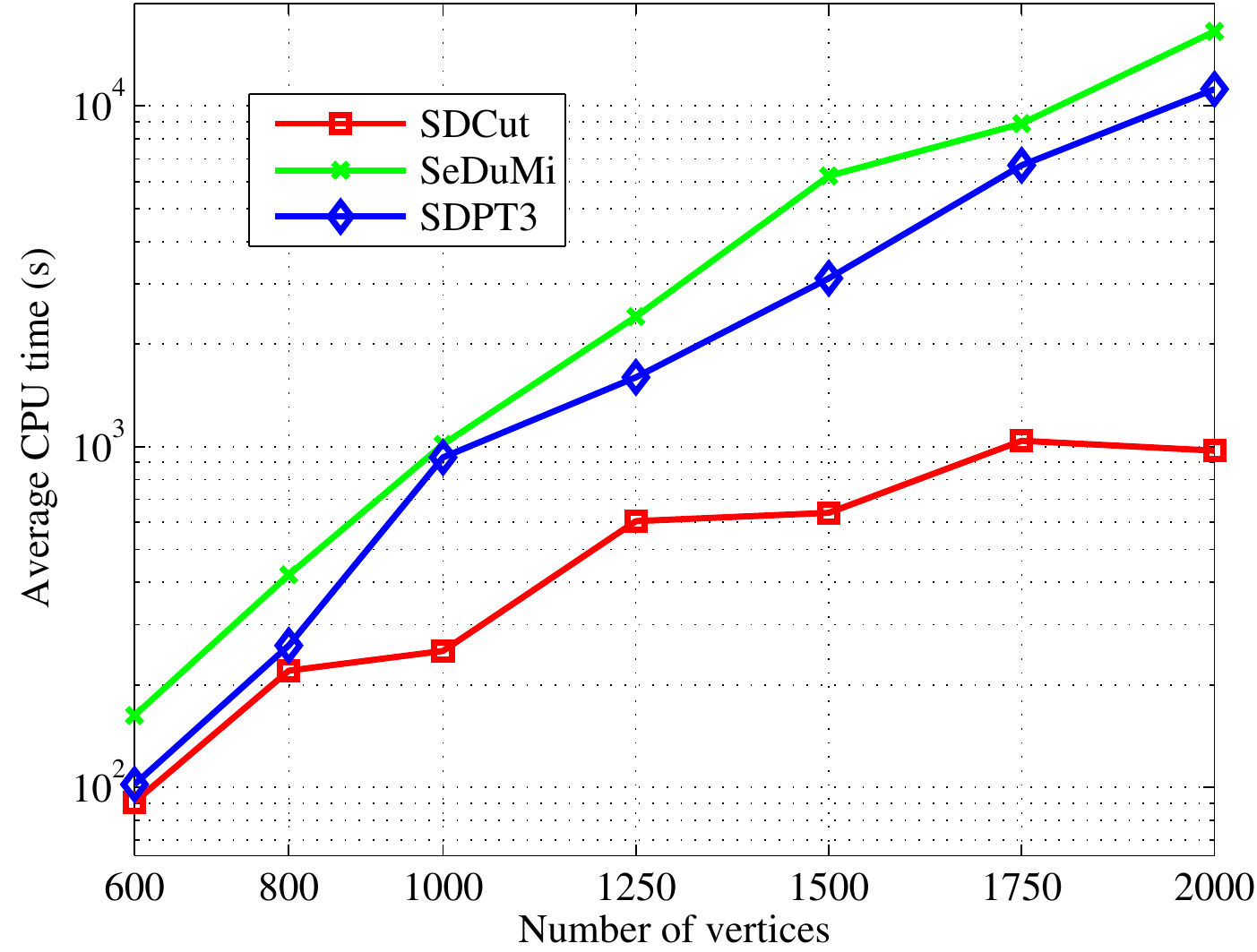}
\includegraphics[width=0.4\textwidth]{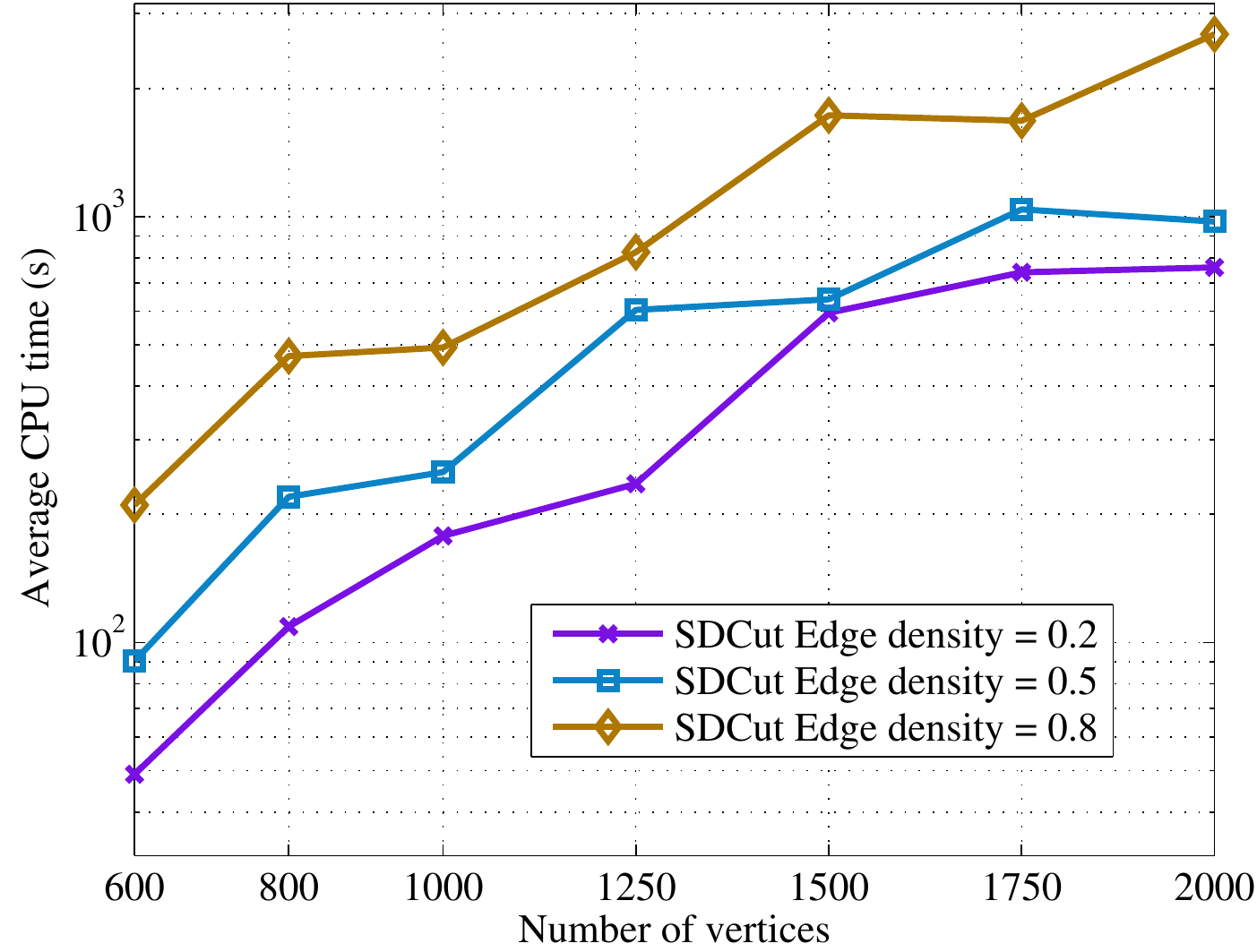}
\caption{Computation time for graph bisection. All the results are the average of $5$ random graphs.
Left: Comparison of \fastsdp, SeduMi and SDPT3. 
Right: Comparison of \fastsdp under different edge densities. $\sigma$ is set to $10^{-3}$ in this case.
\fastsdp is much more faster than the conventional SDP methods, and is faster when the graph is sparse.}
\label{fig:tst_time}
\end{minipage}
\end{figure*}
For~\eqref{eq:gb_1}, $\bX = \bx \bx^{\T}$ satisfies:
\begin{align}
\mathbf{diag}(\bX) = \be, \ \text{and} \ \langle \bX, \be \be^{\T} \rangle = 0. \label{eq:graph_bisect_cons_1}
\end{align}
Since $\bx^{\T} \bD \bx$ is constant for $\bx \in \{-1,1\}^n$, we have
\begin{align}
\min_{\bx \in \{-1,1\}^n} \bx^{\T} \bL \bx \Longleftrightarrow \min_{\bx \in \{-1,1\}^n}  \bx^{\T} (-\bW) \bx.
\end{align}
By substituting $-\bW$ and the constraints~\eqref{eq:graph_bisect_cons_1} into~\eqref{eq:backgd_sdp1} and~\eqref{eq:fastsdp_analysis3},
we then have the formulation of the conventional SDP method and \fastsdp.
To obtain the discrete result from the solution $\bX^\star$, %
we adopt the randomized rounding method in~\cite{Goemans95improved}:
a score vector $\bx^{\star}_r$ is generated from a Gaussian distribution with mean $0$ and covariance $\bX^\star$, 
and the discrete vector $\bx^\star \in \{ -1,1\}^n$ is obtained by thresholding $\bx^{\star}_r$ with its median. 
This process is repeated several times and the final solution  is the one with the highest objective value.

{\bf Experiments}
To show the new SDP formulation has better solution quality than spectral relaxation, 
we compare the bisection results of RatioCut, NCut and \fastsdp %
on two artificial 2-dimensional data.
As shown in Fig.~\ref{fig:2d-cluster}, the data in the first row contain two point sets with different densities, and the second data
contain an outlier.
The similarity matrix $\bW$ is calculated based on the Euclidean distance of points $i$ and~$j$:
\begin{align}
 \bW_{ij} = \left\{ \begin{array}{ll} \mathrm{exp}(-\mathrm{d}(i,j)^2 / \gamma^2)  & \mbox{if }   \mathrm{d}(i,j) < r \\
                0, & \mbox{otherwise.} \end{array} \right.
\end{align}
The parameter $\gamma$ is set to $0.1$ of the maximum distance.
RatioCut and NCut fail to offer satisfactory results on both of the
data sets, possibly due to the loose bound of spectral relaxation.
Our \fastsdp achieves better results on these data sets.

\begin{table}[t]
  \centering
  \footnotesize
  \begin{tabular}{c|ccccc}
  \hline
     $\sigma$       & bound & obj & norm & rank & iters\\
  \hline 
     $10^{-1}$ & $-39.04$  & $-20.55$ & $55.05$ & $18$ & $59$ \\
     $5\times 10^{-2}$ & $-29.91$  & $-20.92$ & $63.80$ & $14$ & $64$ \\
     $10^{-2}$ & $-22.93$  & $-21.26$ & $81.32$ & $9$  & $79$ \\
     $10^{-3}$ & $-21.45$  & $-21.29$ & $87.91$ & $7$  & $150$ \\
     $10^{-4}$ & $-21.31$  & $-21.31$ & $88.68$ & $7$  & $356$ \\
  \hline
  \end{tabular}
  \caption{Effect of $\sigma$. The lower bound, objective value $\langle \bX^\star, -\bW \rangle$, norm and rank of $\bX^\star$ and iterations are shown in each column.
           The number of variables is $19900$ for SDP problems. The results correspond to Fig.~\ref{fig:tst_sigma_obj}. 
           Better solution quality and more iterations are achieved when $\sigma$ becomes small.}
  \label{tab:sigma}
\end{table}

Moreover, to demonstrate the impact of the parameter~$\sigma$,
we test \fastsdp on a random graph with different $\sigma$'s.
The graph has $200$ vertices and its edge density is $0.5$: $50\%$ of edges are assigned with a weight uniformly sampled from $[0,1]$, 
the other half has zero-weights.
In Fig.~\ref{fig:tst_sigma_obj}, we show the convergence of the objective value of the dual~\eqref{eq:fastsdp_dual}, \ie a lower bound of the objective value of the problem~\eqref{eq:backgd_sdp1}.
A smaller $\sigma$ leads to a higher (better) bound.
The optimal objective value of the conventional SDP method is $-21.29$.
For $\sigma = 10^{-4}$, the bound of \fastsdp ($-21.31$) is very close to the SDP optmial.
Table~\ref{tab:sigma} also shows the objective value, the Frobenius norm and the rank of solution $\bX^\star$.
With the decrease of $\sigma$, the quality of the solution $\bX$ is further optimized (the objective value is smaller and the rank is lower).
However, the price of higher quality is the slow convergence speed: more iterations are needed for a smaller $\sigma$.

Finally, experiments are performed to compare the computation time under different conditions.
All the times shown in Fig.~\ref{fig:tst_time} are the mean of $5$ random graphs when $\sigma$ is set to $10^{-3}$.
\fastsdp, SeDuMi and SDPT3 are compared with graph sizes ranging from $600$ to $2000$ vertices.
Our method is faster than SeDuMi and SDPT3 on all graph sizes. When the problem size is larger, the speedup is more significant.
For graphs with $2000$ vertices, 
\fastsdp runs $11.5$ times faster than SDPT3 and $17.0$ times faster than SeDuMi.
The computation time of \fastsdp is also tested under $0.2$, $0.5$ and $0.8$ edge density.
Our method runs faster for smaller edge densities, which validates that our method can take the advantage of graph sparsity.

We also test the memory usage of MATLAB for \fastsdp, SeDuMi and SDPT3.
Because L-BFGS-B and ARPACK use limited memory, the total memory used by our method is also relatively small. 
Given a graph with $1000$ vertices, \fastsdp requires $100$MB memory, while SeDuMi and SDPT3 use around $700$MB.

\subsection{Application 2: Image Segmentation}

\begin{figure*}[t]
\vspace{-0.0cm}
\centering
\subfloat{
\centering
\includegraphics[width=0.1\textwidth]{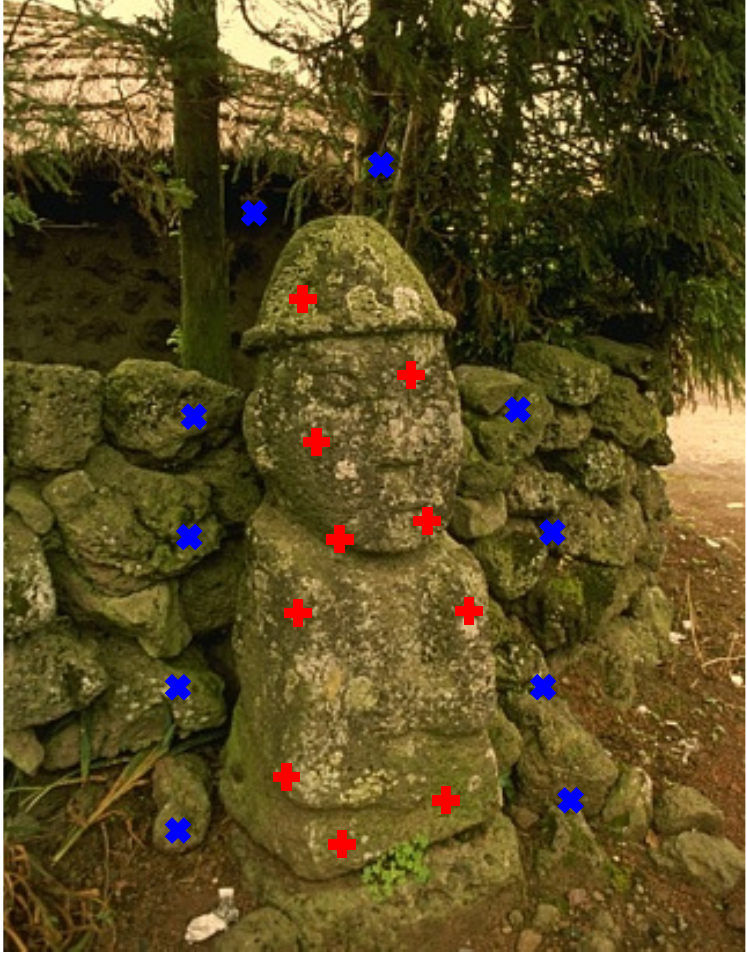}
\includegraphics[width=0.192\textwidth]{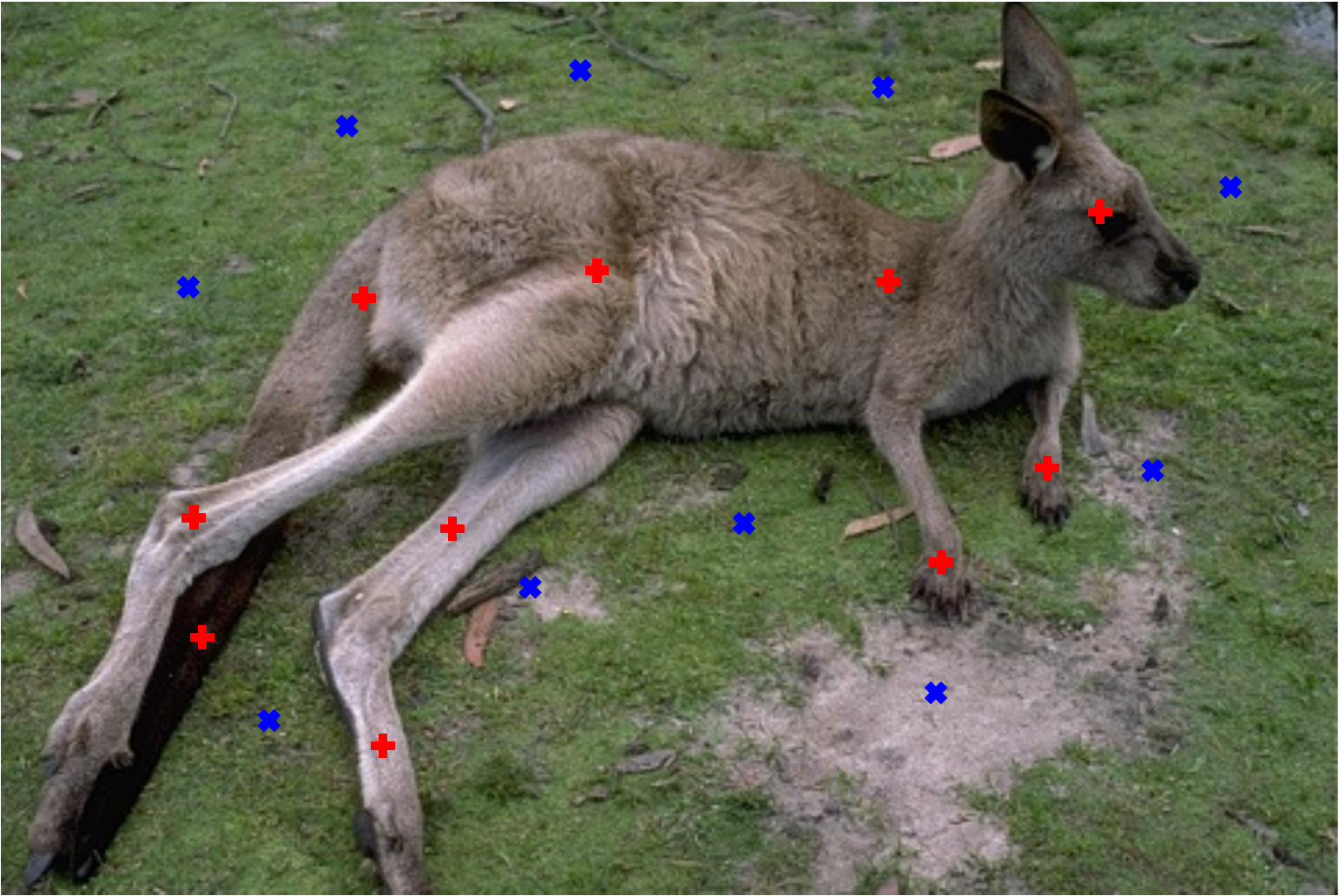}
\includegraphics[width=0.192\textwidth]{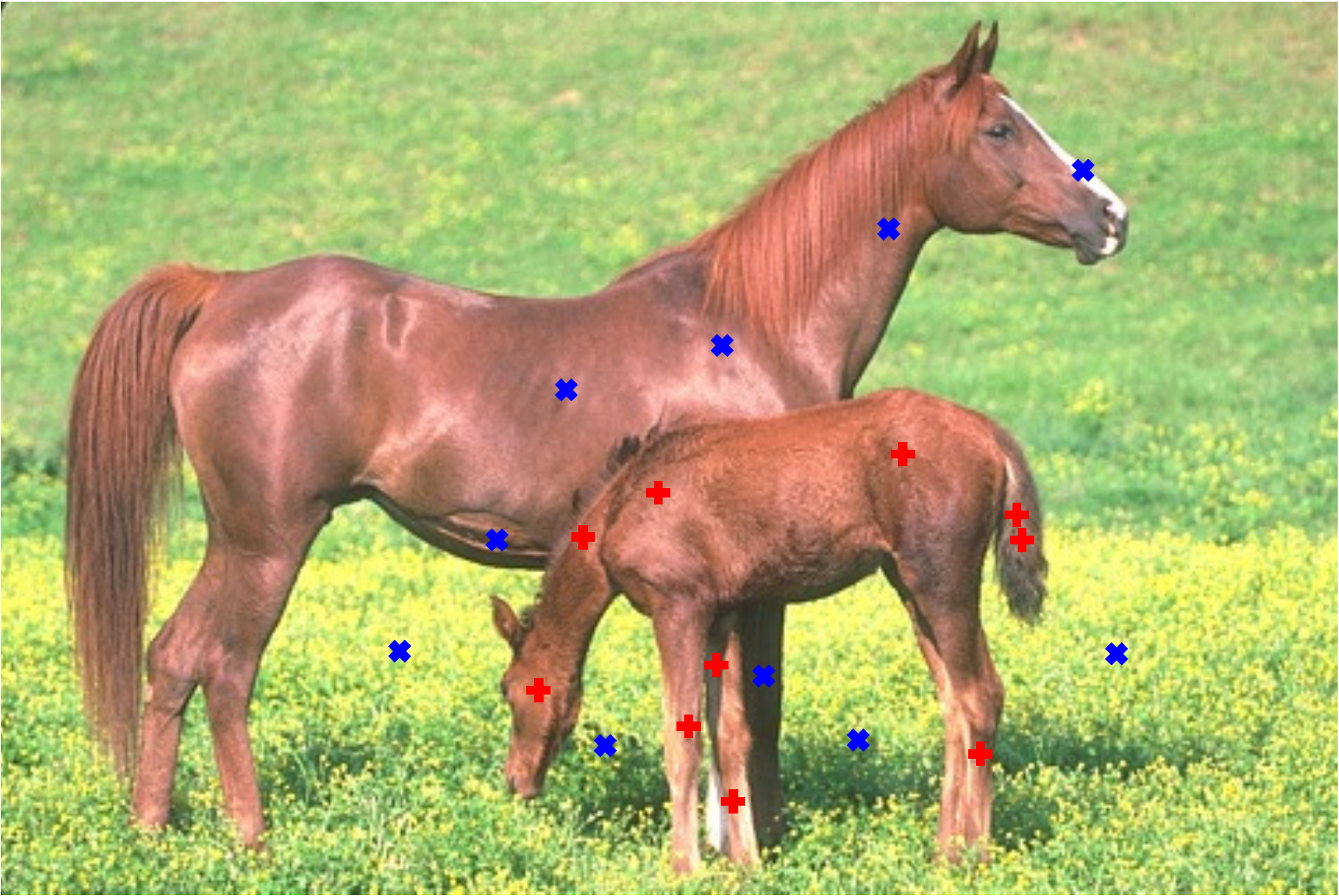}
\includegraphics[width=0.192\textwidth]{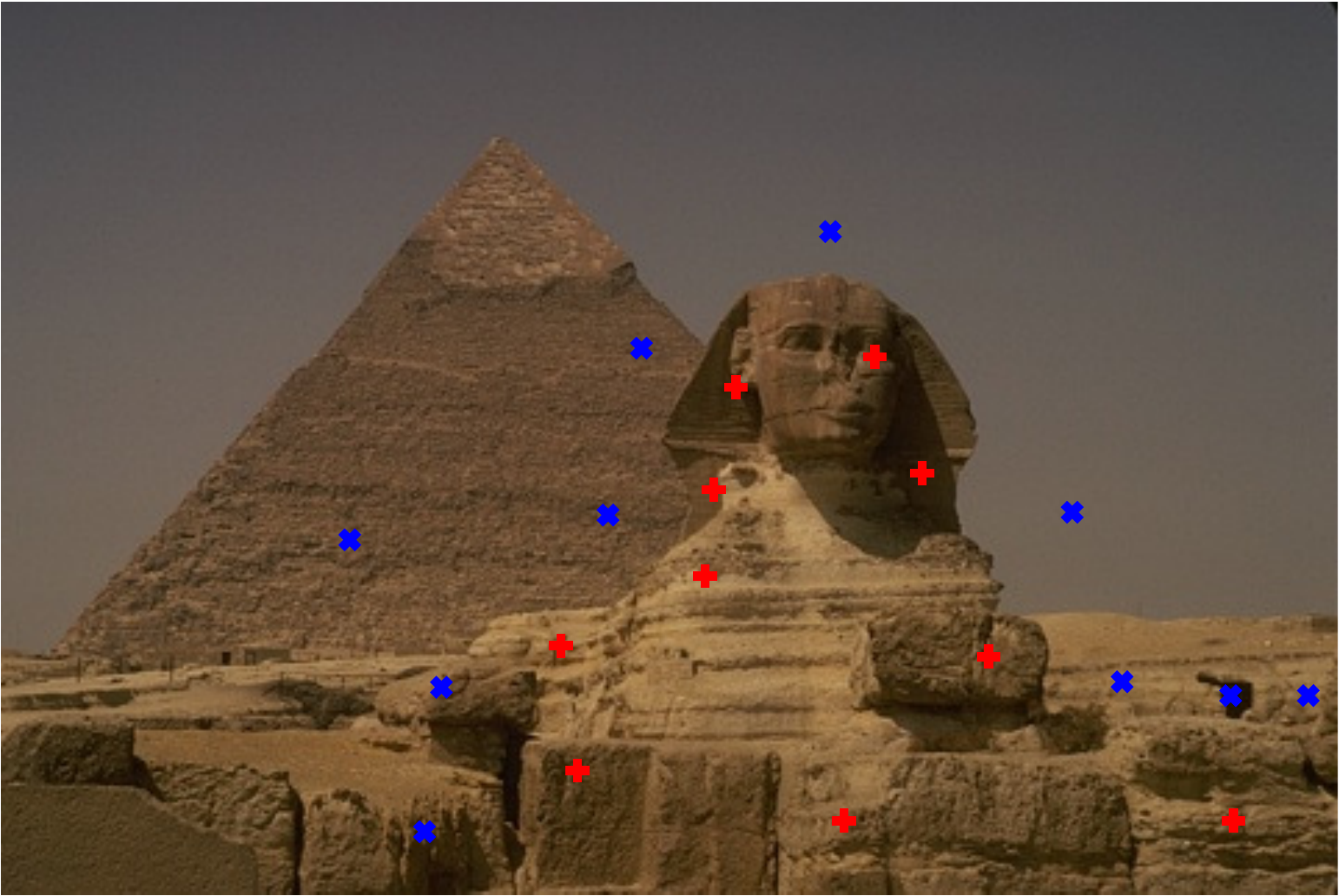}
\includegraphics[width=0.156\textwidth]{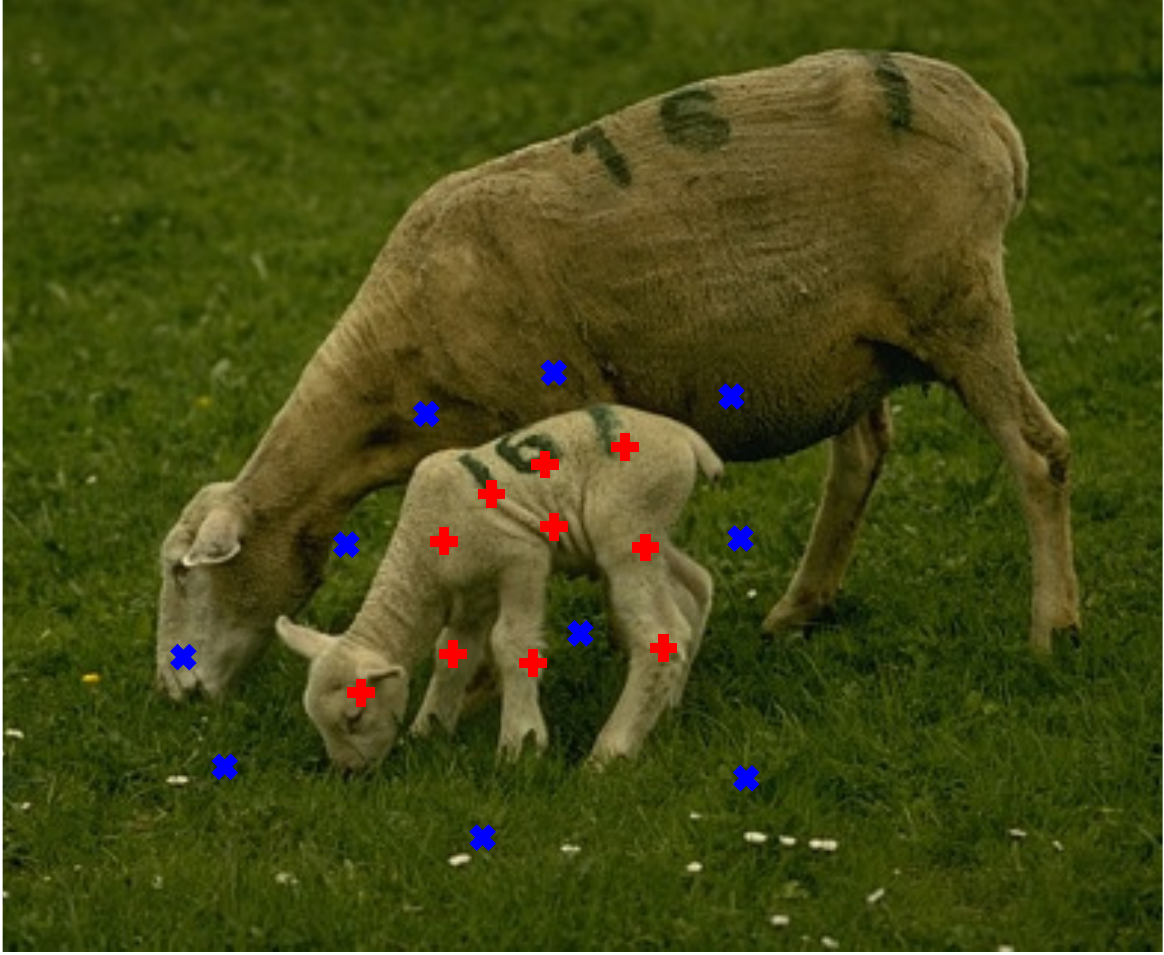}}\\
\subfloat{
\centering
\includegraphics[width=0.1\textwidth]{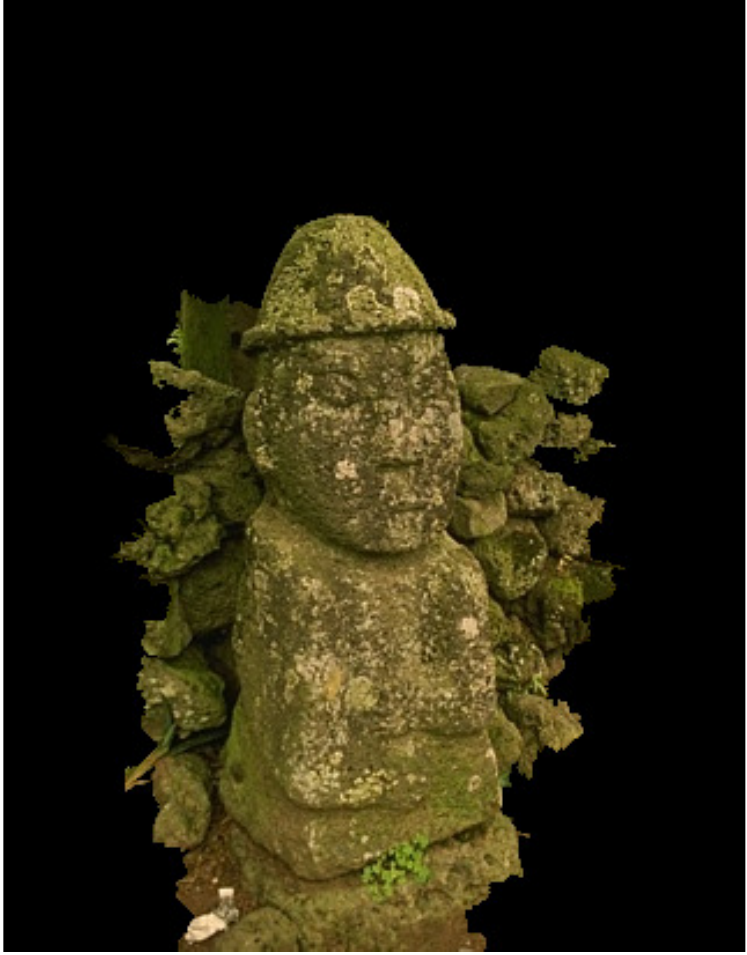}
\includegraphics[width=0.192\textwidth]{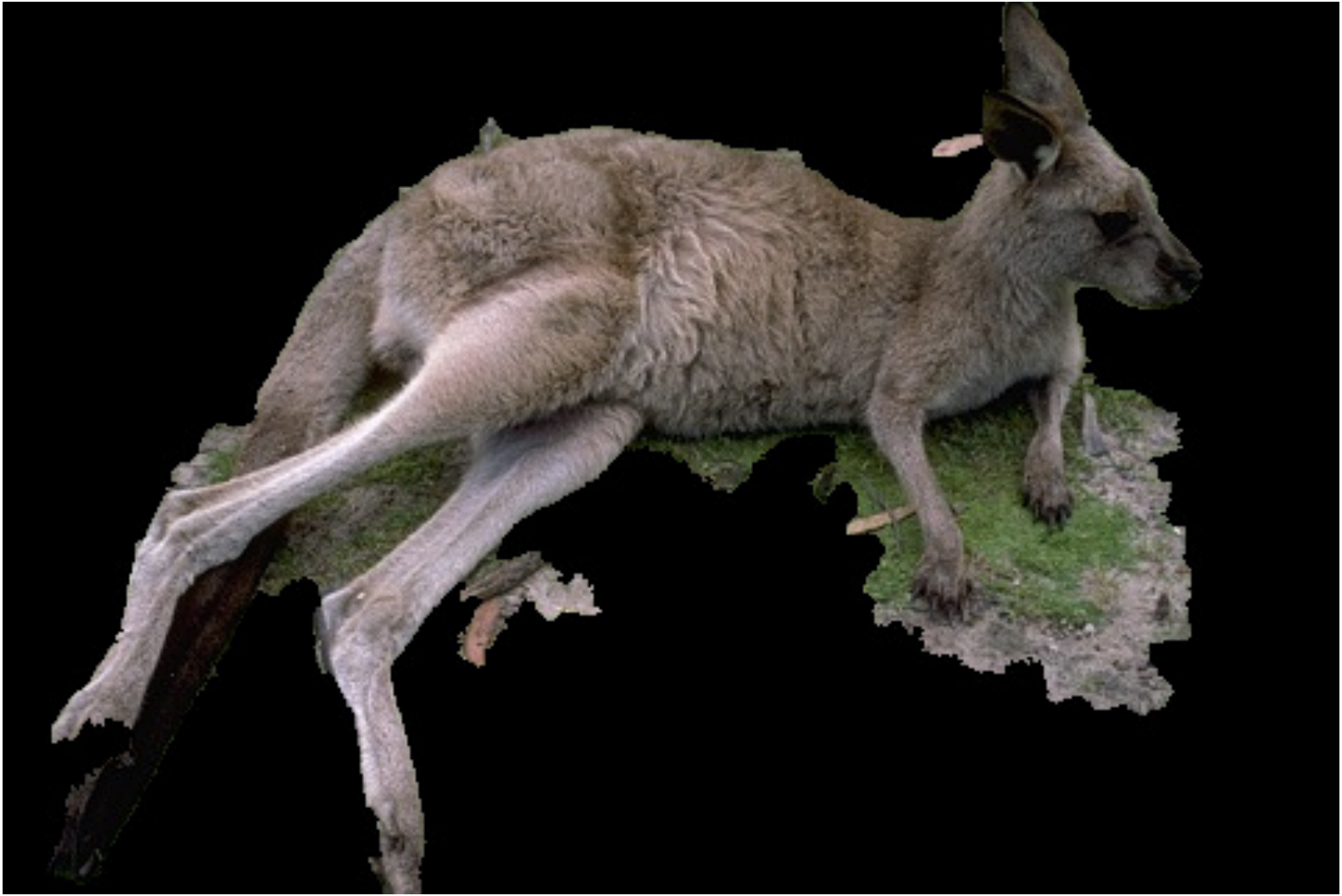}
\includegraphics[width=0.192\textwidth]{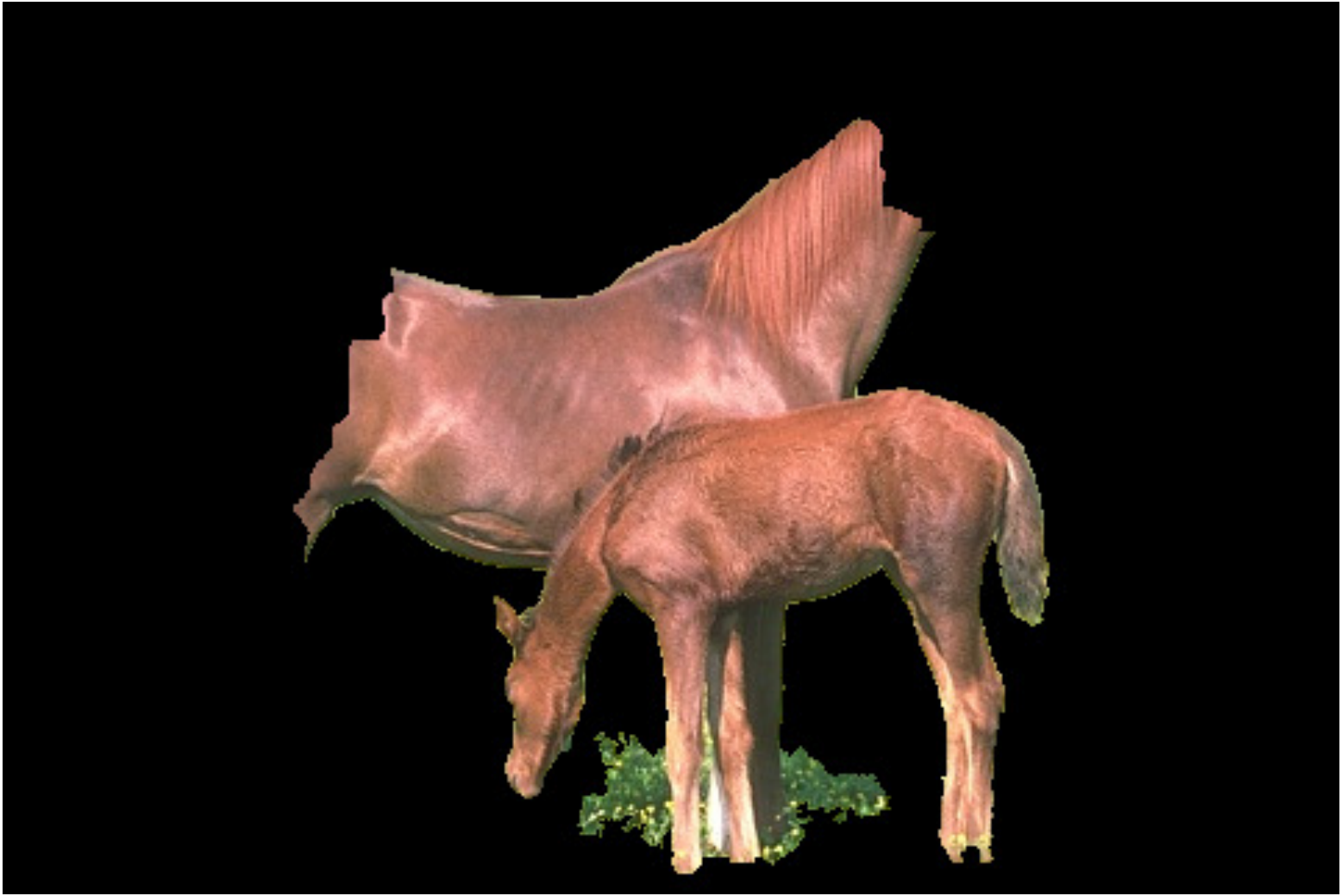}
\includegraphics[width=0.192\textwidth]{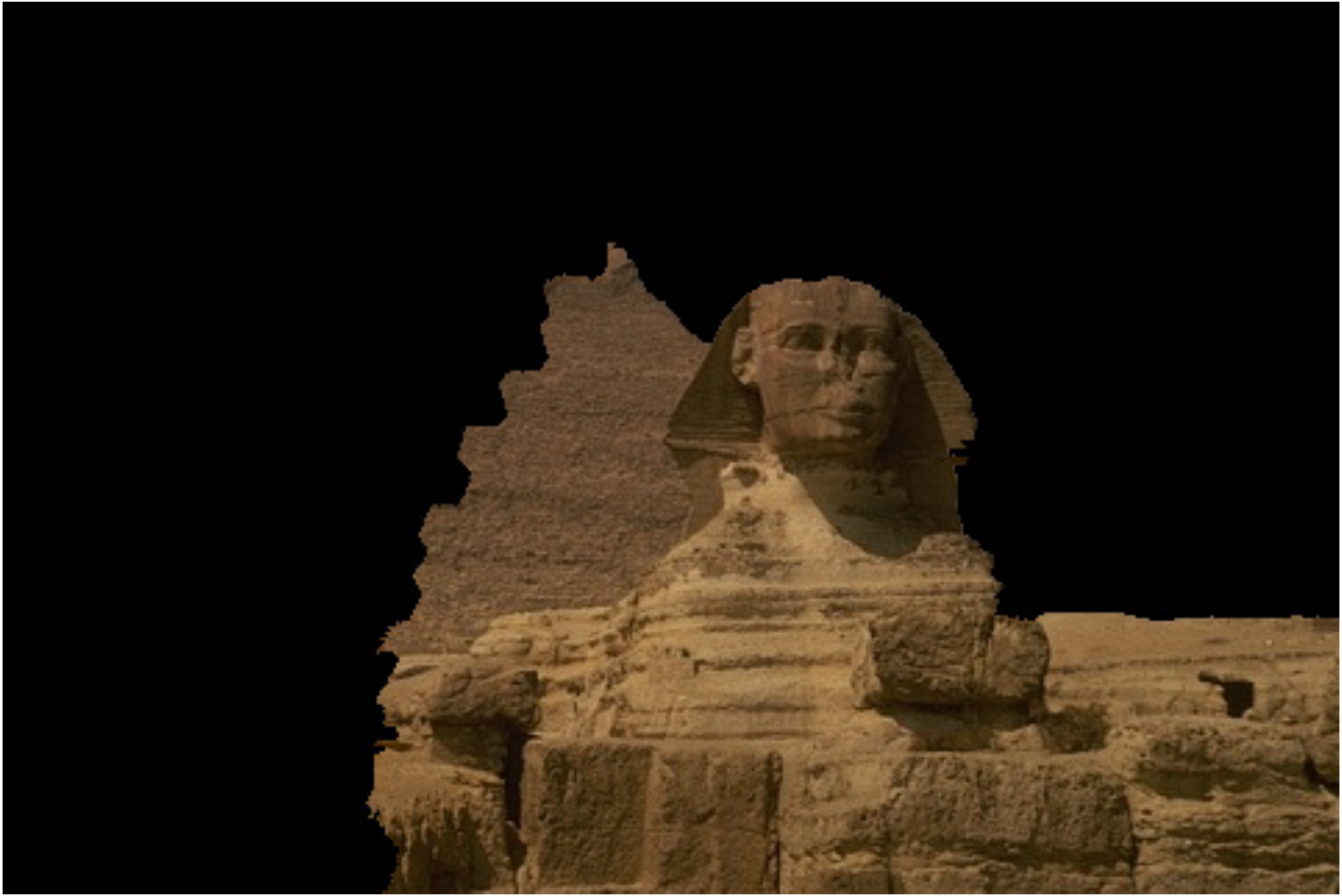}
\includegraphics[width=0.156\textwidth]{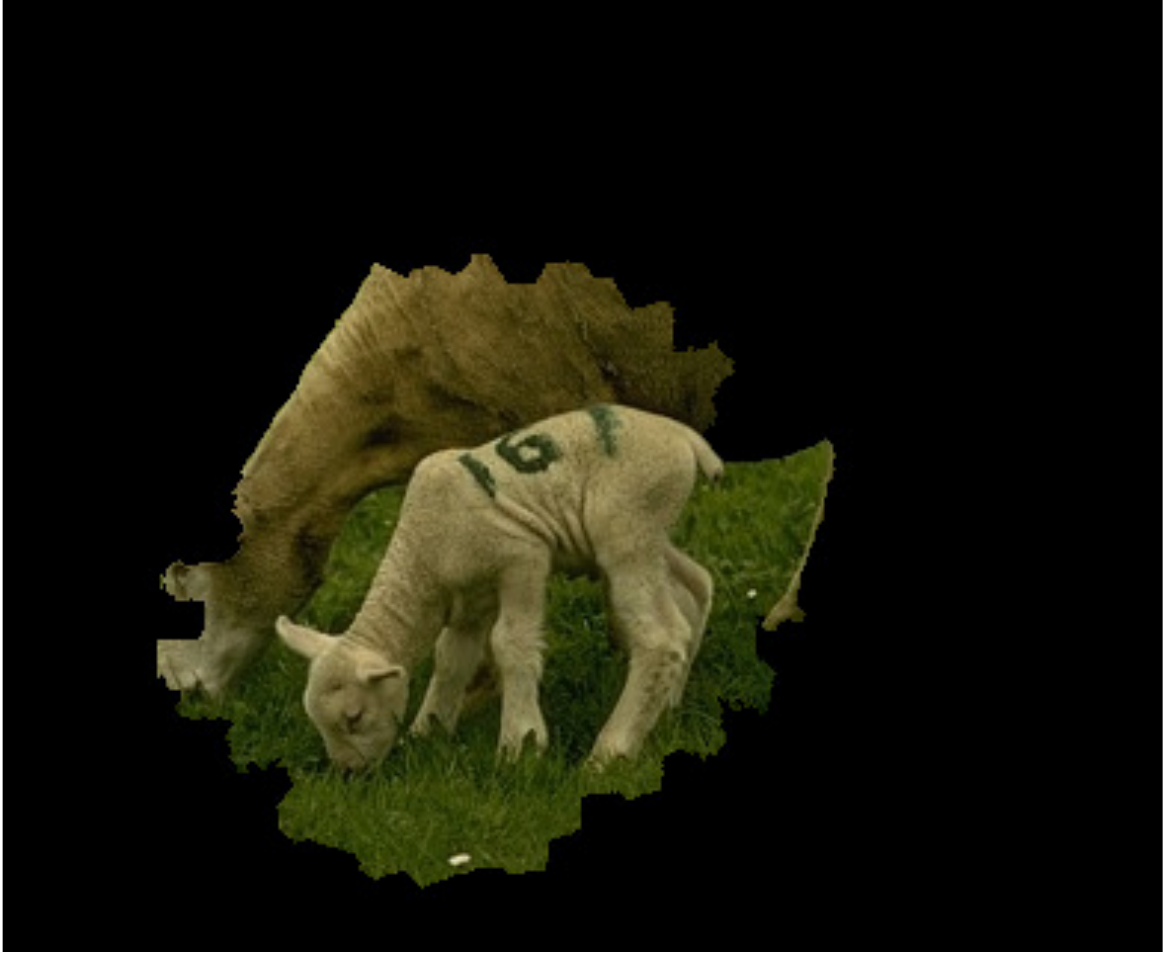}}\\
\subfloat{
\centering
\includegraphics[width=0.1\textwidth]{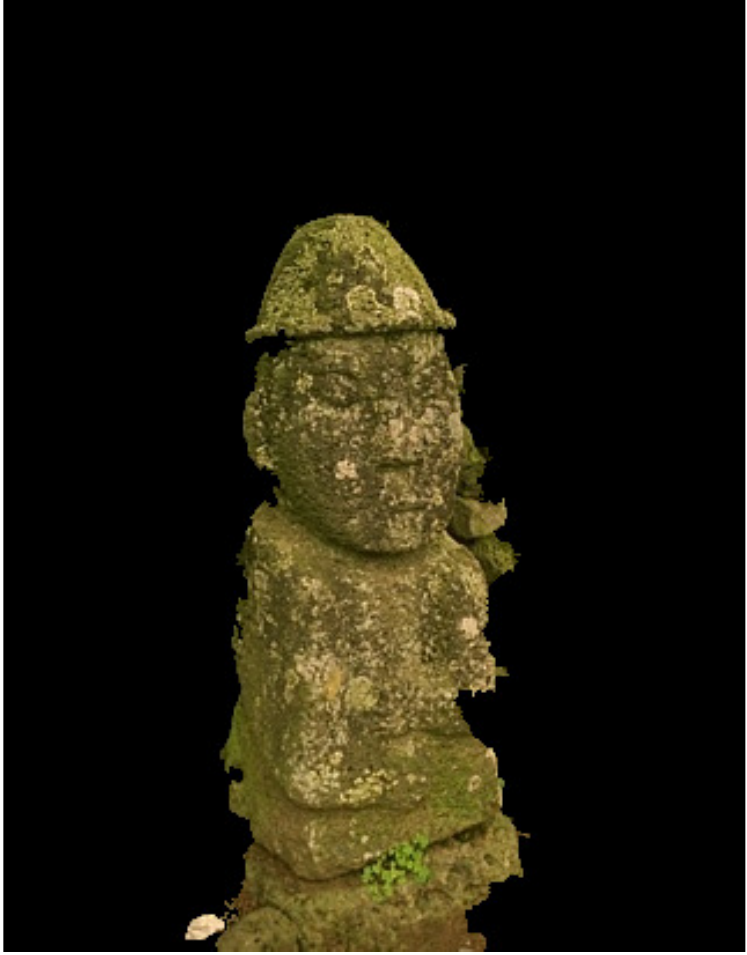}
\includegraphics[width=0.192\textwidth]{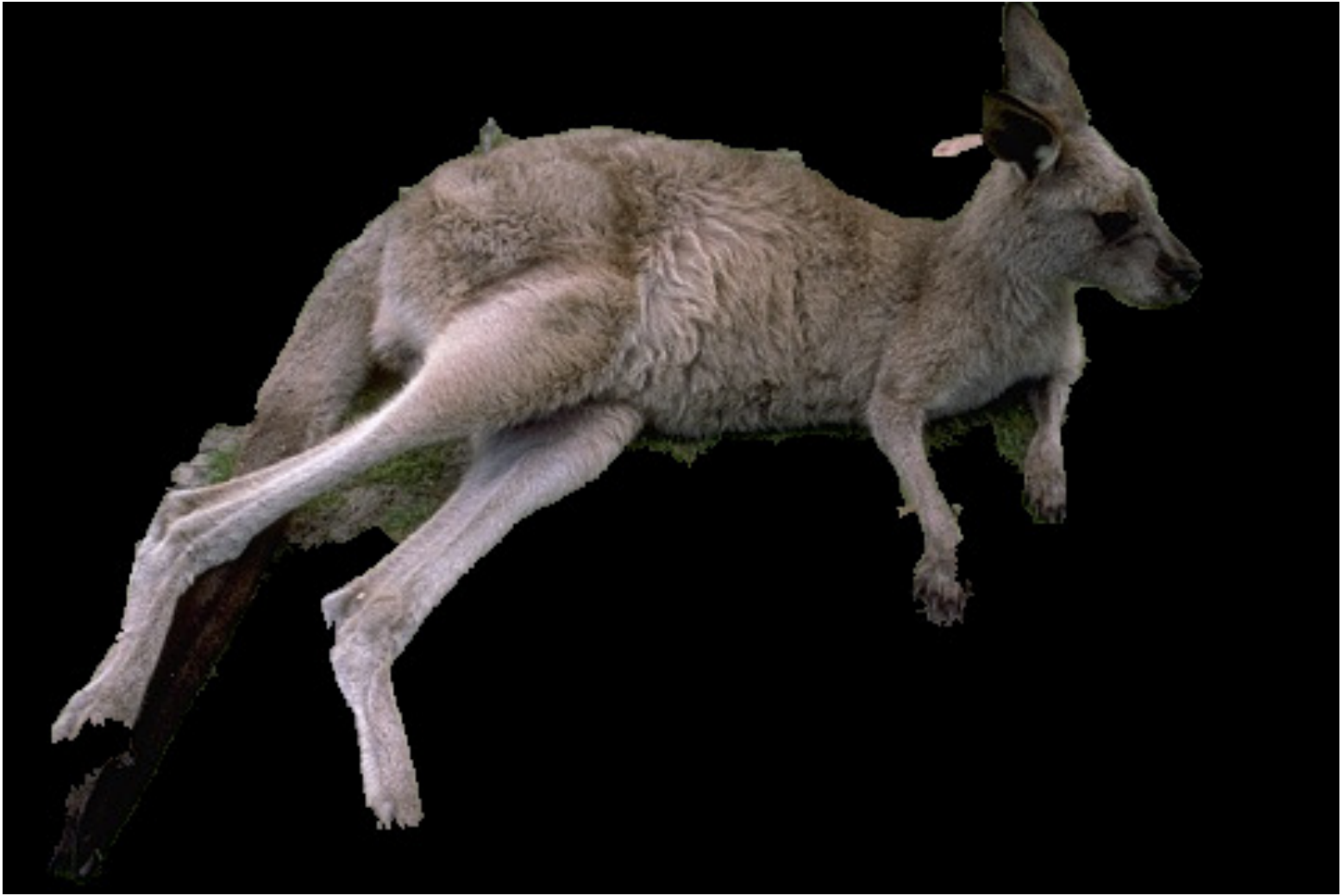}
\includegraphics[width=0.192\textwidth]{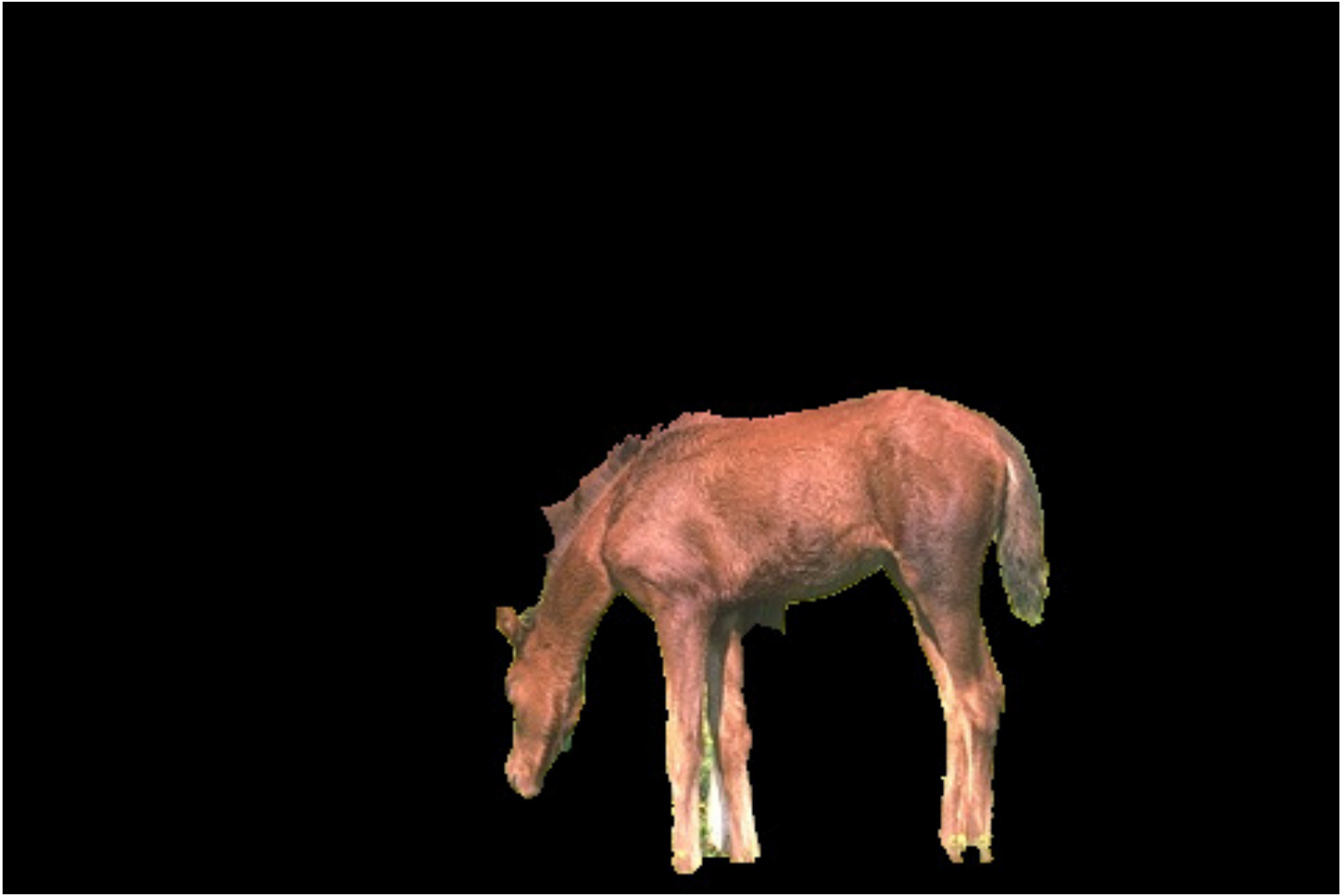}
\includegraphics[width=0.192\textwidth]{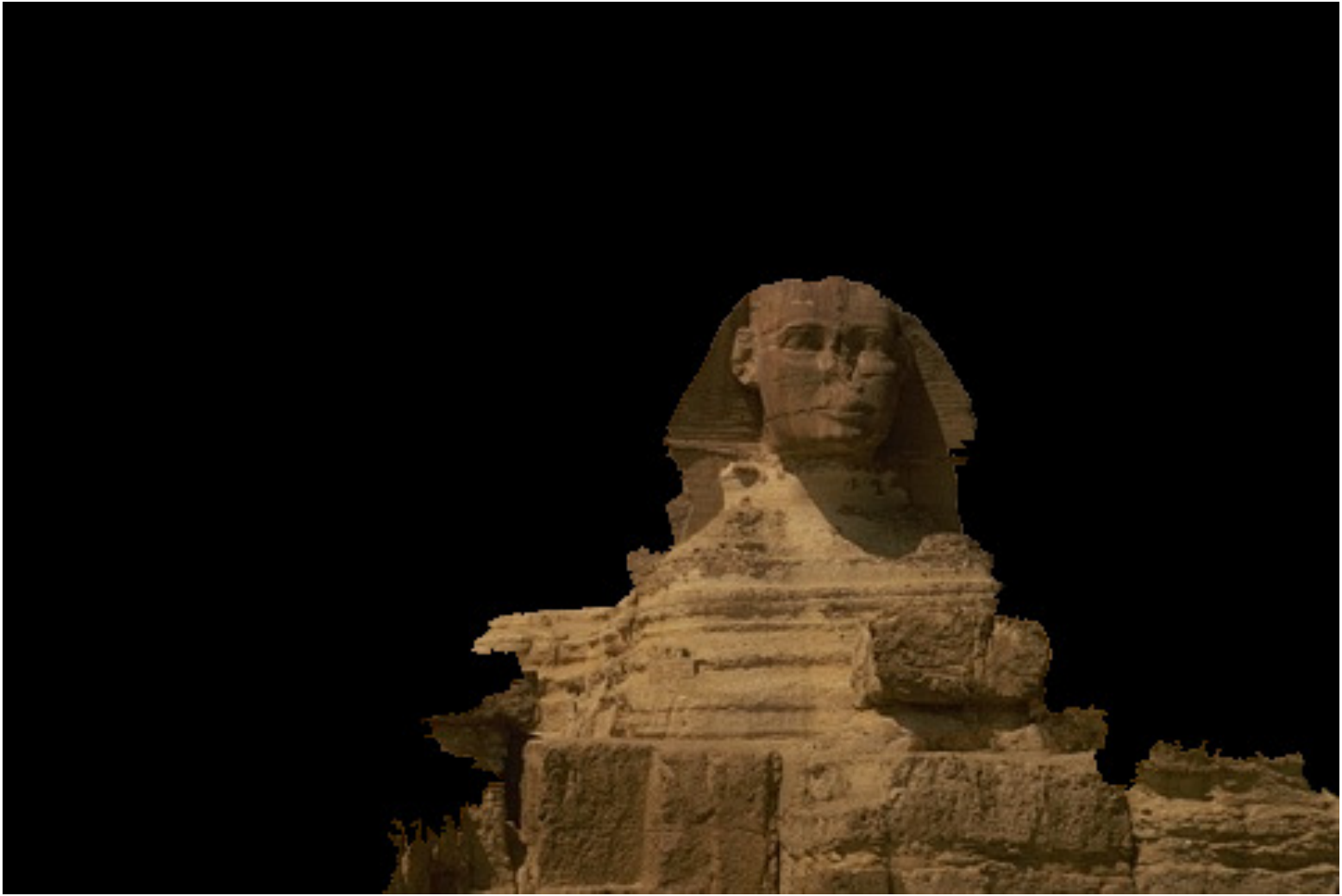}
\includegraphics[width=0.156\textwidth]{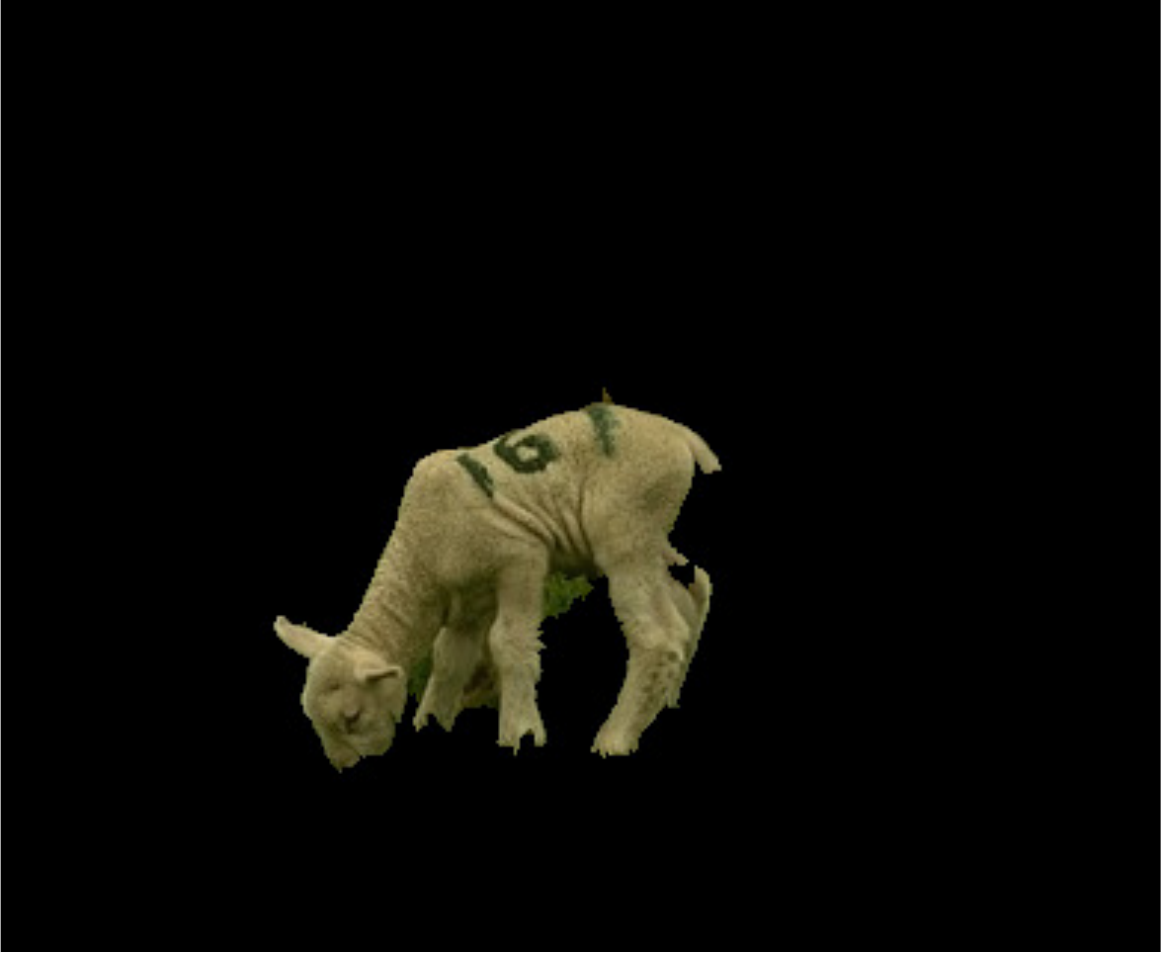}}
\vspace{-0.0cm}
\caption{Segmentation results on the Berkeley dataset. The top row shows
the original images with partial labelled pixels. 
         Our method (bottom) achieves better results than BNCut (middle).          }
\label{fig:img_segm}
\end{figure*}

{\bf Formulation}
In graph based segmentation, images are represented by weighted graphs $G(V,E)$,
with vertices corresponding to pixels and edges encoding feature similarities between pixel pairs.
A partition $\bx \in \{-1,1\}^n$ is optimized to cut the minimal edge weights and results into two balanced disjoint groups.
Prior knowledge can be introduced to improve performance, encoding by labelled vertices of a graph, \ie, pixels/superpixels in an image.
As shown in the top line of Fig.~\ref{fig:img_segm}, $10$ foreground pixels and $10$ background pixels are annotated by red and blue markers respectively.
Pixels should be grouped together if they have the same color; otherwise they should be separated.

Biased normalized cut (BNCut)~\cite{Maji2011biased} is an extension of NCut~\cite{Shi2000normalized}, 
which considers the partial group information of labelled foreground pixels.
Prior knowledge is encoded as a quadratic constraint on $\bx$.
The result of BNCut is a weighted combination of the eigenvectors of normalized Laplacian matrix.
One disadvantage of BNCut is that at most one quadratic constraint can be incorporated into its formulation.
Furthermore, no explicit results can be obtained: the weights of
eigenvectors must be tuned by the user.
In our experiments, we use the parameters suggested in~\cite{Maji2011biased}.

Unlike BNCut, \fastsdp can incorporate multiple quadratic constraints on $\bx$.
In our method, the partial group constraints of $\bx$ are formulated as:
$(\bt_f^{\T} \bP \bx)^2 $ $ \geq $ $ \kappa \lVert \bt_f^{\T} \bP \rVert^2_1$, 
$(\bt_b^{\T} \bP \bx)^2 $ $ \geq  $ $ \kappa \lVert \bt_b^{\T} \bP \rVert^2_1$ and 
$((\bt_f - \bt_b)^{\T} \bP \bx)^2 \geq \kappa \lVert (\bt_f - \bt_b)^{\T} \bP \rVert^2_1$,
where $\kappa \in [0,1]$ . $\bt_f, \bt_b \in \{0,1\}^n$ are the indicator vectors of foreground and background pixels.
$\bP = \bD^{-1}\bW$ is the normalized affinity matrix, which smoothes the partial group constraints~\cite{Yu2004segmentation}.
After lifting, the partial group constraints are:
\begin{subequations}
    \label{EQ:SEGM1}
\begin{align}
\langle \bP \bt_f \bt_f^{\T} \bP, \bX \rangle   &\geq  \kappa \lVert \bt_f^{\T} \bP \rVert_1^2, \label{eq:img_segm_cons1}\\
\langle \bP \bt_b \bt_b^{\T} \bP, \bX \rangle   &\geq  \kappa \lVert \bt_b^{\T} \bP \rVert_1^2, \label{eq:img_segm_cons2}\\
\langle \bP (\bt_f - \bt_b) (\bt_f - \bt_b)^{\T} \!\bP, \bX \rangle   &\geq   \kappa \lVert (\bt_f - \bt_b)^{\T} \bP \rVert_1^2. \label{eq:img_segm_cons3}
\end{align}
\end{subequations}
We have the formulations of the standard SDP and \fastsdp, 
with constraints \eqref{eq:graph_bisect_cons_1} and \eqref{EQ:SEGM1}
for this particular application.
%
%
%
%
The standard SDP~\eqref{eq:backgd_sdp1} is solved by SeDuMi and SDPT3.

Note that constraint~\eqref{eq:graph_bisect_cons_1} enforces the {\it
equal partition}; after rounding, this equal partition may
only be partially satisfied, though. 
We still use the method in~\cite{Goemans95improved} to generate a
score vector, and the threshold is set to $0$ instead of median.

\begin{table}[t]
  \centering
  \footnotesize
  \begin{tabular}{l|cccc}
  \hline
     Methods             &  BNCut & \fastsdp & SeDuMi & SDPT3\\
  \hline
  \hline  
     Time(s)               & $0.258$   & $23.7$   & $372$    & $329$   \\
     obj       & $-112.55$  & $-116.10$ & $-116.30$ & $-116.32$\\
  \hline
  \end{tabular}
  \caption{Results on image segmentation, 
           which are the mean of results of images in Fig.~\ref{fig:img_segm}.        
           \fastsdp has similar objective value with SeDuMi and SDPT3.
           $\sigma$ is set to $10^{-2}$. obj $= \langle {\bx^\star}{\bx^\star}^{\T}, -\bW \rangle$.  } %
  \label{tab:img_segm}
\end{table}

{\bf Experiments}
We test our segmentation method on the Berkeley segmentation dataset~\cite{MartinFTM01}.
Images are converted to Lab color space and over-segmented into SLIC superpixels using the VLFeat toolbox~\cite{vedaldi08vlfeat}.
The affinity matrix $\bW$ is constructed based on the color
similarities and spatial adjacencies between superpixels:
\begin{align}
\bW_{ij} = \left\{ 
                \begin{array}{ll} 
                \!\mathrm{exp}(- \frac{\lVert \mathbf{f}_i - \mathbf{f}_j \rVert_2^2}{\sigma_f^2} - \frac{\mathrm{d}(i,j)^2}{\sigma_d^2}) & \!\mbox{if }   \mathrm{d}(i,j) \!<\! r,\\
                \!0 &\! \mbox{otherwise.}  \end{array}  
         \right. 
\label{eq:img_segm_affinity_mat} %
\end{align}
where $\mathbf{f}_i$ and $\mathbf{f}_j$ are color histograms of superpixels $i$,~$j$, and $\mathrm{d}(i,j)$ is the spatial distance between superpixels $i$,~$j$.

From Fig.~\ref{fig:img_segm}, we can see that BNCut did not accurately extract foreground, 
because it cannot use the information about which pixels {\em cannot} be grouped together:
BNCut only uses the information provided by red markers.
In contrast, our method clearly extracts the foreground. 
We omit the segmentation results of SeDuMi and SDPT3, since they are similar with the one using \fastsdp.
In Table~\ref{tab:img_segm}, we compare the CPU time and the objective value of BNCut, \fastsdp, SeDuMi and SDPT3.
The results are the average of the five images shown in Fig.~\ref{fig:img_segm}.
In this example, $\sigma$ is set to $10^{-2}$ for \fastsdp.
All the five images are over-segmented into $760$ superpixels, and so the numbers of variables for SDP are the same ($289180$).
We can see that BNCut is much faster than SDP based methods, but with higher (worse) objective values.
\fastsdp achieves the similar objective value with SeDuMi and SDPT3, and is over $10$ times faster than them.

\subsection{Application 3: Image Co-segmentation}

{\bf Formulation}
Image co-segmentation performs partition on multiple images simultaneously. 
The advantage of co-segmentation over traditional single image segmentation is that
it can recognize the common object over multiple images.
Co-segmentation is conducted by optimizing two criteria: 1) the color and spatial consistency within a single image.
2) the separability of foreground and background over multiple images, measured by discriminative features, such as SIFT.
Joulin~\etal~\cite{Joulin2010dis} adopted a discriminative clustering method to the problem of co-segmentation, 
and used a low-rank factorization method~\cite{Journee2010lowrank} (denoted by \lowrank)  to solve the associated SDP program.
The \lowrank method finds a locally-optimal factorization $\bX = \bY
\bY^{\T}$, where the columns of $\bY$ is incremented until a certain condition is met. 
The formulation of discriminative clustering for co-segmentation can be expressed as:
\begin{align}
\min_{\bx \in \{-1,1\}^n} \langle \bx \bx^{\T}\!, \bA  \rangle,  
\sst \ ( \bx^{\T} \delta_i )^2 \!<\! \lambda^2,  \forall i = 1 ,\dots, q,  \label{eq:img_cosegm_x}
\end{align}
where $q$ is the number of images and $n = \sum_{i=1}^{q} n_i$ is total number of pixels.
Matrix $\bA = \bA_b + (\mu/n) \bA_w $, and $\bA_w = \bI_n - \bD^{-1/2} \bW \bD^{-1/2}$ is the intra-image affinity matrix, 
and $\bA_b \lambda_k (\bI - \be_n \be_n^{\T} / n) (n \lambda_k \bI_n + \bK)^{-1} (\bI - \be_n \be_n^{\T} / n)$ is the inter-image discriminative clustering cost matrix.
$\bW$ is a block-diagonal matrix, 
whose $i$th block is the affinity matrix~\eqref{eq:img_segm_affinity_mat} of the $i$th image, and $\bD = \mathbf{diag}(\bW \be_n)$.
$\bK$ is a kernel matrix, which is based on the $\chi^2-$distance of SIFT features: $\bK_{lm} = \mathrm{exp}(-\sum_{d=1}^{k} ((\bx_d^l - \bx_d^m)^2 / (\bx_d^l + \bx_d^m)))$.
Because there are multiple quadratic constraints, 
spectral methods are {\it not} applicable to  problem~\eqref{eq:img_cosegm_x}.

The constraints for $\bX = \bx \bx^{\T}$ are:
\begin{align}
\mathbf{diag} (\bX) = \be, \ \langle \bX, \delta_i \delta_i^{\T}  \rangle \leq \lambda^2, \ \forall i = 1 ,\dots, q. \label{eq:img_cosegm_cons}
\end{align}
We then introduce $\bA$ and the constraints~\eqref{eq:img_cosegm_cons} into~\eqref{eq:backgd_sdp1} and~\eqref{eq:fastsdp_analysis3} to get the associated SDP formulation.
%

%
%
%
The strategy in \lowrank is employed to recover a score vector ${\bx_r^\star}$ from the solution $\bX^{\star}$, which is based on the eigen-decomposition of $\bX^{\star}$.
The final binary solution $\bx^\star$ is obtained by thresholding
${\bx_r^\star}$ (comparing with $0$).

{\bf Experiments}
The Weizman horses\footnote{\url{http://www.msri.org/people/members/eranb/}}
and
MSRC\footnote{\url{http://www.research.microsoft.com/en-us/projects/objectclassrecognition/}}
datasets
are used for this image co-segmentation.
There are $6\!\thicksim\!10$ images in each of four classes, namely car-front, car-back, face and horse.
Each image is oversegmented to $400\!\thicksim\!700$ SLIC superpixels using VLFeat~\cite{vedaldi08vlfeat}.
The number of superpixels for each image class is then increased to $4000\!\thicksim\!7000$.

Standard toolboxes like SeDuMi and SDPT3 cannot handle such large-size
problems on a standard desktop.
We compare \fastsdp with the \lowrank approach. 
In this experiment, $\sigma$ is set to $10^{-4}$ for \fastsdp.
As we can see in Table~\ref{tab:img_cosegm}, the speed of \fastsdp is
about $5.7$ times faster than \lowrank on average.    
The objective values (to be minimized) of \fastsdp are lower than \lowrank for all the four image classes.
Furthermore, the solution of \fastsdp also has lower rank than that of \lowrank for each class. 
For car-back, the largest eigenvalue of the solution for \fastsdp has $81\%$ of total energy while the one for \lowrank only has $56\%$.

Fig.~\ref{fig:img_cosegm} visualizes the score vector ${\bx}_r^\star$ on some sample images.
The common objects (cars, faces and horses) are identified by our co-segmentation method.
\fastsdp and \lowrank achieve visually similar results in the experiments.

\begin{figure}[t]
\vspace{-0.0cm}
\centering
\subfloat{
\includegraphics[width=0.11\textwidth,clip]{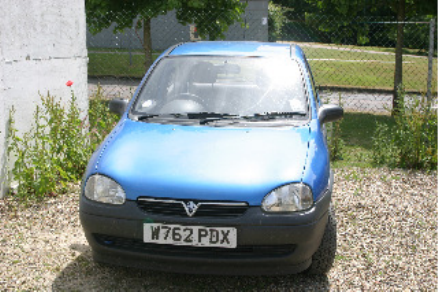}
\includegraphics[width=0.11\textwidth,clip]{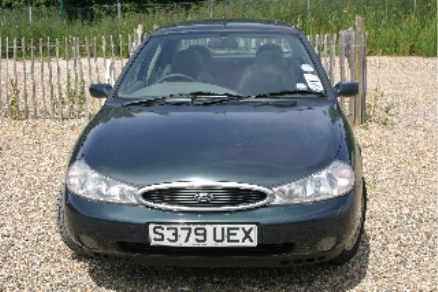}
\includegraphics[width=0.11\textwidth,clip]{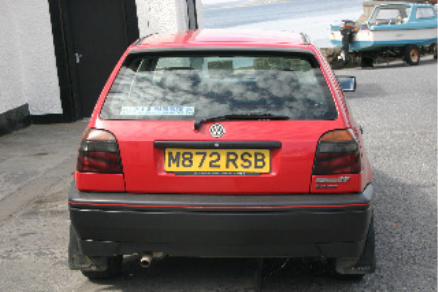}
\includegraphics[width=0.11\textwidth,clip]{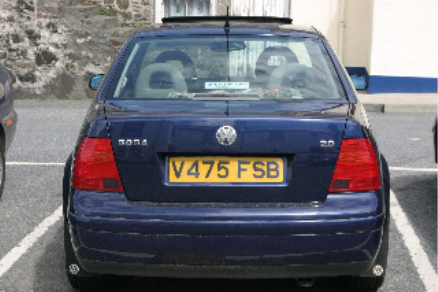}
}\\ 
\subfloat{
\includegraphics[width=0.11\textwidth,clip]{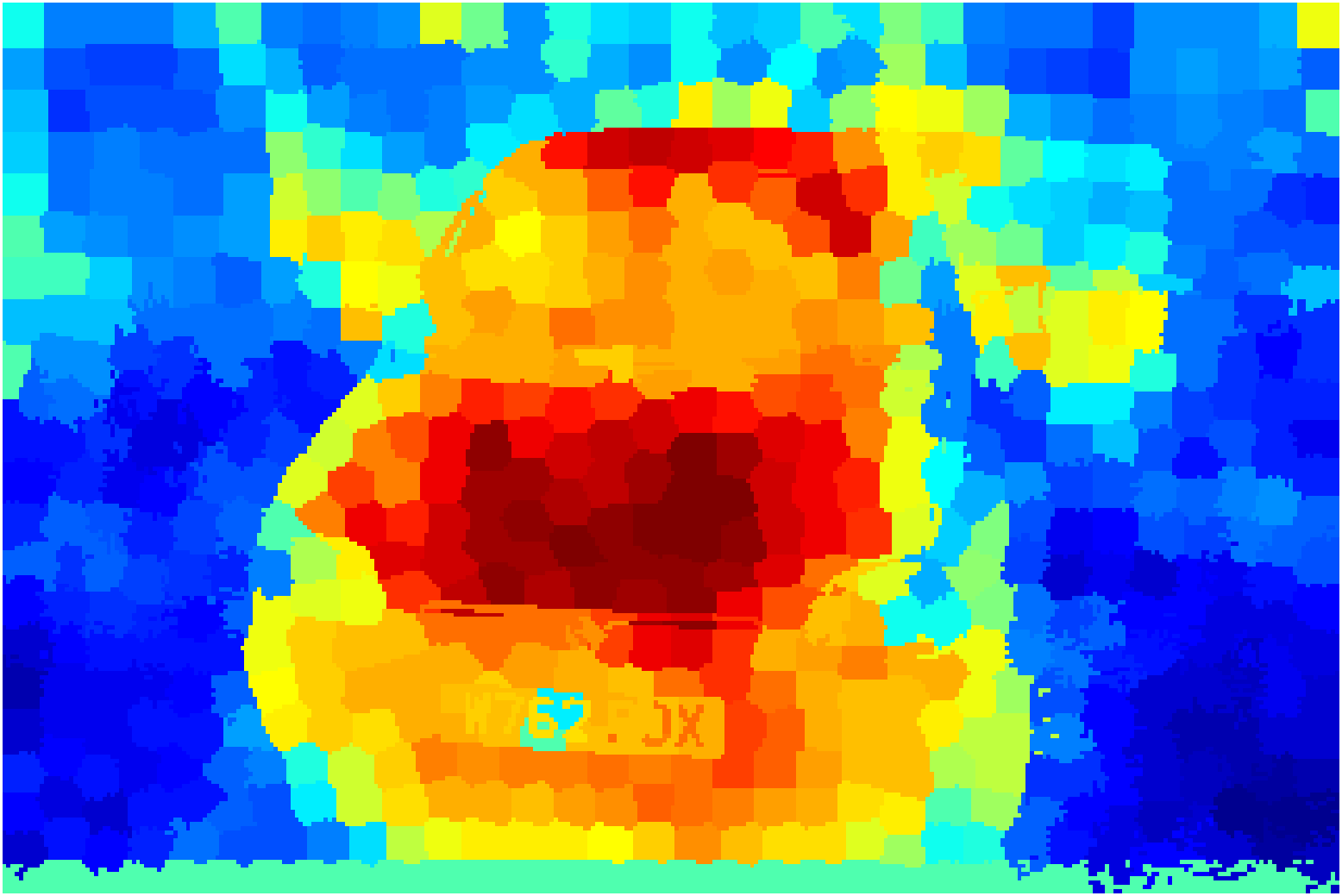}
\includegraphics[width=0.11\textwidth,clip]{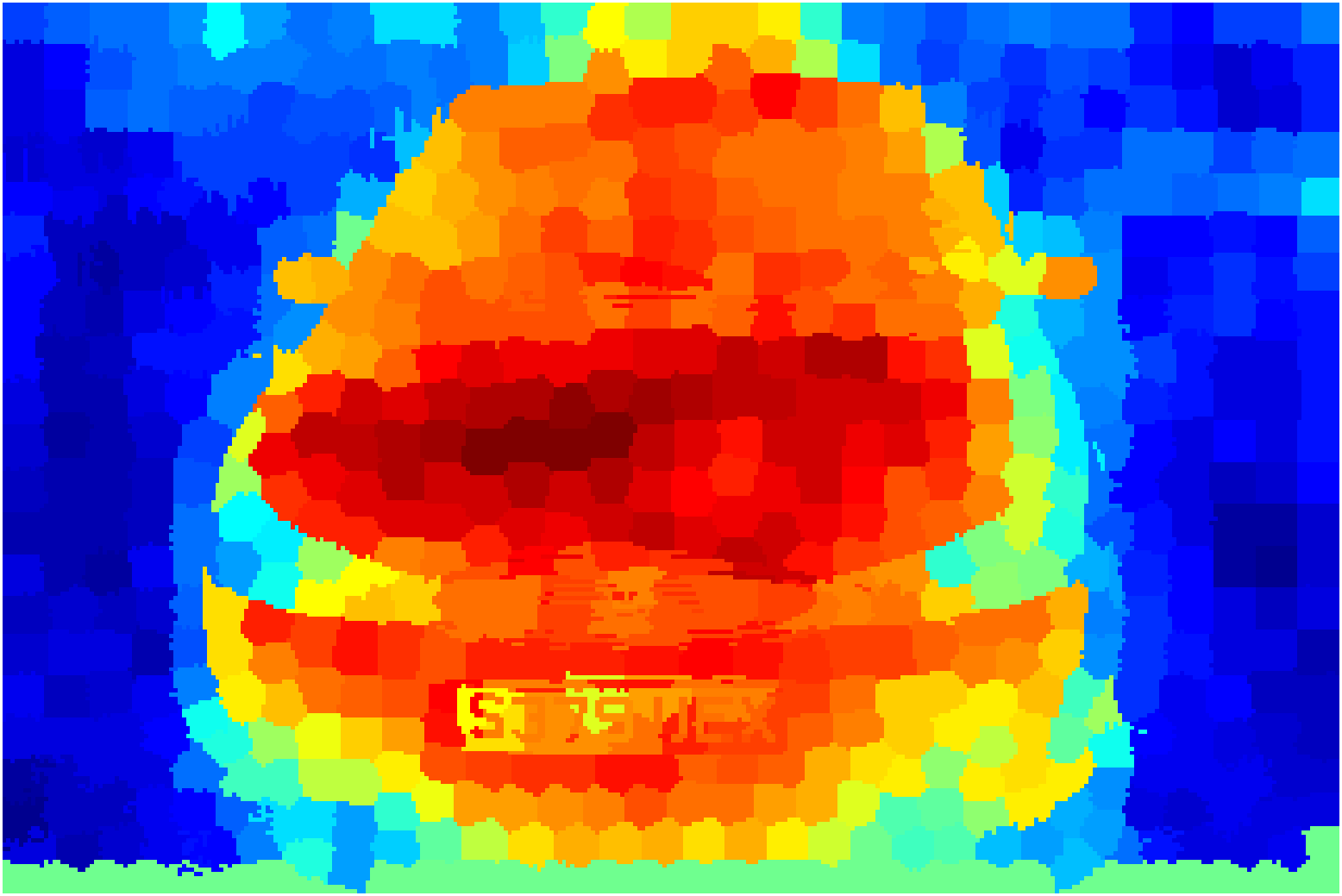}
\includegraphics[width=0.11\textwidth,clip]{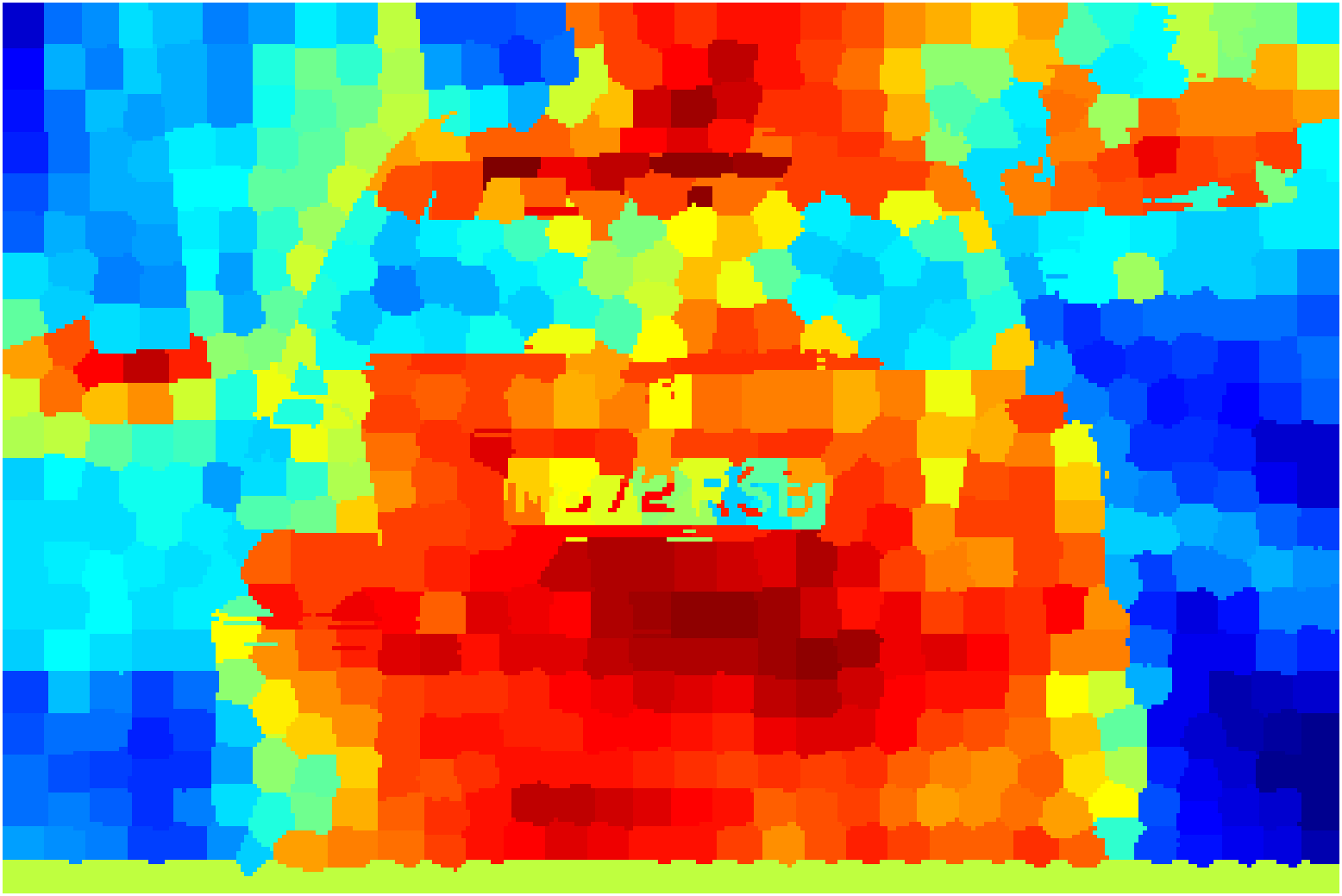}
\includegraphics[width=0.11\textwidth,clip]{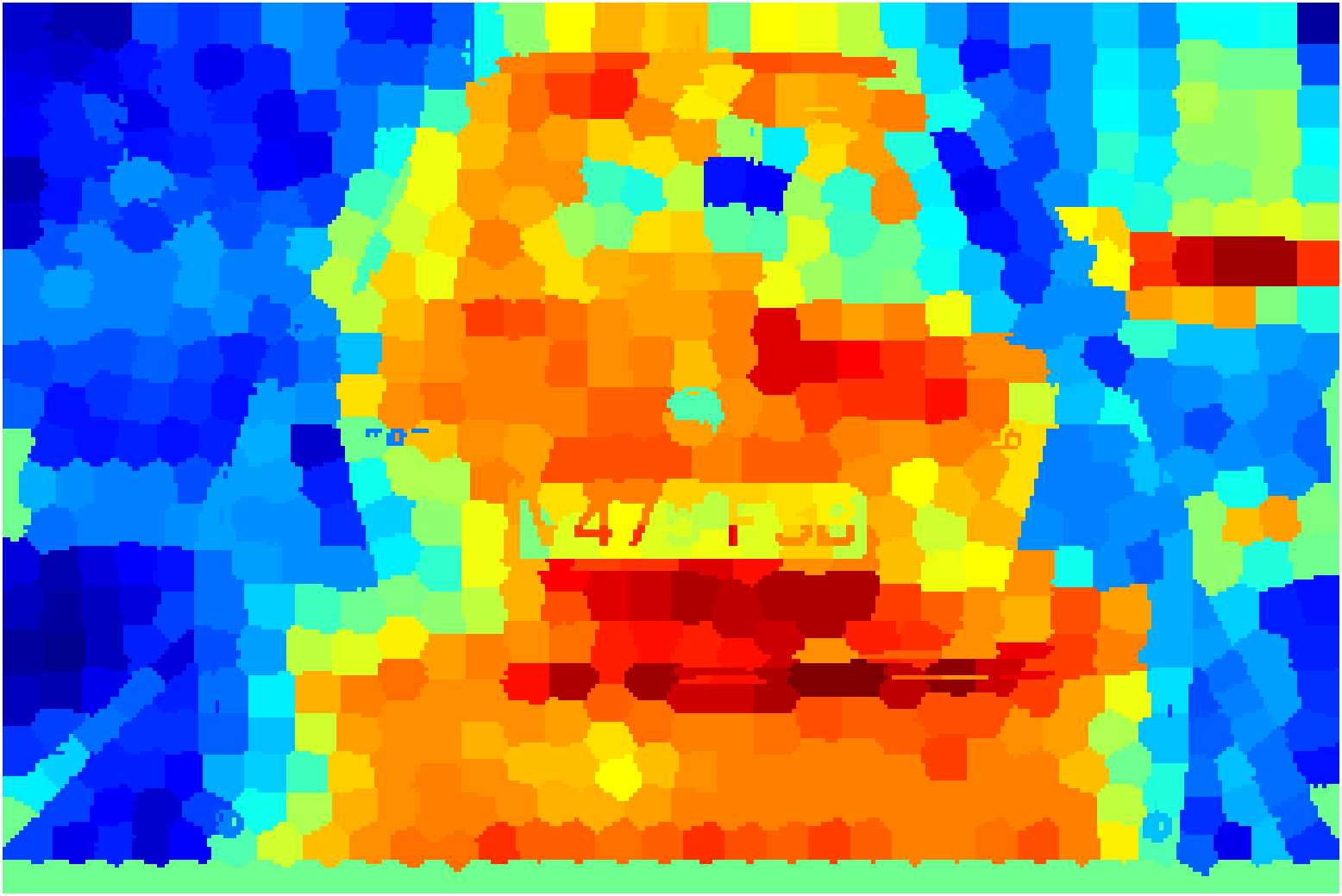}
}\\ 
\subfloat{
\includegraphics[width=0.11\textwidth,clip]{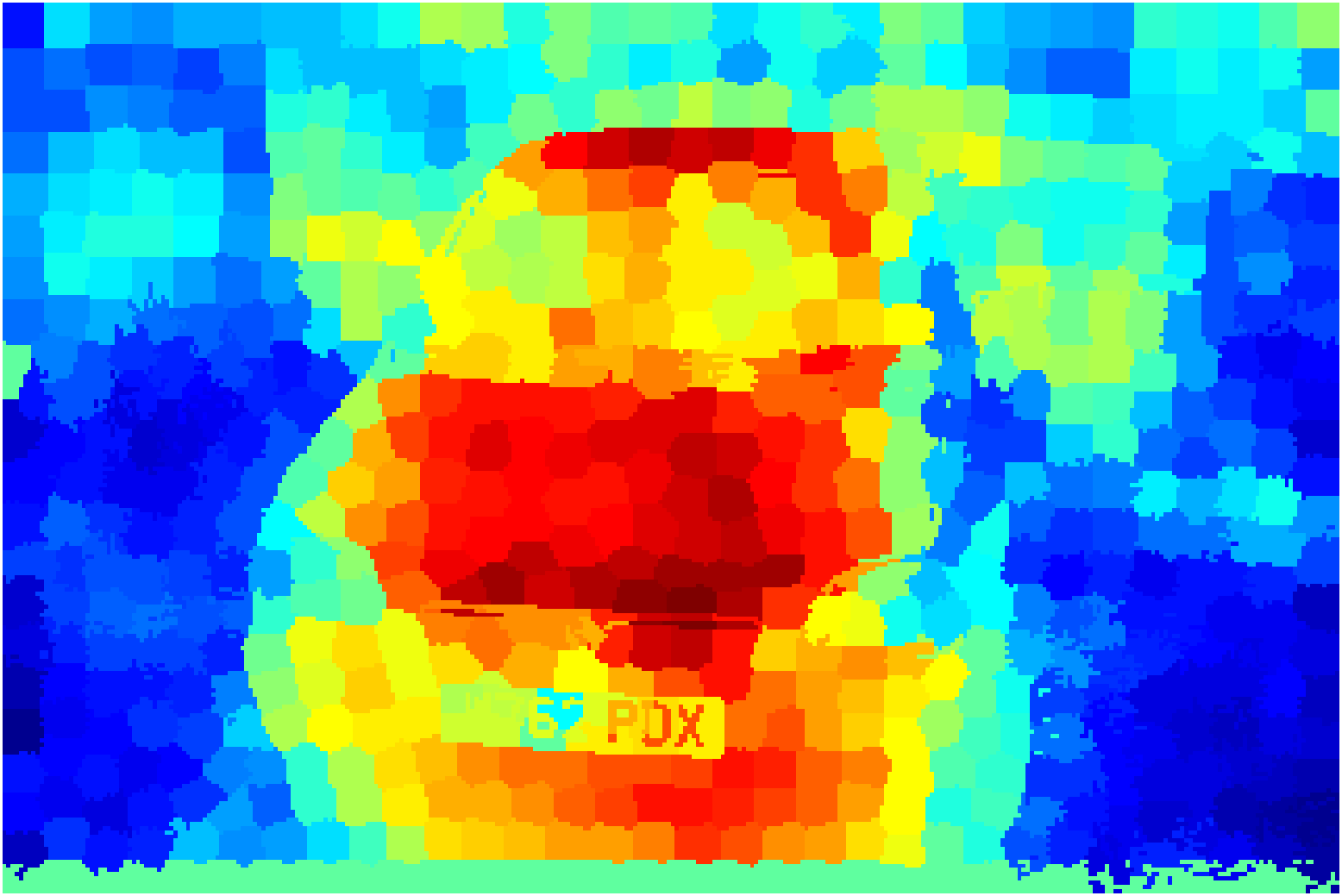}
\includegraphics[width=0.11\textwidth,clip]{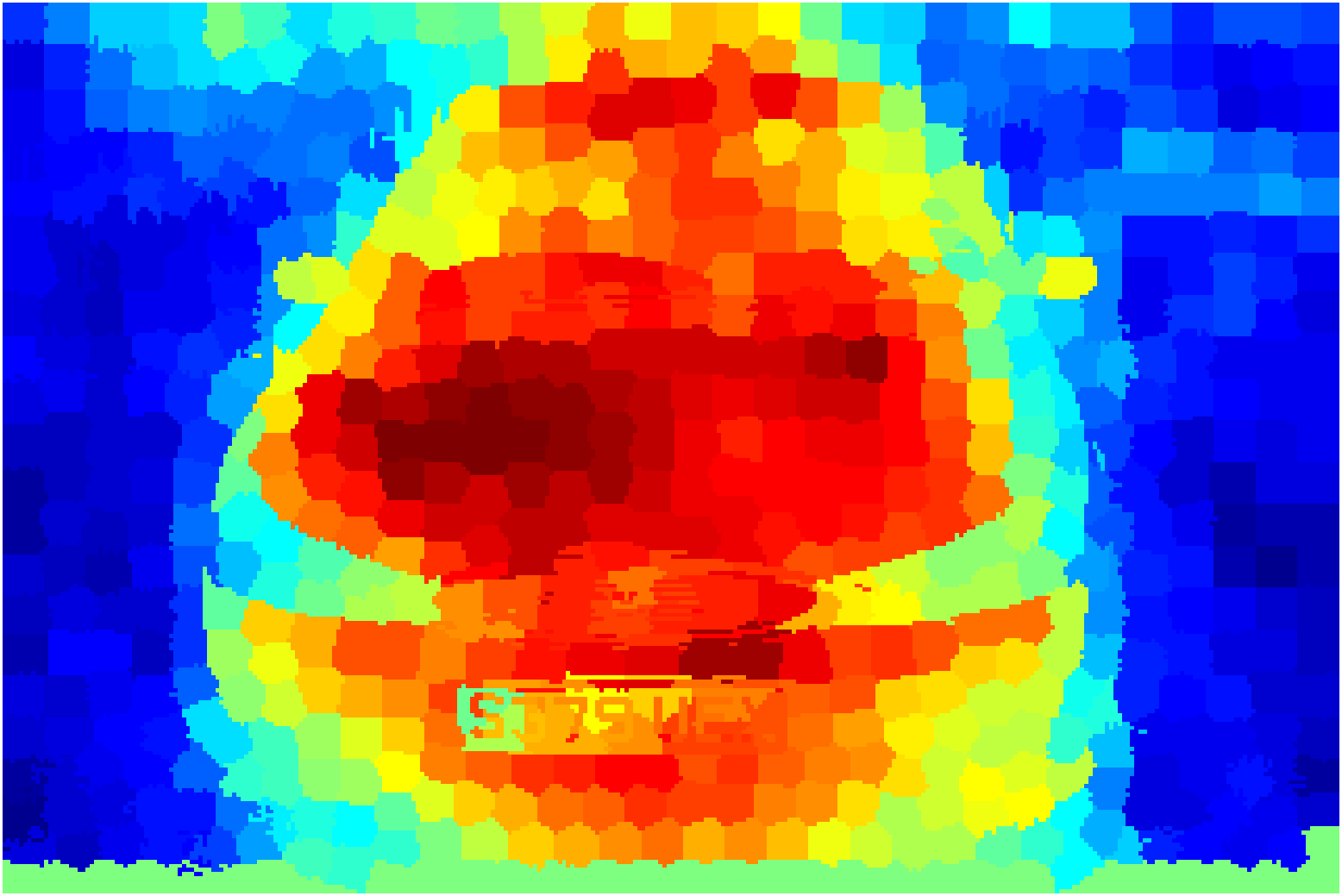}
\includegraphics[width=0.11\textwidth,clip]{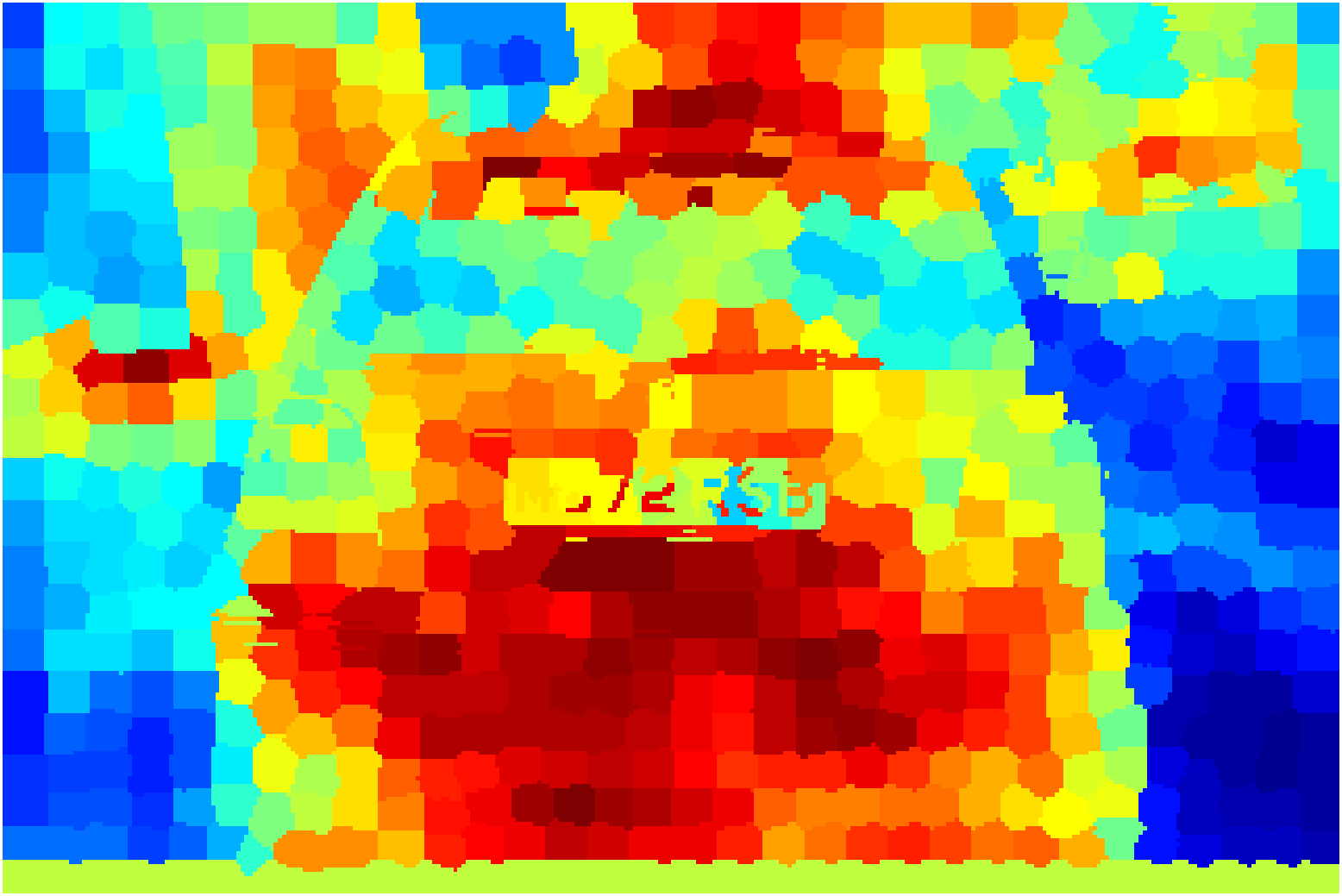}
\includegraphics[width=0.11\textwidth,clip]{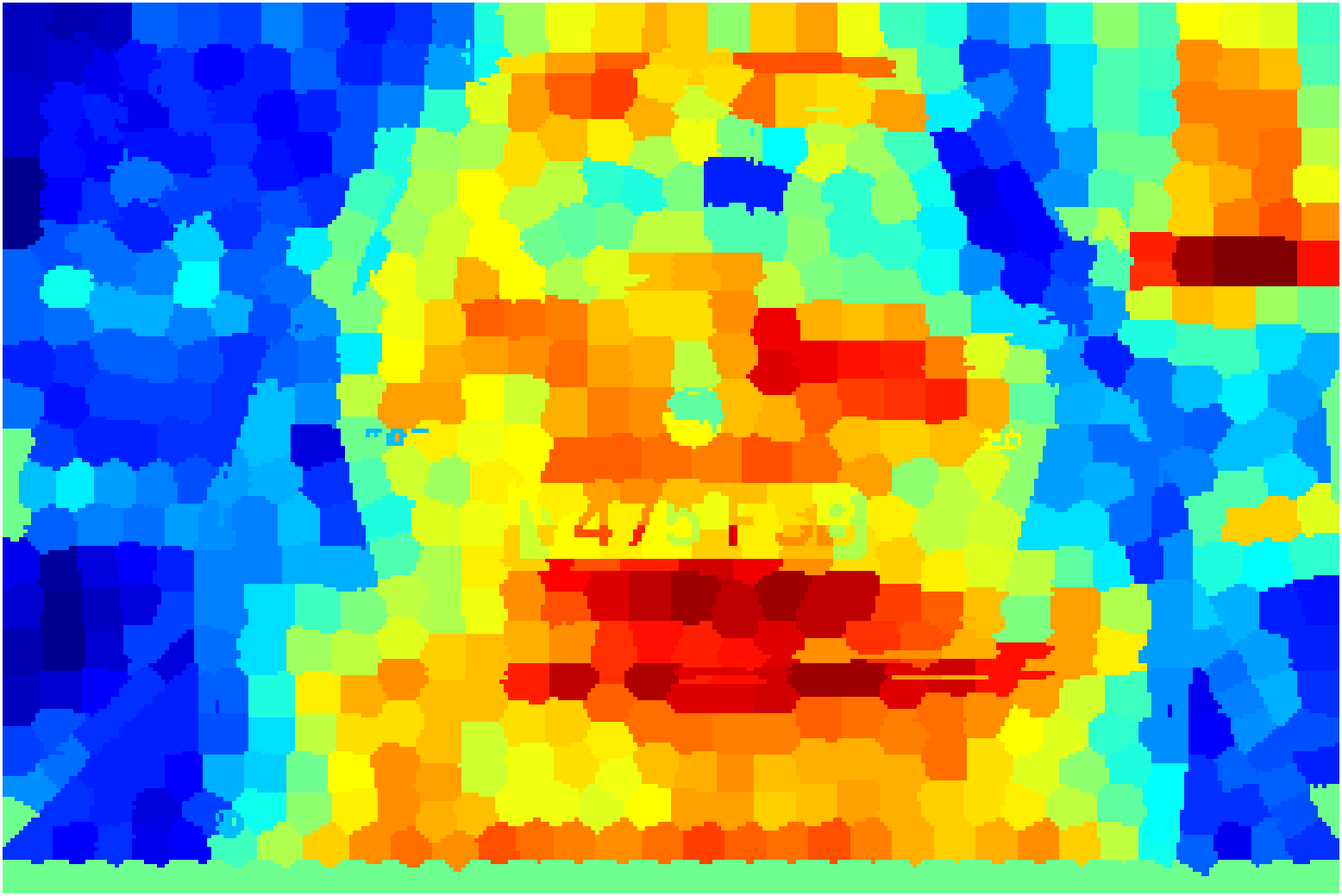}
}\\ 
\subfloat{
\includegraphics[width=0.11\textwidth,clip]{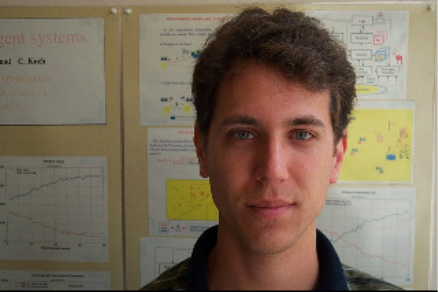}
\includegraphics[width=0.11\textwidth,clip]{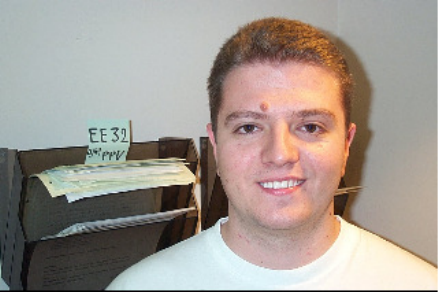}
\includegraphics[width=0.11\textwidth, height = 0.074\textwidth, clip]{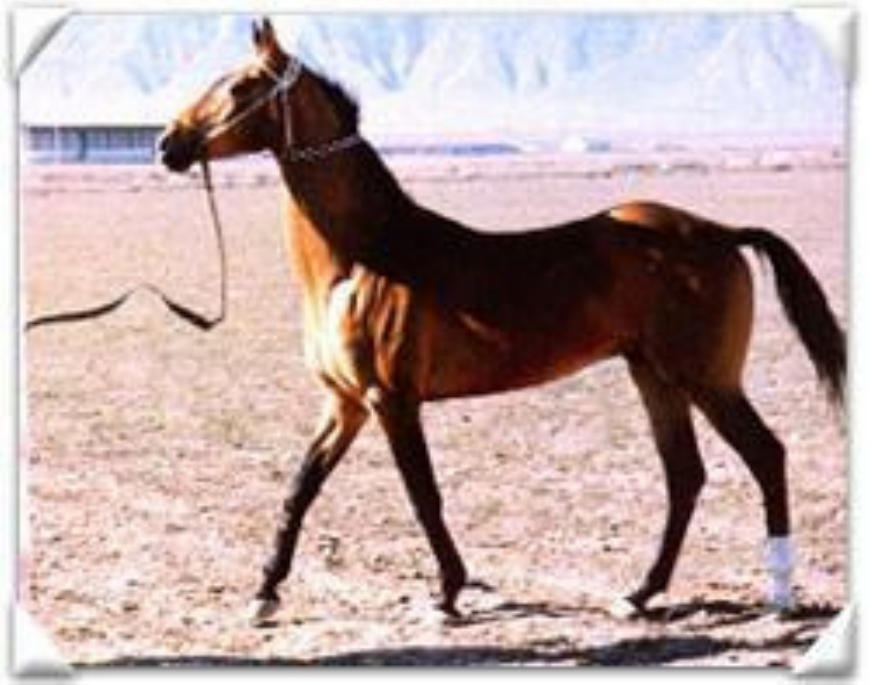}
\includegraphics[width=0.11\textwidth, height = 0.074\textwidth, clip]{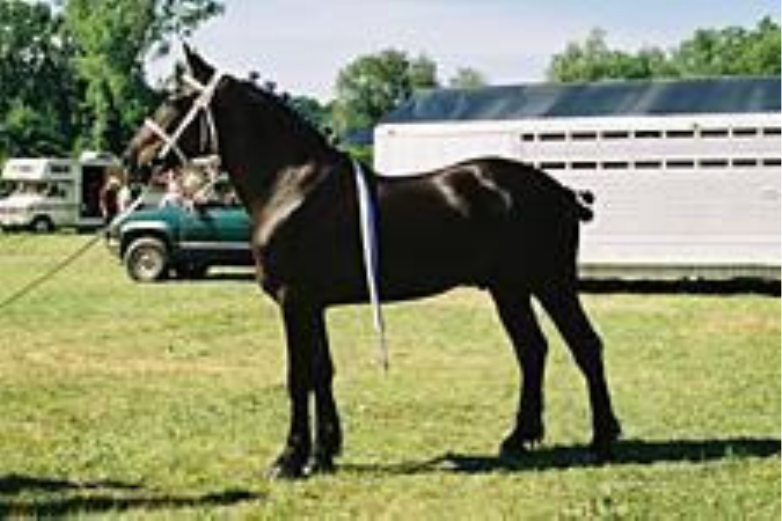}
}\\ 
\subfloat{
\includegraphics[width=0.11\textwidth,clip]{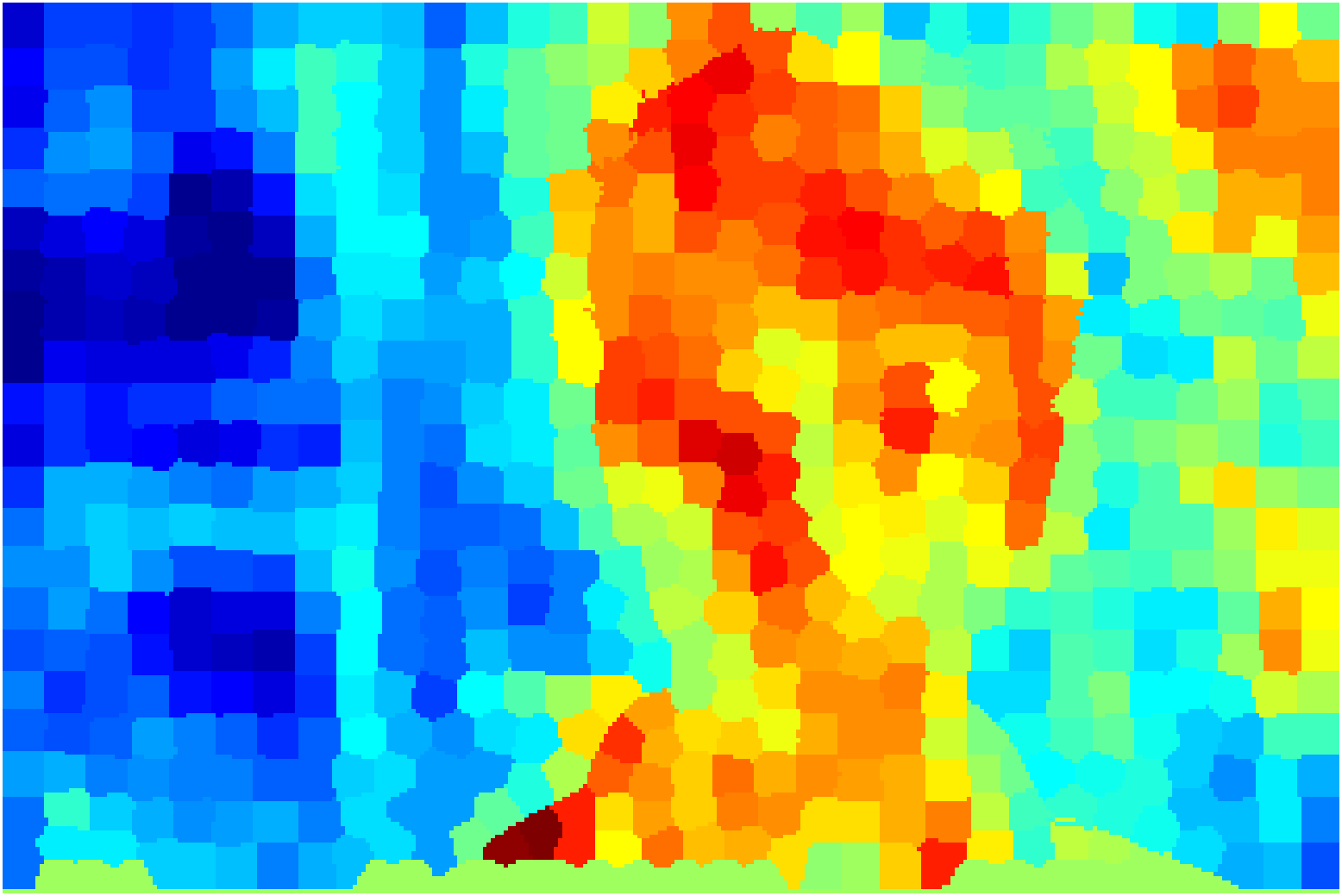}
\includegraphics[width=0.11\textwidth,clip]{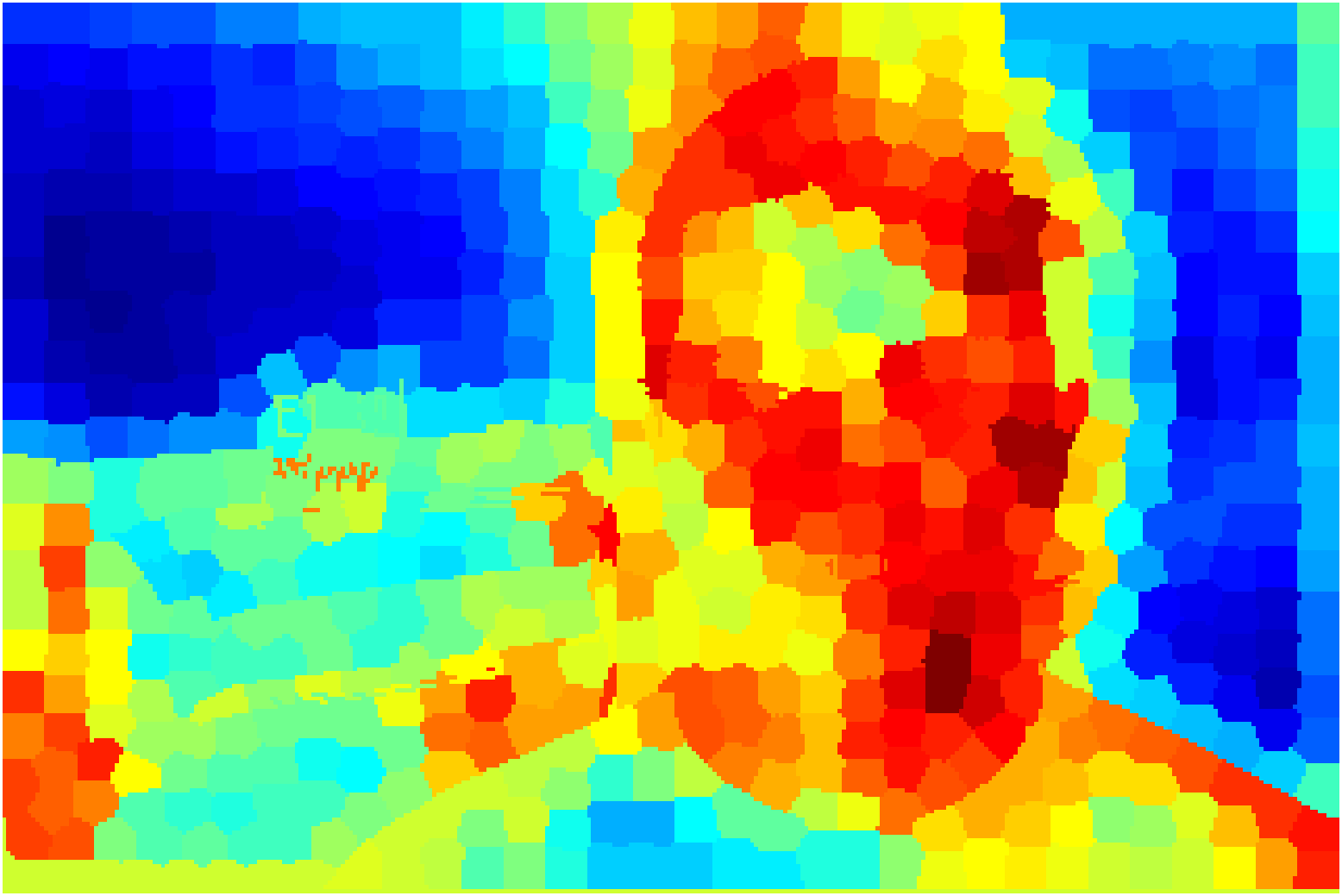}
\includegraphics[width=0.11\textwidth, height = 0.074\textwidth, clip]{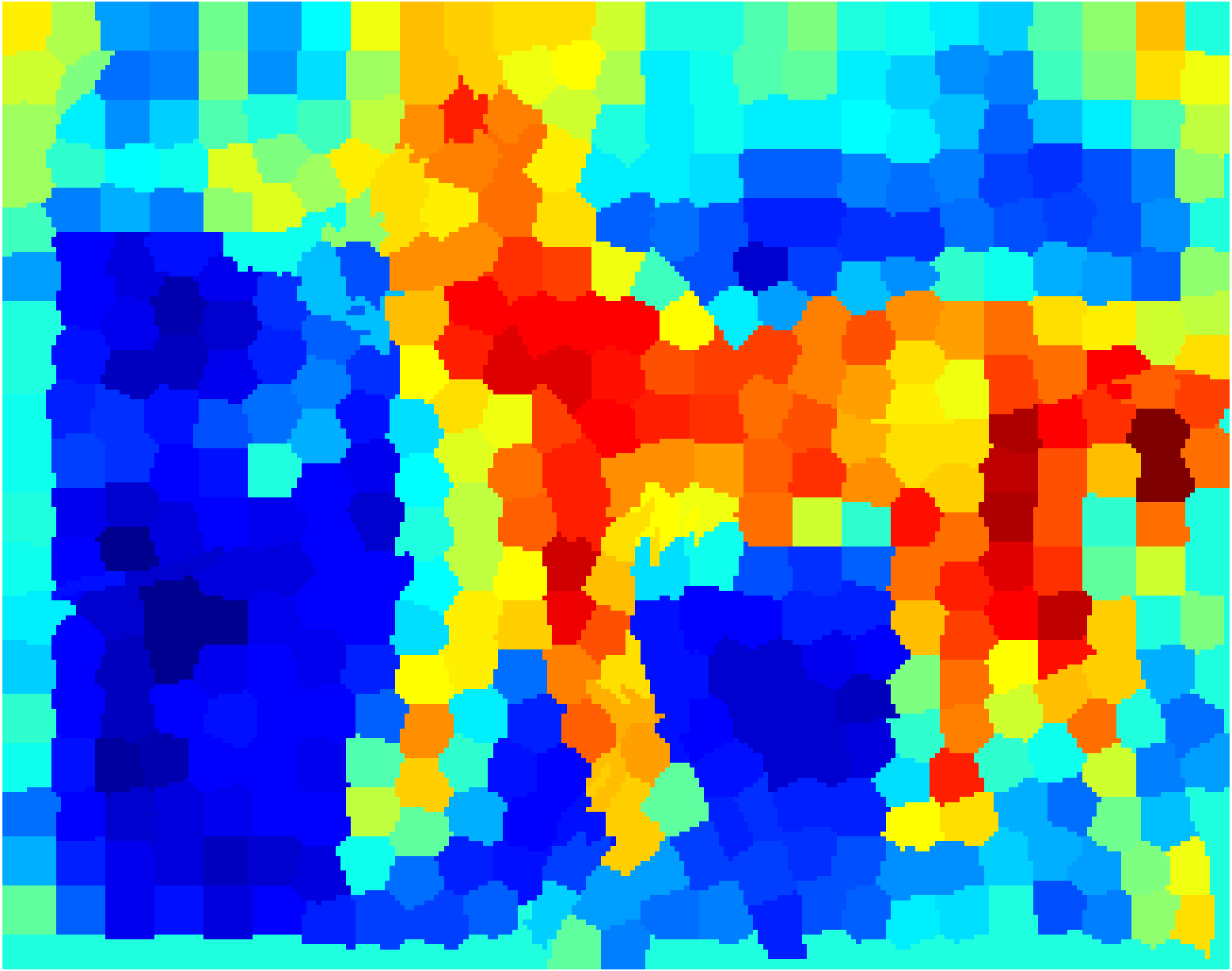}
\includegraphics[width=0.11\textwidth, height = 0.074\textwidth, clip]{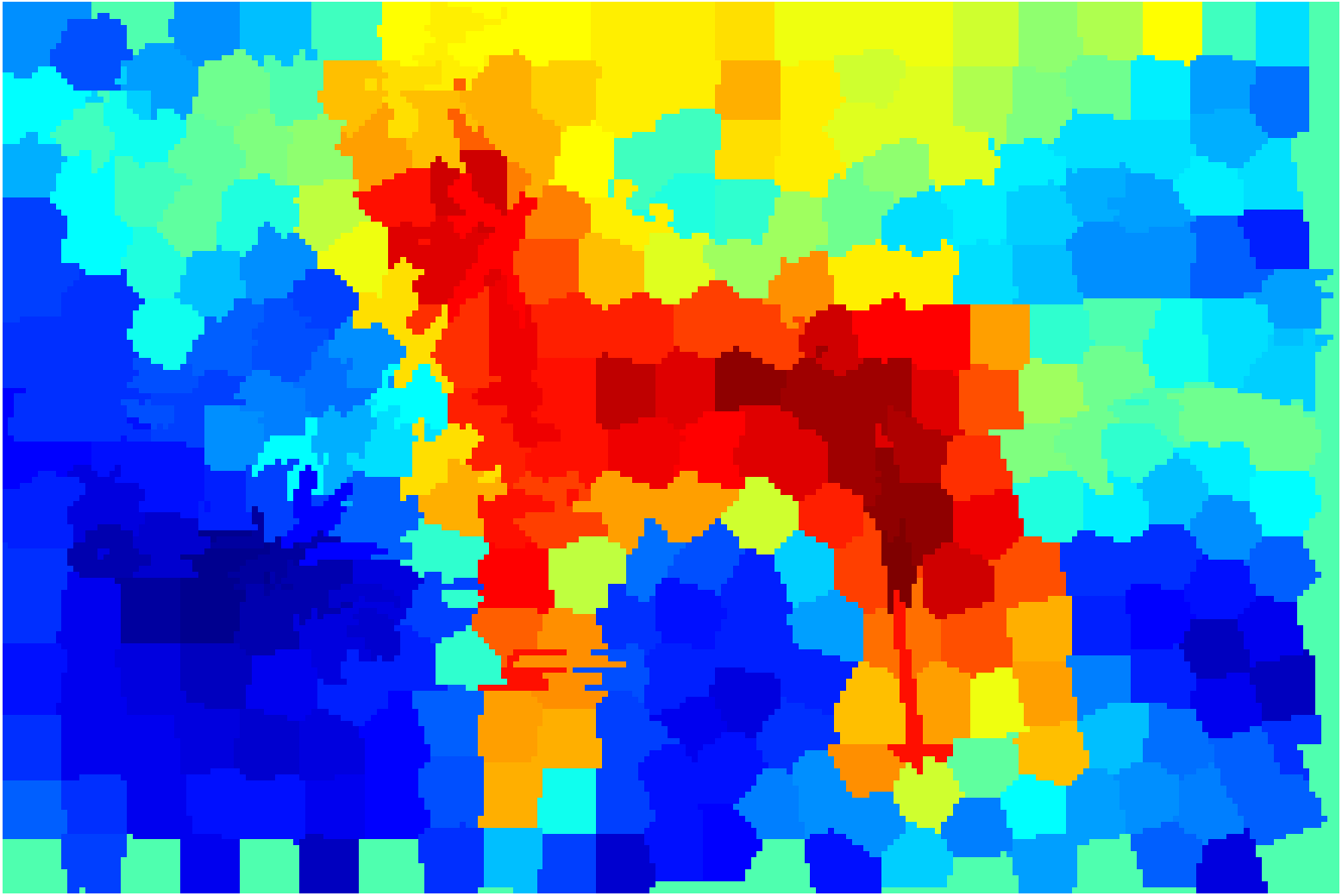}
}\\ 
\subfloat{
\includegraphics[width=0.11\textwidth,clip]{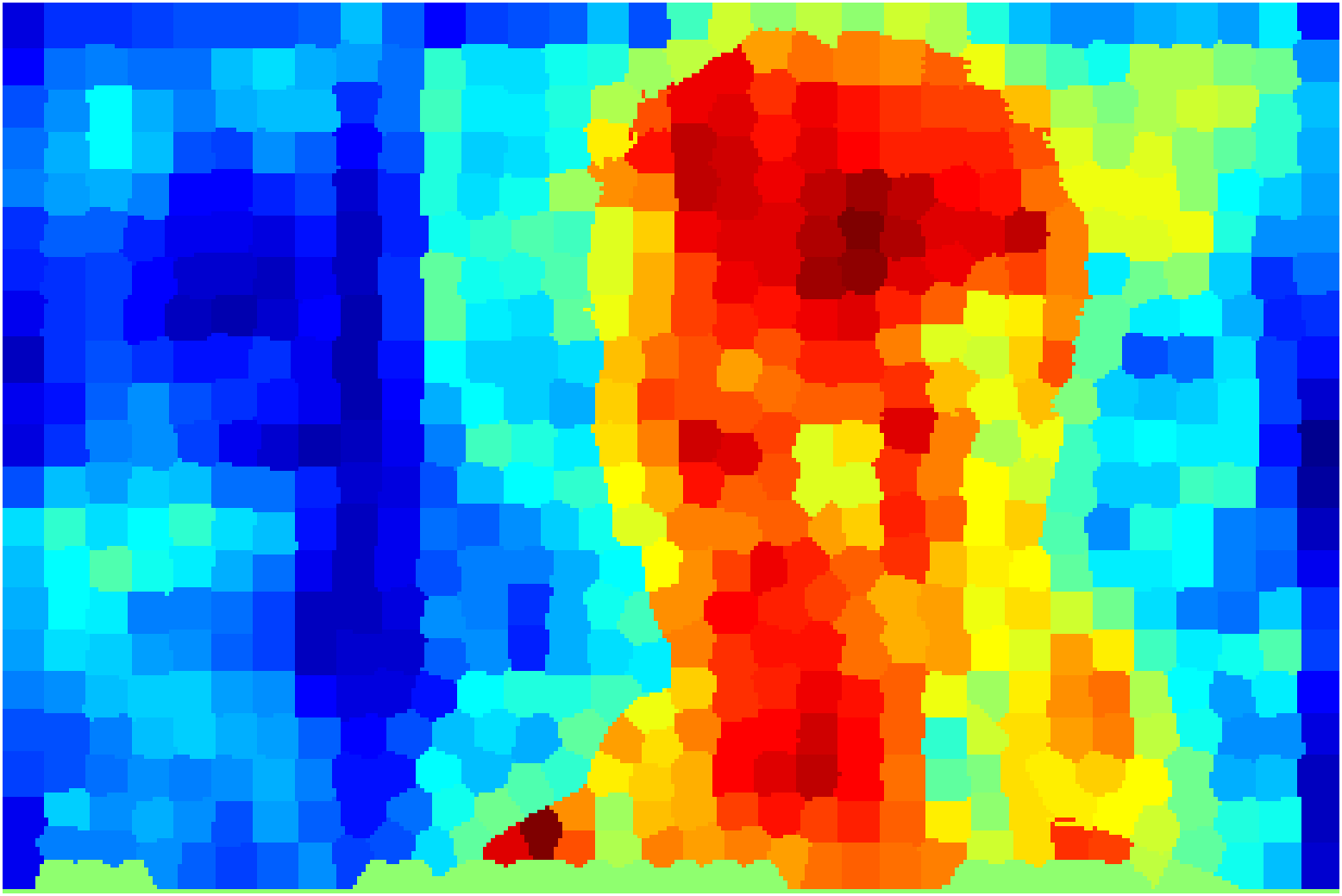}
\includegraphics[width=0.11\textwidth,clip]{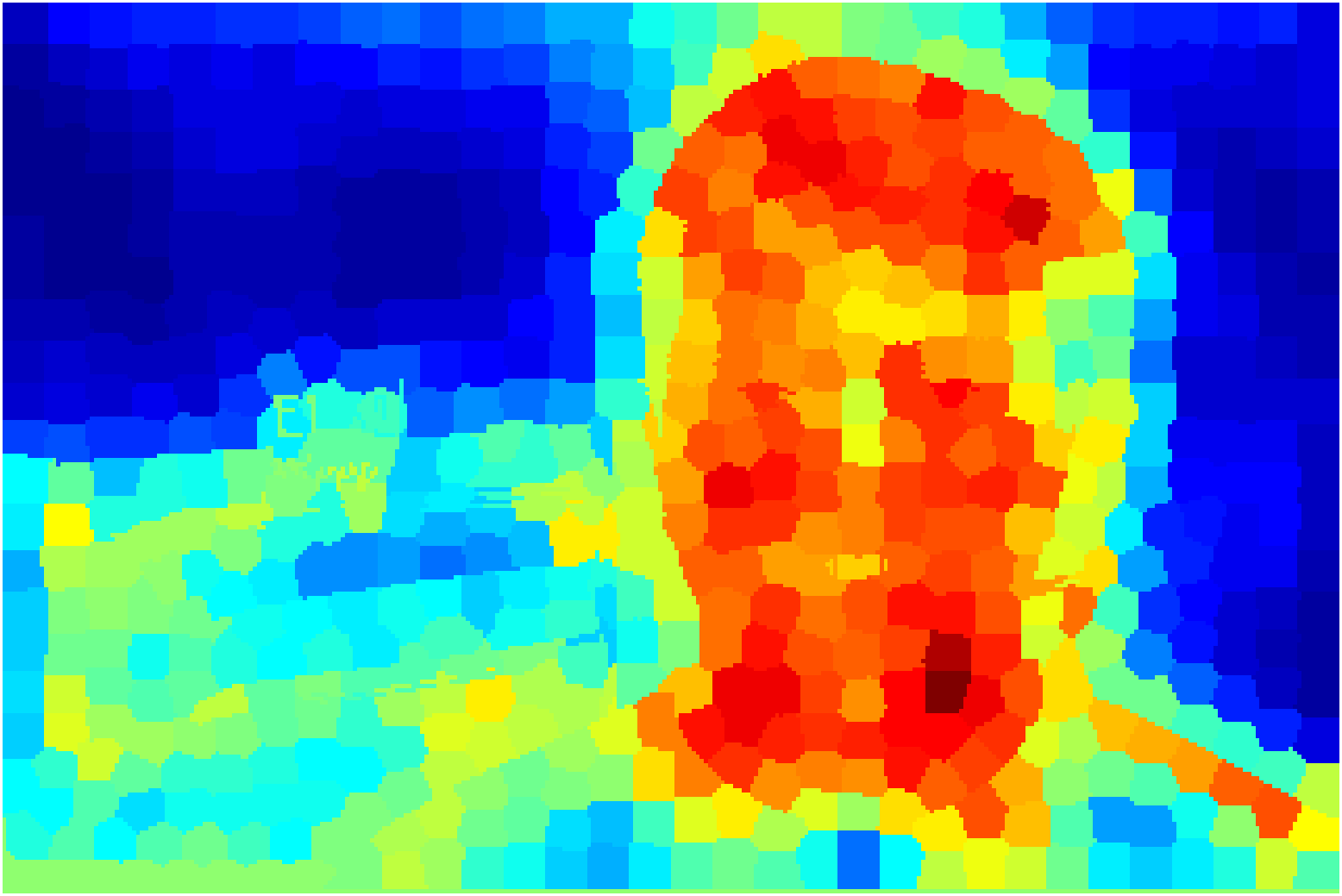}
\includegraphics[width=0.11\textwidth, height = 0.074\textwidth, clip]{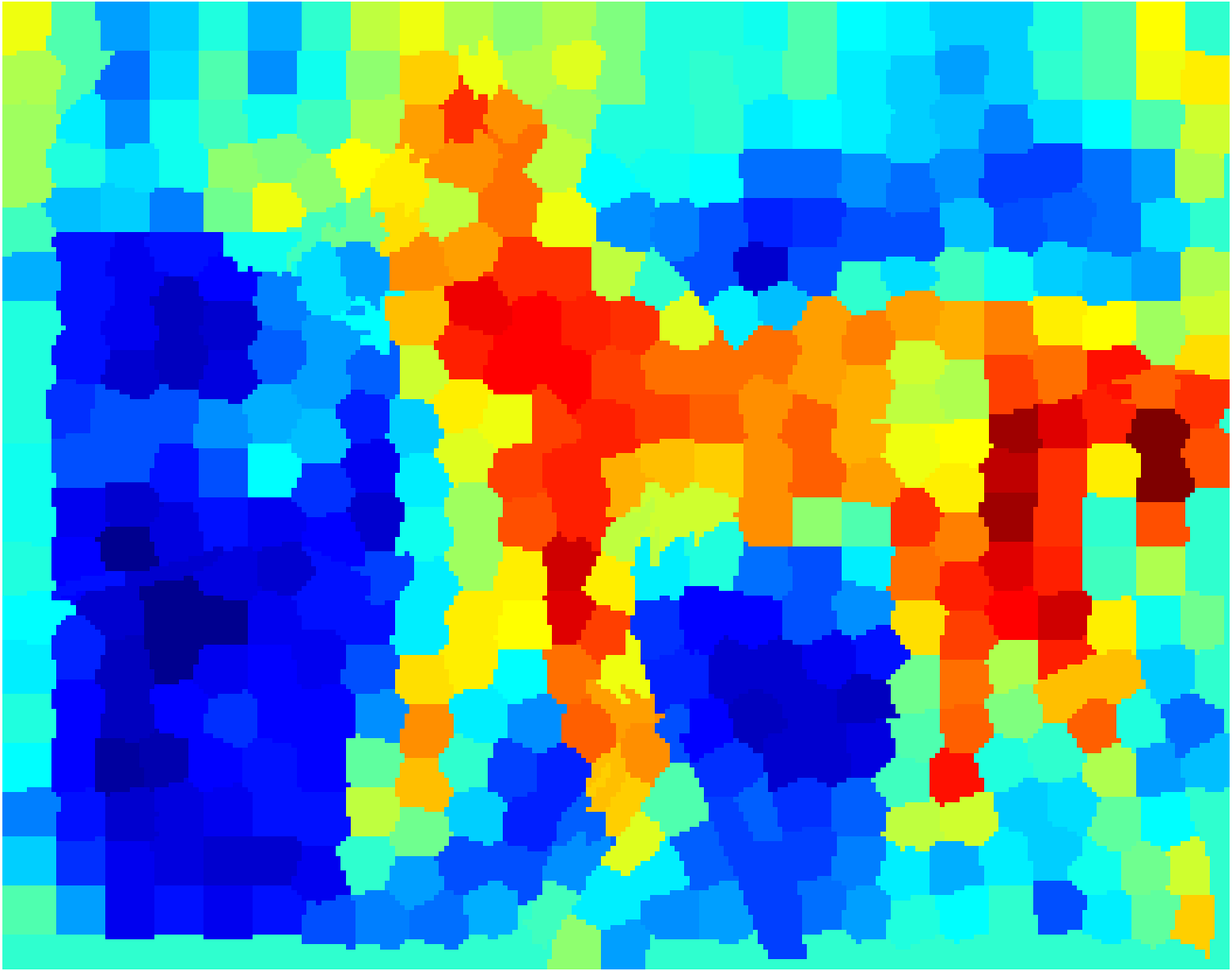}
\includegraphics[width=0.11\textwidth, height = 0.074\textwidth, clip]{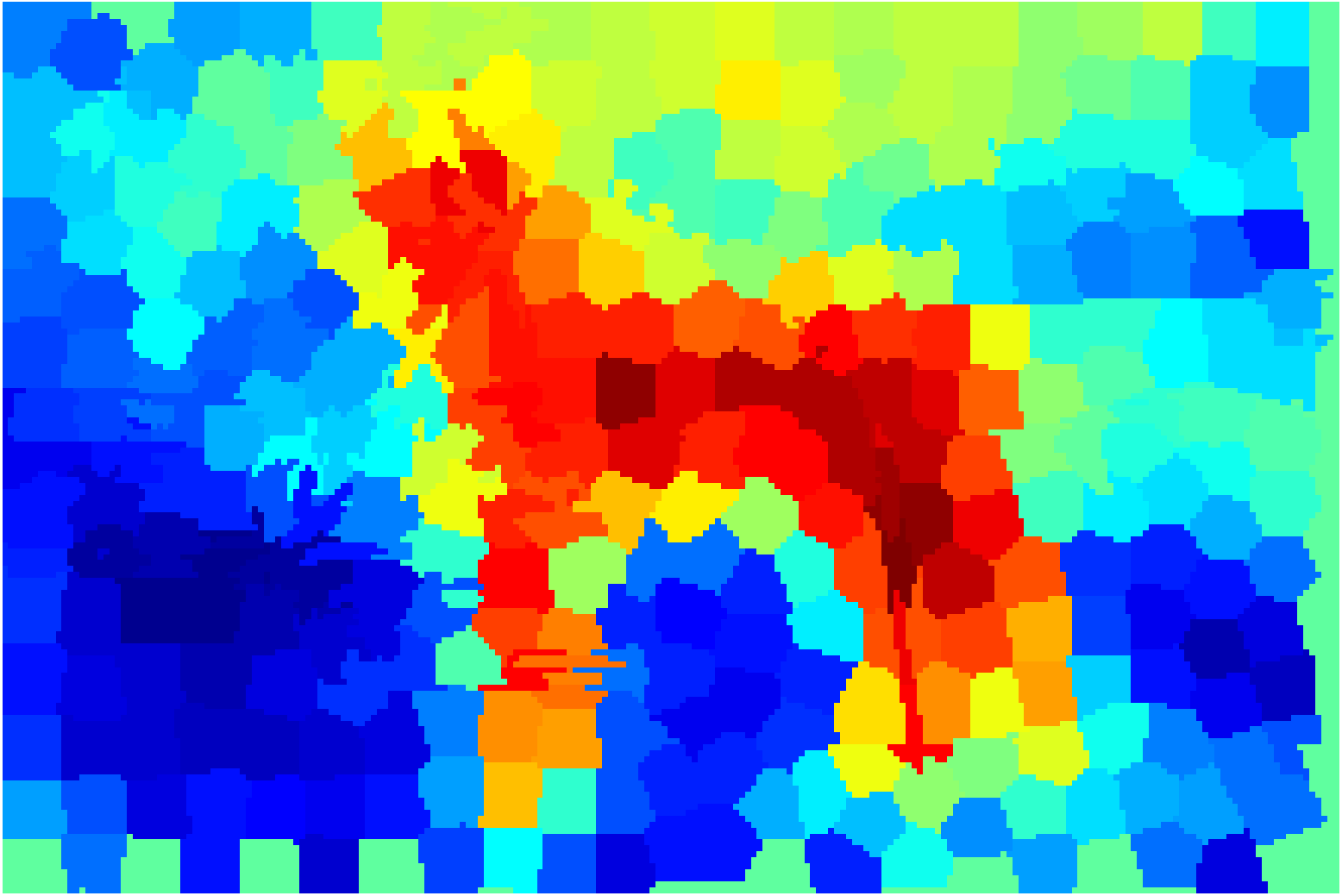}
}
\vspace{-0.0cm}
\caption{Co-segmentation results on Weizman horses and MSRC datasets.
The original images, the results of \lowrank and \fastsdp are illustrated from top to bottom.
\lowrank and \fastsdp produce similar results.}
\vspace{-0.0cm}
\label{fig:img_cosegm}
\end{figure}

\begin{table}[t]
  \centering
  \vspace{-0.0cm}
  \footnotesize
  \begin{tabular}{l|l|@{\hspace{0.1cm}}c@{\hspace{0.1cm}}c@{\hspace{0.1cm}}c@{\hspace{0.1cm}}c}
  \hline
     \multicolumn{2}{l|@{\hspace{0.1cm}}}{Dataset}             &  horse & face & car-back & car-front \\
  \hline
  \hline
    \multicolumn{2}{l|@{\hspace{0.1cm}}}{{\#}Images}              &  $10$       & $10$        & $6$        & $6$ \\
  \hline    
    \multicolumn{2}{l|@{\hspace{0.1cm}}}{{\#}Vars of BQPs~\eqref{eq:img_cosegm_x}}            &  $4587$       & $6684$        & $4012$        & $4017$ \\
  \hline  
    \multirow{2}{*}{Time(s)} & \lowrank          &  $1724$       & $3587$        & $2456$       & $2534$ \\
                             & \fastsdp          &  $430.3$      & $507.0$       & $251.1$      & $1290$ \\
  \hline    
    \multirow{2}{*}{obj}    & \lowrank           &  $-4.90$       & $-4.55$        & $-4.19$       & $-4.15$ \\
                                             & \fastsdp           &  $-5.24$       & $-4.94$        & $-4.53$       & $-4.27$ \\
  \hline
    \multirow{2}{*}{rank}   & \lowrank           &  $17$ & $16$ & $13$ & $11$ \\
                                             & \fastsdp           &  $3$  & $3$  & $3$  & $3$  \\
  \hline
  \end{tabular}
  \caption{Performance comparison of \lowrank~\cite{Journee2010lowrank} and \fastsdp 
           for co-segmentation.
           \fastsdp achieves faster speeds and better solution quality than \lowrank, 
           on all the four datasets.          
           obj $= \langle {\bx^\star}{\bx^\star}^{\T}, \bA \rangle$.%
           $\sigma$ is set to $10^{-4}$. 
          }
  \label{tab:img_cosegm}
\end{table}

\subsection{Application 4: Image Registration}

{\bf Formulation}
In image registration, $K$ source points must be matched to $L$ target points, where $K < L$.
The matching should maximize the local feature similarities of matched-pairs 
and also the structure similarity between the source and target graphs. 
The problem is expressed as a BQP, as in~\cite{Schellewald05}:
\begin{subequations}
\begin{align}
\min_{\bx \in \{0,1\}^{KL} } &\ \bh^{\T} \bx + \alpha \bx^{\T} \bH \bx,  \\
\sst   \,\,\,\,\,\,\,    &\ {\textstyle \sum\nolimits_j \bx_{ij}=1, \forall i = 1, \dots, K, }\label{eq:graph_match_cons01} \\
           &\ {\textstyle \sum\nolimits_i \bx_{ij} \leq 1, \forall j = 1, \dots, L,} \label{eq:graph_match_cons02} 
\end{align}
\label{eq:graph_match_1}
\end{subequations} 
where $\bx_{ij} = \bx_{(i-1)\!L+j} = 1$ if the source point $i$ is matched to the target point $j$; otherwise 0.
$\bh \in \mathbb{R}^{K\!L}$ records the local feature similarity between each pair of source-target points;
$\bH_{ij,kl} = \mathrm{exp} (-(d_{ij} - d_{kl})^2 / \sigma^2)$ encodes the structural consistency of source points $i$, $j$ and target points $k$,~$l$.

By adding one row and one column to $\bH$ and $\bX = \bx \bx^{\T}$, we have:
$\hat{\bH} = [ 0, \ 0.5\bh^{\T} ; 0.5\bh, \ \alpha \bH   ]$,
$\hat{\bX} = [ 1, \ \bx^{\T} ; \bx, \ \bX  ]$.
Schellewald \etal~\cite{Schellewald05} formulate the constraints for $\hat{\bX}$ as: 
\begin{subequations}
\begin{align}
                     & \hat{\bX}_{11} = 1, \label{eq:graph_match_cons0}\\
                     & 2 \cdot \mathbf{diag}(\bX) = \bX_{1:}^{\T} + \bX_{:1}, \label{eq:graph_match_cons1} \\
                     & \bN \cdot \mathbf{diag}(\bX) = \be_K, \label{eq:graph_match_cons2}\\
                     & \bM \circ \bX = \mathbf{0}, \label{eq:graph_match_cons3}
\end{align}
\end{subequations}
where $\bN = \bI_K \otimes \be_L^{\T}$ and 
$\bM = \bI_K \otimes (\be_L \be_L^{\T} -\bI_L) + (\be_K \be_K^{\T} -\bI_K) \otimes \bI_L$ .
Constraint~\eqref{eq:graph_match_cons1} arises from the fact that $x_i \!=\! x_i^2$; constraint~\eqref{eq:graph_match_cons2} arises from~\eqref{eq:graph_match_cons01};
constraint~\eqref{eq:graph_match_cons3} avoids undesirable solutions that match one point to multiple points.
The SDP formulations are obtained by introducing into~\eqref{eq:backgd_sdp1} and~\eqref{eq:fastsdp_analysis3} 
the matrix $\hat{\bH}$ and the constraints~\eqref{eq:graph_match_cons0} to~\eqref{eq:graph_match_cons3}.
In this case, the BQP is a $\{0,1\}$-problem, instead of $\{-1,1\}$-problem.
Based on~\eqref{eq:graph_match_cons01}, $\eta = \mathrm{trace}(\hat{\bX})  = K\!+\!1$.
The binary solution $\bx^\star$ is obtained by solving the linear program:
\begin{align}
\max_{\bx \in \mathbb{R}^{KL} } \bx^{\T} \mathbf{diag}(\bX^\star), \ \ \sst \ \bx \geq \mathbf{0},~\eqref{eq:graph_match_cons01},~\eqref{eq:graph_match_cons02}, \label{eq:graph_match_round}
\end{align}
which is guaranteed to have integer solutions~\cite{Schellewald05}.

{\bf Experiments}
We apply our registration formulation on some toy data and real-world data.
For toy data, we firstly generate $30$ target points from a uniform distribution, and randomly select $15$ source points.
The source points are rotated and translated by a random similarity transformation $\by = \bR \bx + \bt$ with additive Gaussian noise.
For the Stanford bunny data, %
$50$ points are randomly sampled and similar transformation and noise are applied.
$\sigma$ is set to $10^{-4}$.

From Fig.~\ref{fig:graph_match},
we can see that the source and target points are matched correctly.
For the toy data, our method runs over $170$ times and $50$ times  faster than SeDuMi and SDPT3 respectively.
For the bunny data with $3126250$ variables, \fastsdp spends $412$ seconds and SeDuMi/SDPT3 did not find solutions after $3$ hours running.
The improvements on speed for \fastsdp is more significant than previous experiments. 
The reason is that the SDP formulation for registration has much more
constraints, which slows down SeDuMi and SDPT3 but has much less impact on \fastsdp.

\begin{figure}[t]
\begin{minipage}[t]{\linewidth}
  \centering
  \includegraphics[width=0.32\textwidth,clip]{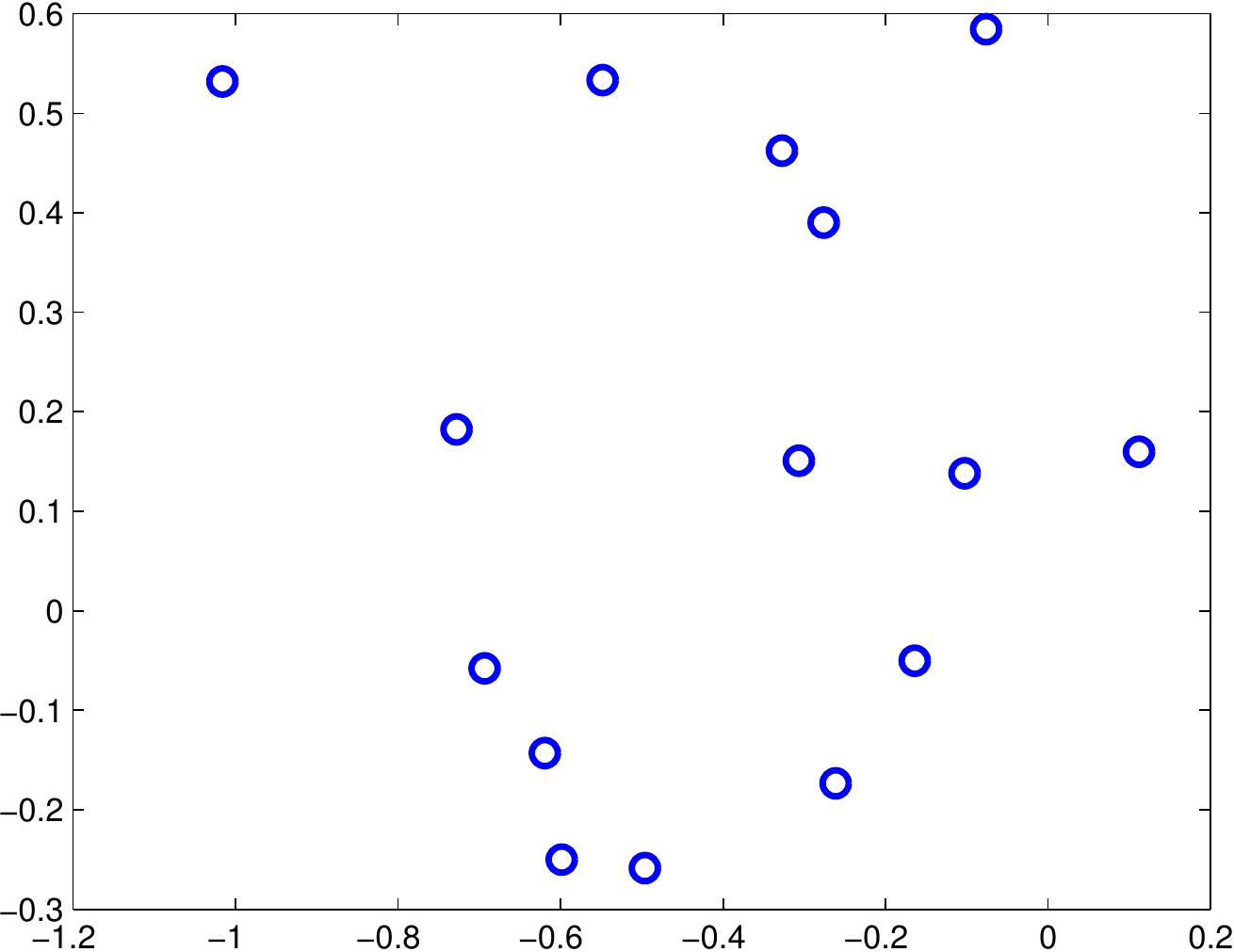}
  \includegraphics[width=0.32\textwidth,clip]{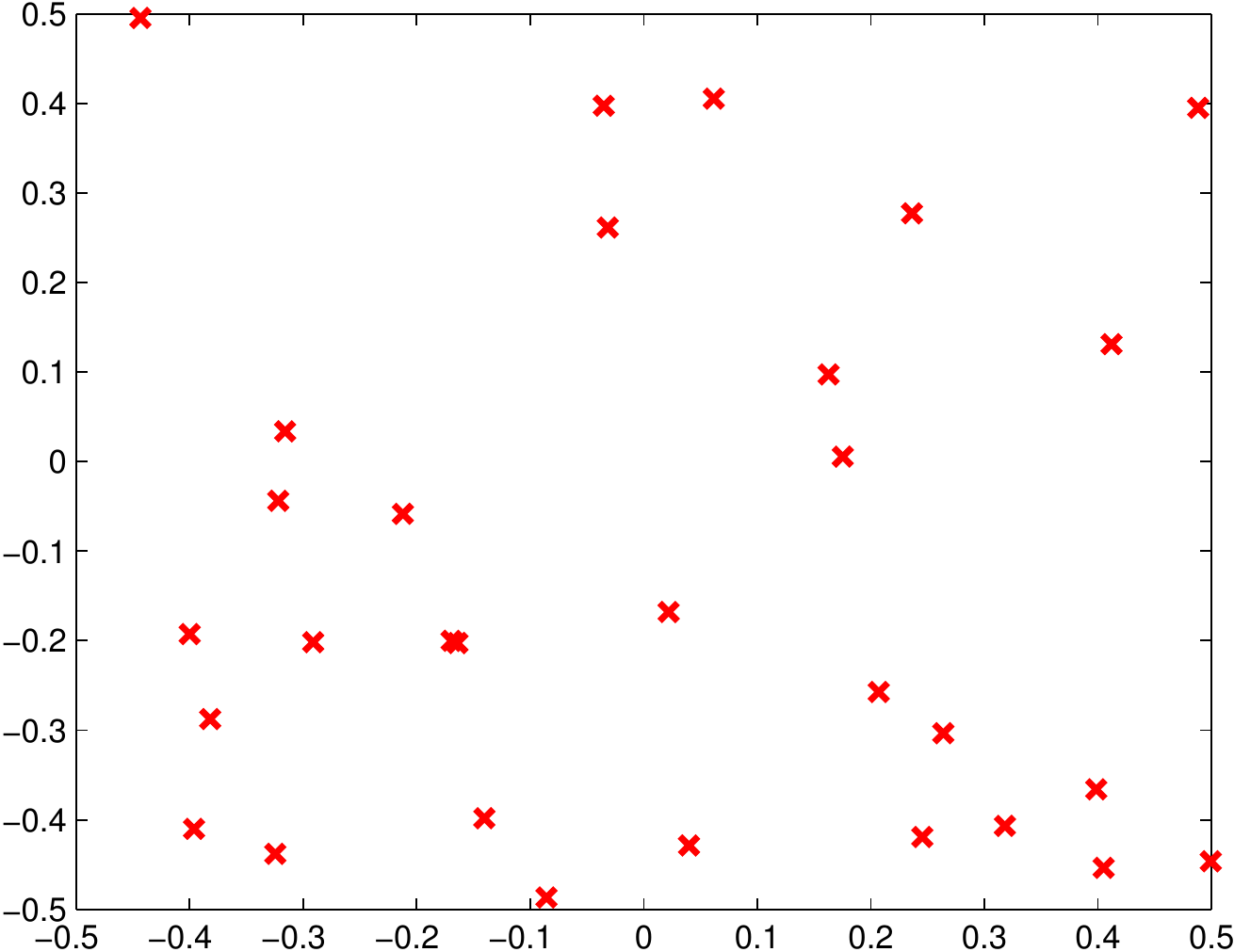}
  \includegraphics[width=0.32\textwidth,clip]{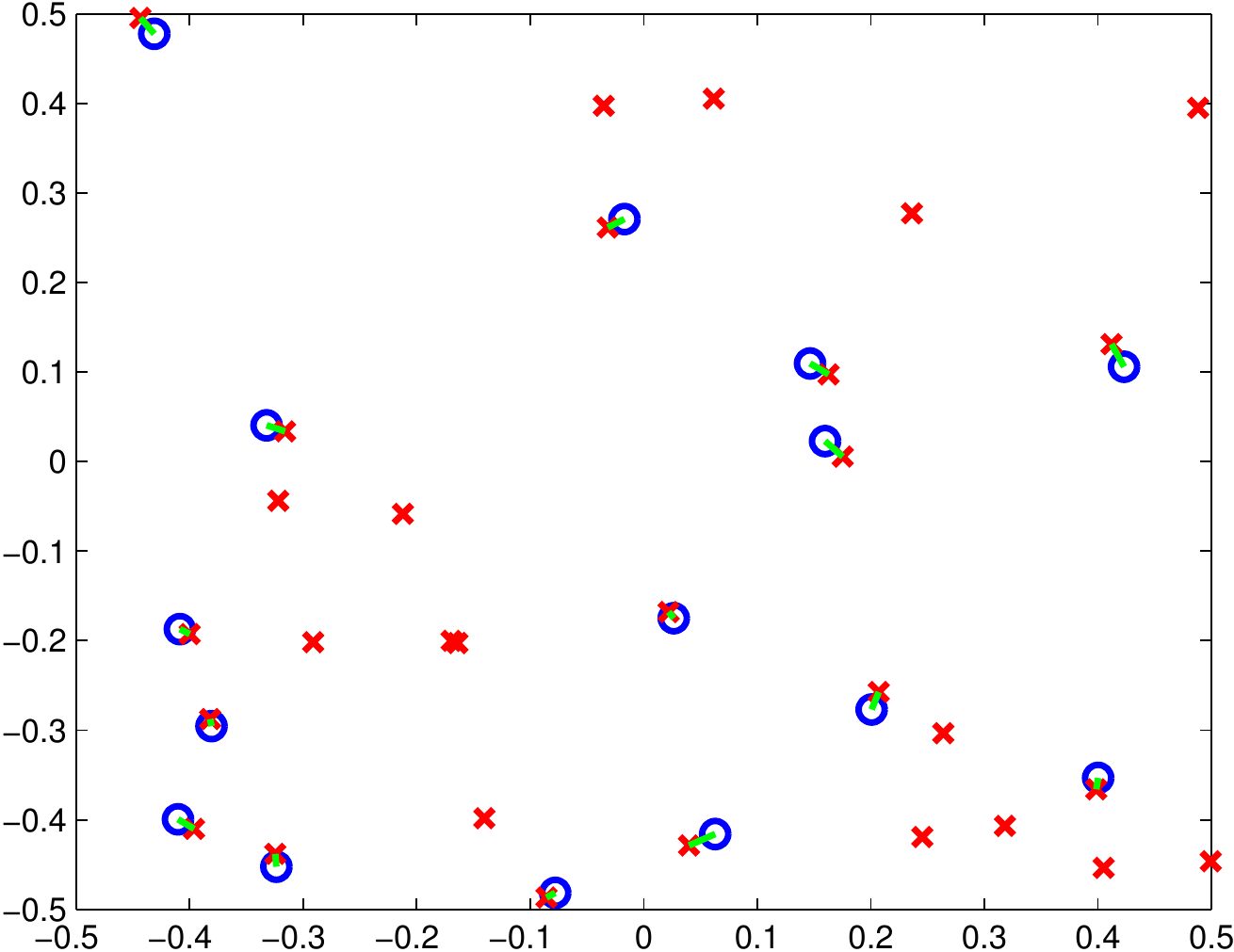}
  \includegraphics[width=0.32\textwidth,clip]{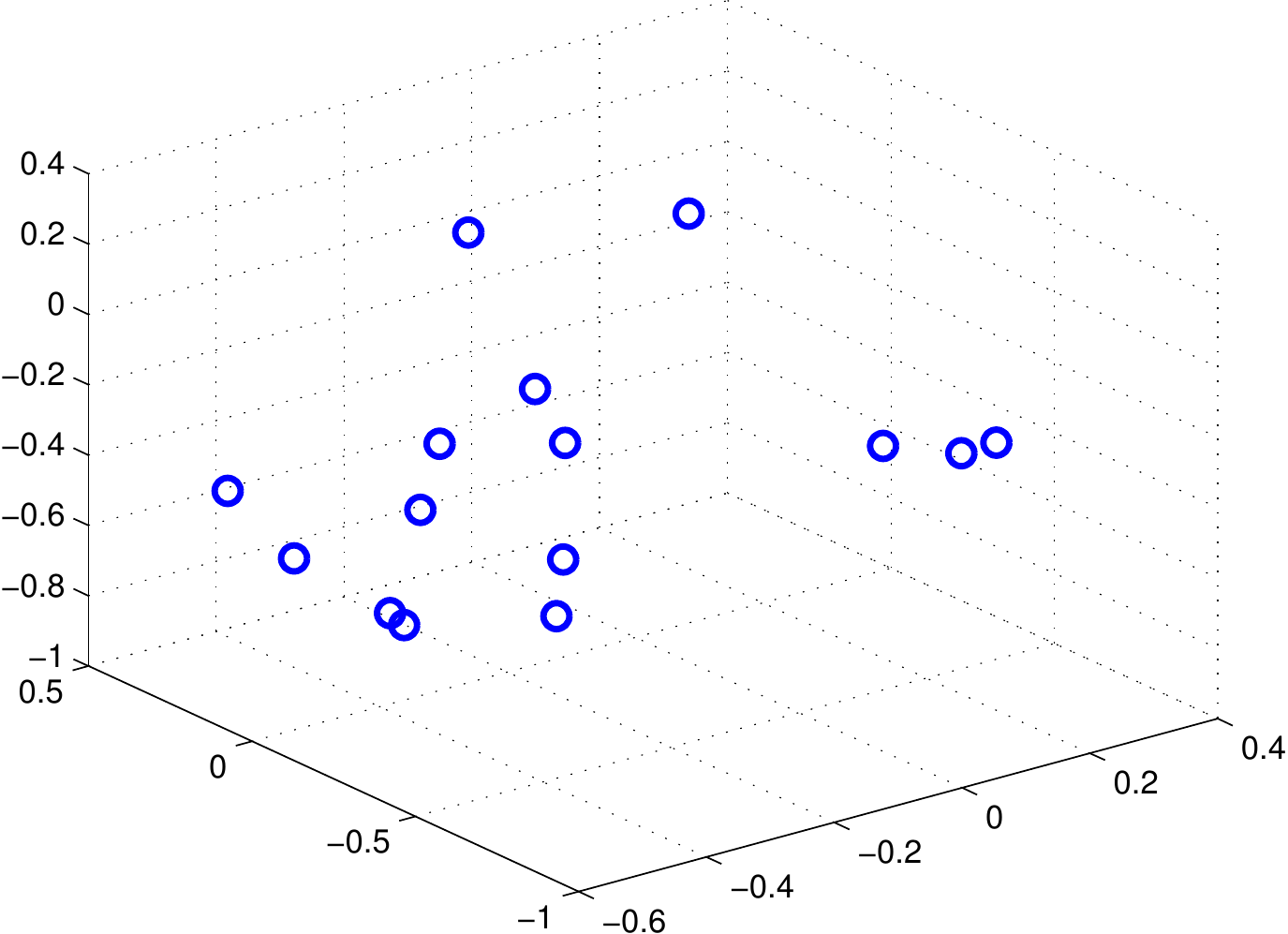}
  \includegraphics[width=0.32\textwidth,clip]{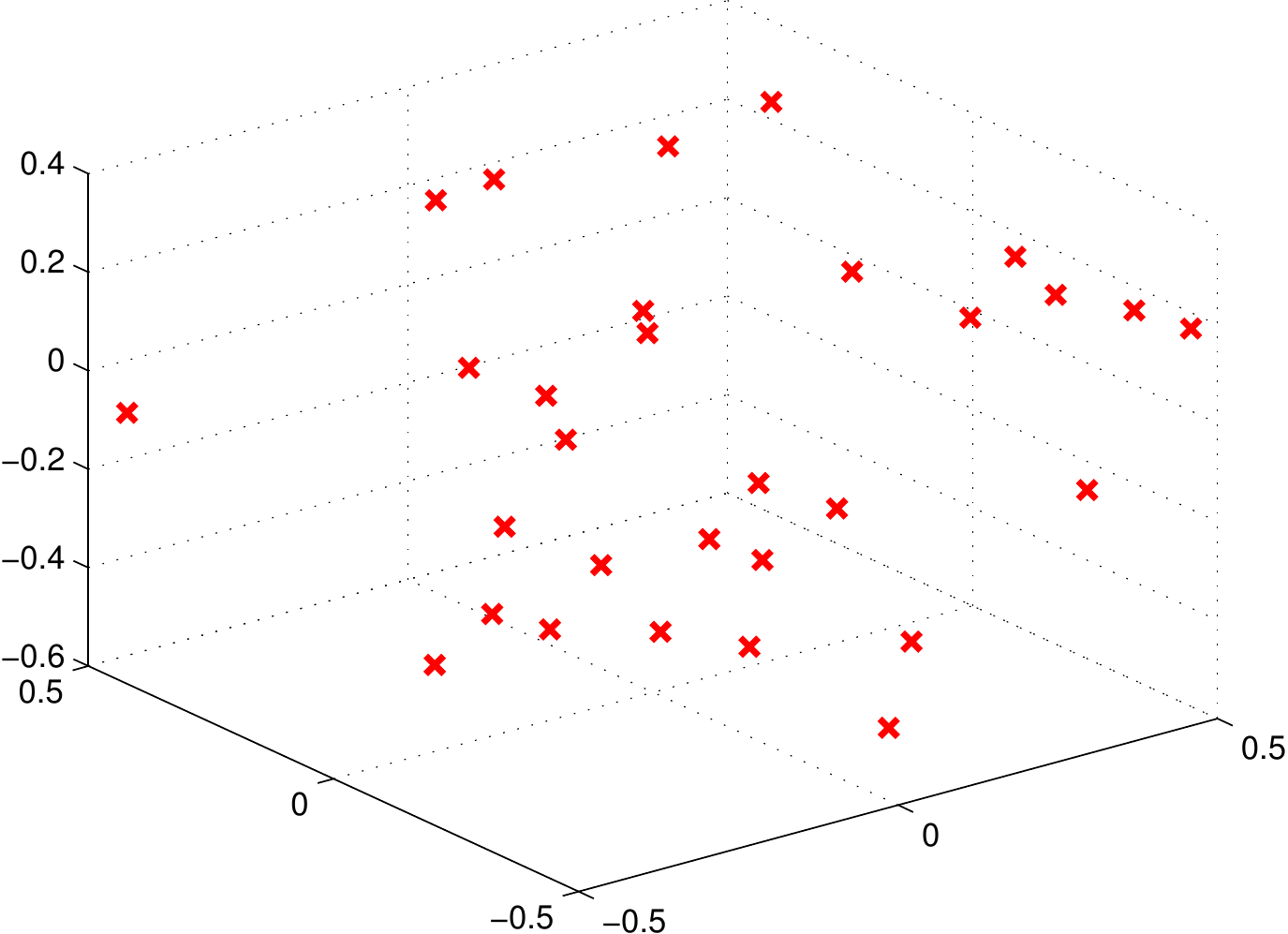}
  \includegraphics[width=0.32\textwidth,clip]{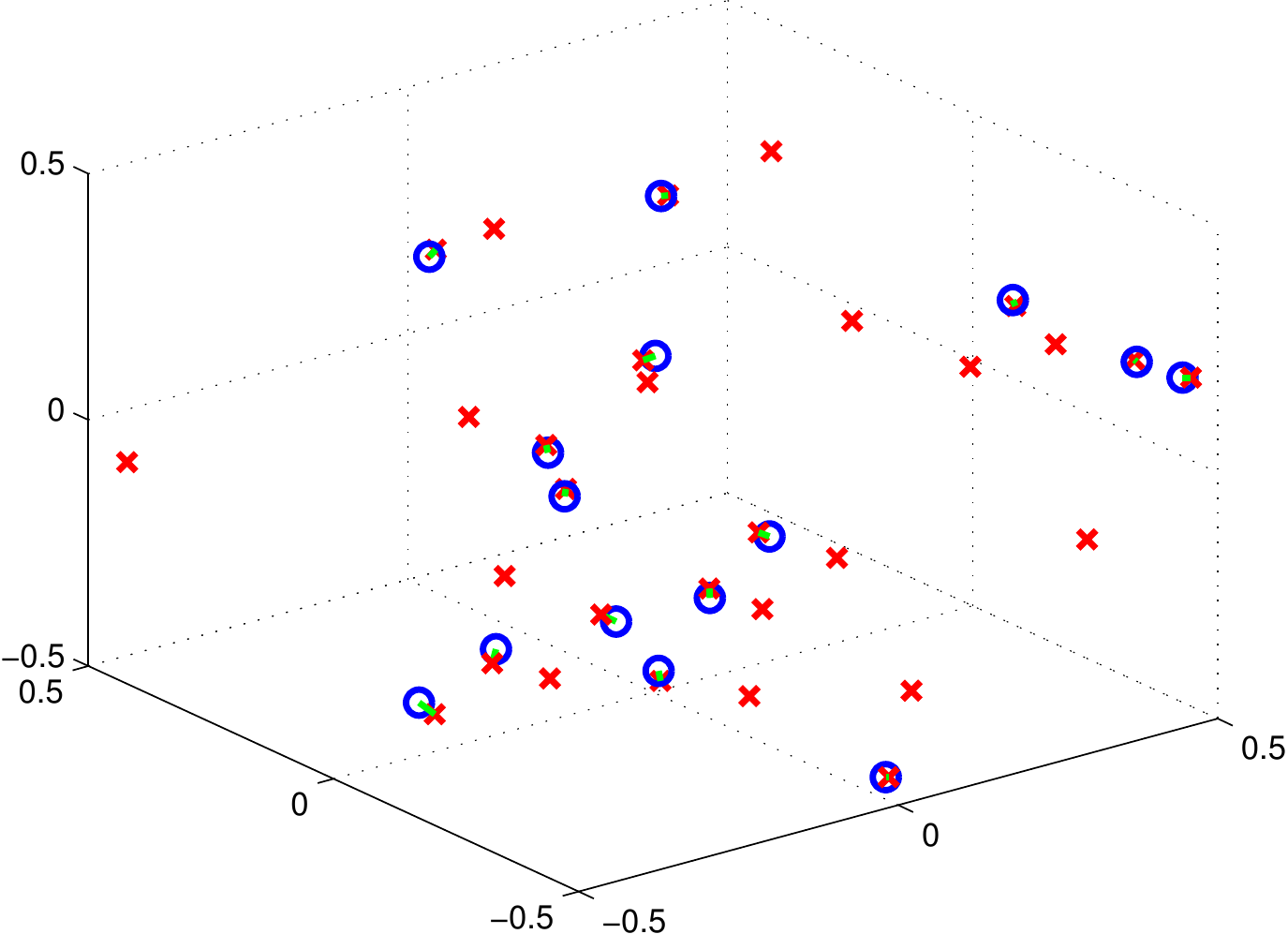}
  \includegraphics[width=0.32\textwidth,clip]{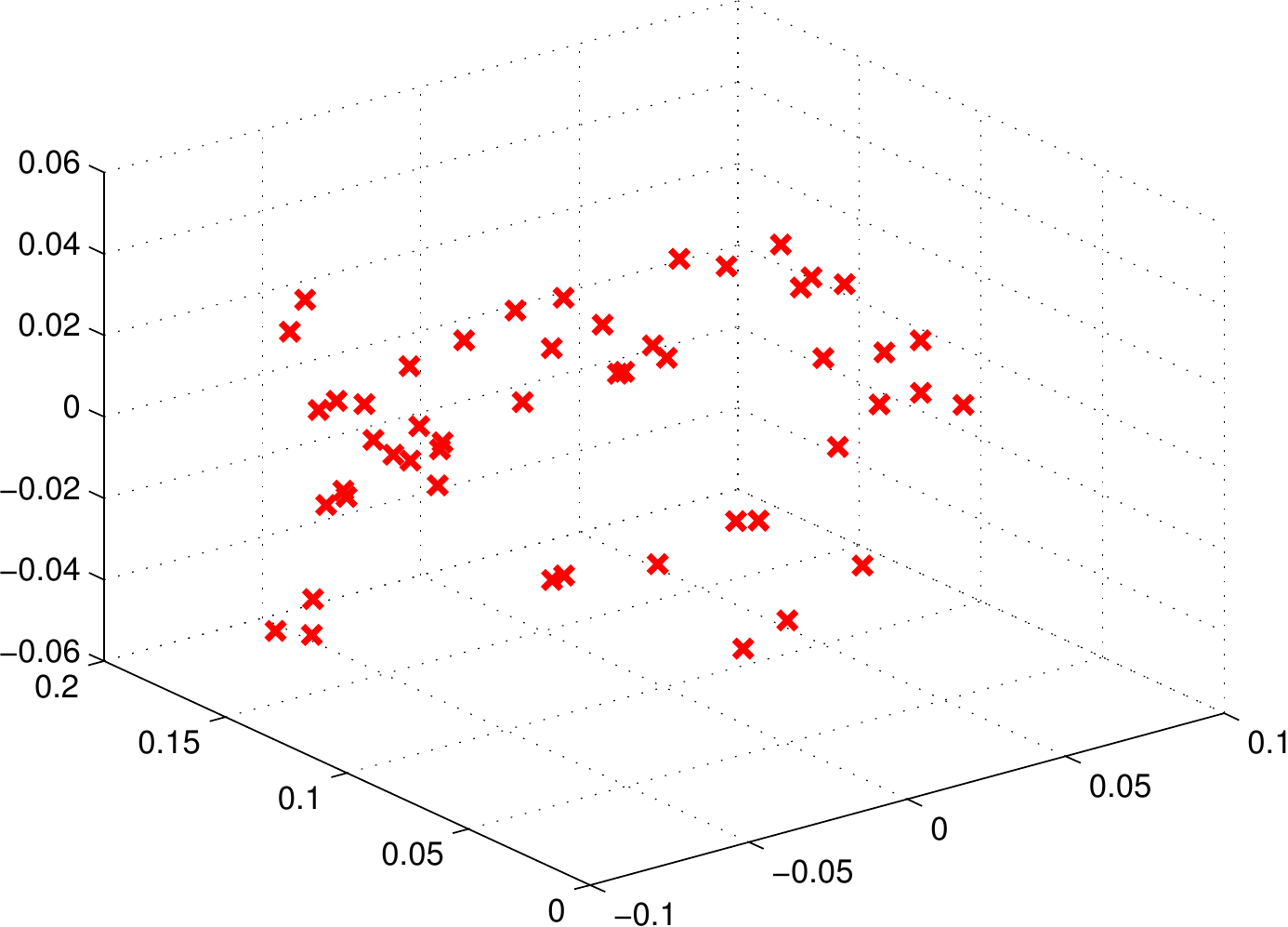}
  \includegraphics[width=0.32\textwidth,clip]{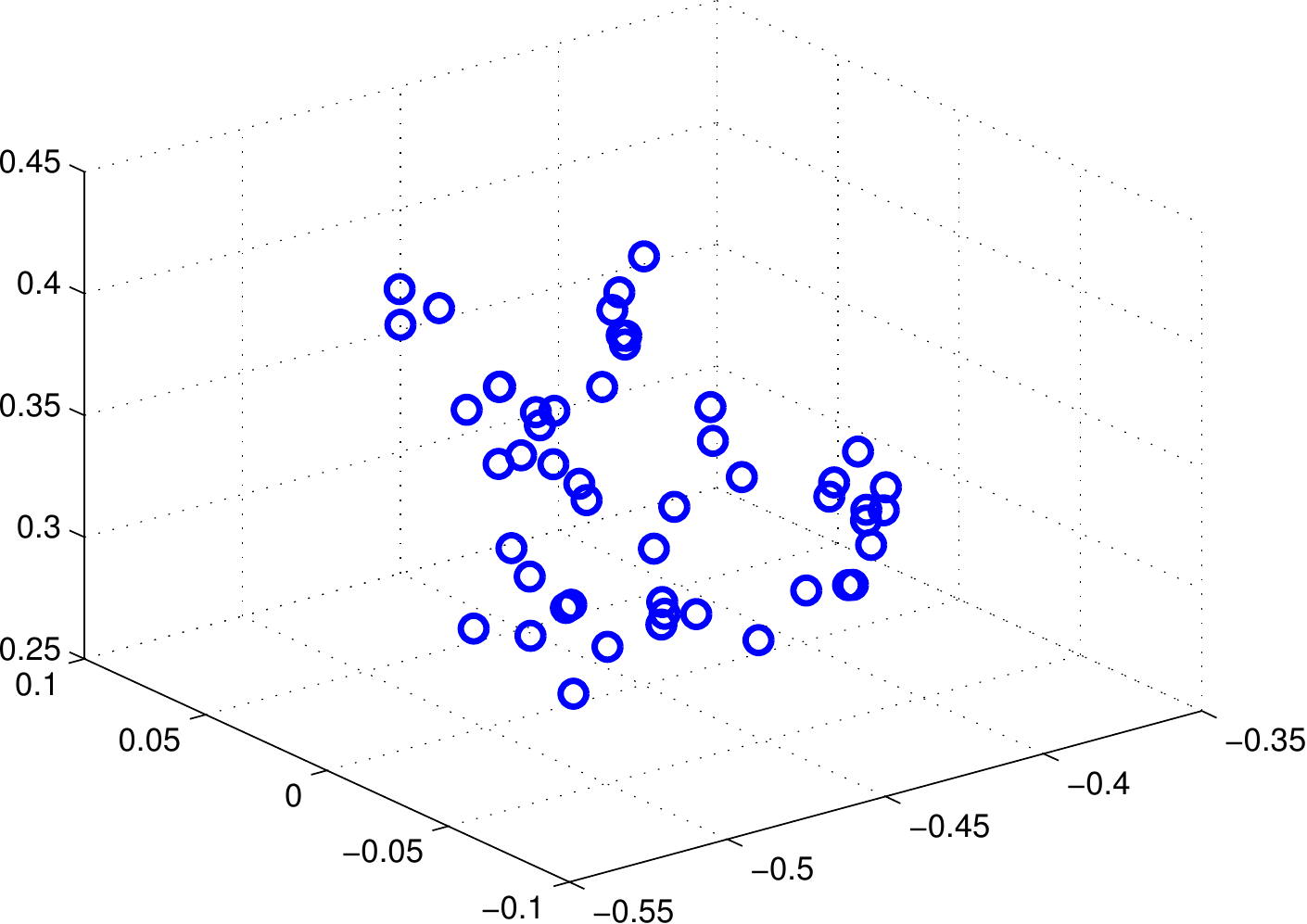}
  \includegraphics[width=0.32\textwidth,clip]{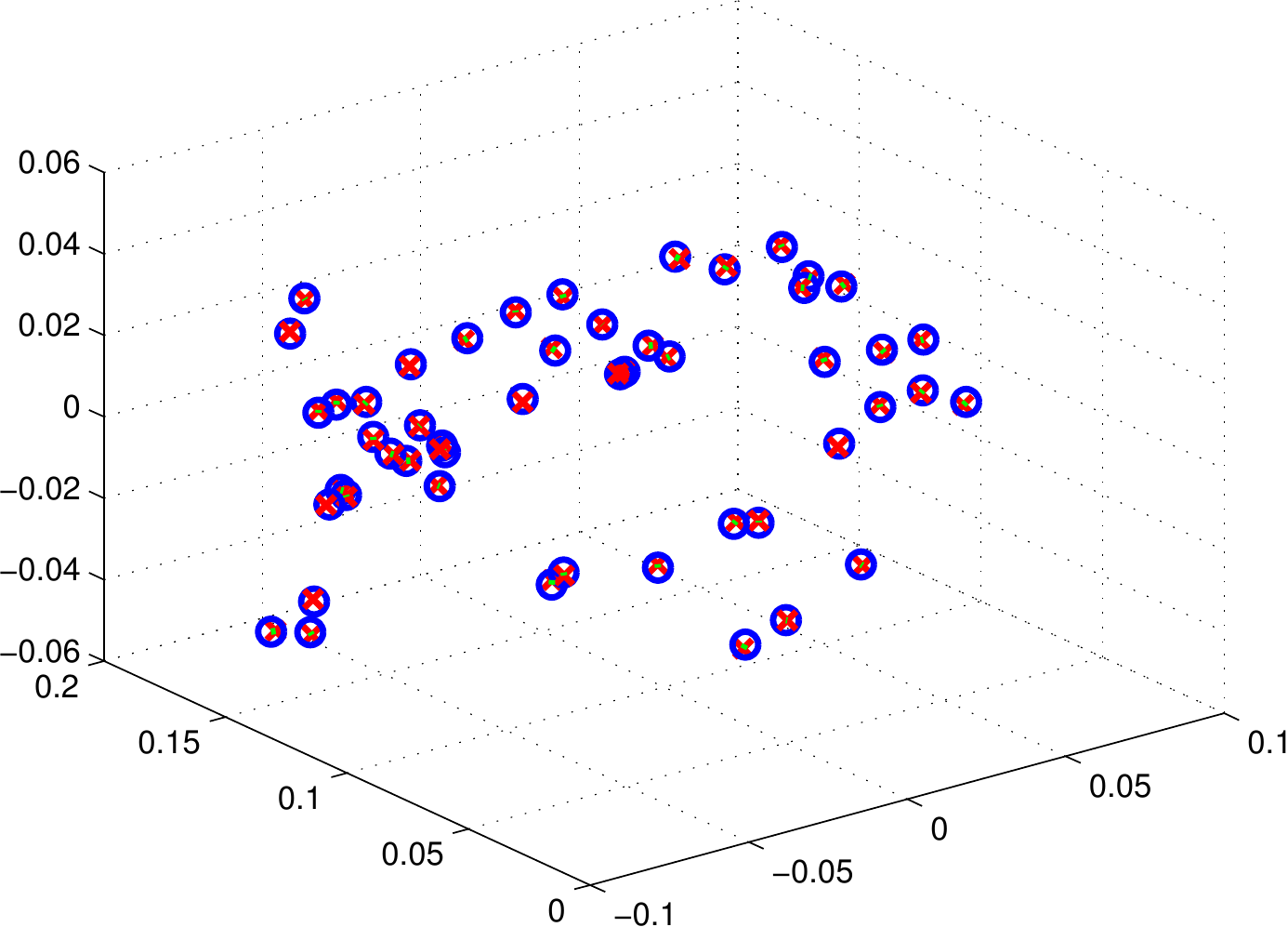}
\hspace{20cm} {\footnotesize source points} \hspace{1cm} {\footnotesize target points} \hspace{1cm} {\footnotesize matching results} 
\end{minipage}
\begin{minipage}[t]{\linewidth}
  \vspace{0.2cm}
  \centering  
  \footnotesize
  \begin{tabular}{l|l|ccc}
  \hline
     \multicolumn{2}{l|}{Data}       &  2d-toy & 3d-toy & bunny  \\
  \hline
  \hline  
    \multicolumn{2}{l|}{{\#} variables in BQP~\eqref{eq:graph_match_1}} &  $450$  & $450$  & $2500$ \\  
  \hline
    \multirow{3}{*}{Time(s)} & \fastsdp      &  $16.1$    & $19.0$    & $412$ \\
                             & SeDuMi       &  $2828$    & $3259$    & $>10000$    \\
                             & SDPT3        &  $969$     & $981$     & $>10000$    \\                            
  \hline
  \end{tabular}
\end{minipage}
\caption{Registration results. For 2d (top row) and 3d (middle row) artificial data, $15$ source points are matched to a subset of $30$ target points.
For bunny data (bottom row), there are $50$ source points and $50$ target points. $\sigma$ is set to $10^{-4}$.}
\label{fig:graph_match}
\end{figure}

\section{Conclusion}

In this paper, we have presented an efficient semidefinite formulation (\fastsdp) for BQPs.
\fastsdp produces a similar lower bound with the conventional SDP formulation, and therefore is tighter than spectral relaxation.
Our formulation is easy to implement by using the L-BFGS-B toolbox and standard eigen-decomposition software,
and therefore is much more scalable than the conventional SDP formulation.
We have applied \fastsdp to a few computer vision problems,
which demonstrates its flexibility in formulation.
Experiments also show the computational efficiency and good solution quality of \fastsdp.

We have made the code 
available online\footnote{\url{http://cs.adelaide.edu.au/~chhshen/projects/BQP/}}.

{\bf Acknowledgements} 
\noindent
    This work was in part supported by ARC Future
    Fellowship FT120100969. Correspondence should be address to C. 
    Shen. 

{
\bibliographystyle{abbrv}
\bibliography{sdcut1}
}

\end{document}